\documentclass{article}

\usepackage{microtype}
\usepackage{graphicx}
\usepackage{booktabs} 

\usepackage{hyperref}

\usepackage[accepted]{icml2023}

\usepackage{amsmath}
\usepackage{amssymb}
\usepackage{mathtools}
\usepackage{amsthm}
\usepackage[capitalize,noabbrev]{cleveref}
\usepackage{lipsum}
\usepackage{caption}
\usepackage{enumitem}
\usepackage{multirow}
\usepackage{xcolor}

\theoremstyle{plain}
\newtheorem{theorem}{Theorem}[section]

\newtheorem{corollary}[theorem]{Corollary}
\theoremstyle{definition}

\theoremstyle{remark}
\newtheorem{remark}[theorem]{Remark}

\usepackage[textsize=tiny]{todonotes}

\usepackage{hyperref}
\usepackage{url}
\usepackage{amsmath}
\usepackage{mathrsfs}
\usepackage{amssymb}
\usepackage{subfigure}
\usepackage{graphicx}
\usepackage{multicol}
\usepackage{multirow}
\usepackage{makecell} 
\usepackage{amsthm}
\usepackage{stackengine}
\usepackage{bbding}
\usepackage{diagbox}
\usepackage{algorithm, algorithmic}
\usepackage{wrapfig}
\usepackage{hhline}
\usepackage{listings}
\usepackage{amsthm}
\usepackage{xcolor}
\usepackage{enumitem}
\usepackage{booktabs}
\usepackage{booktabs}

\def\Algname{Dual Activation Precision}
\def\Algnameunderline{\underline{D}ual Act\underline{IV}ation Prec\underline{ISION}}
\def\Algnameabbr{DIVISION}

\icmltitlerunning{\Algnameabbr{}: Memory Efficient Training via Dual Activation Precision}

\begin{document}

\twocolumn[
\icmltitle{\Algnameabbr{}: Memory Efficient Training via Dual Activation Precision}


\begin{icmlauthorlist}
\icmlauthor{Guanchu Wang}{aff1}
\icmlauthor{Zirui Liu}{aff1}
\icmlauthor{Zhimeng Jiang}{aff2} \\
\icmlauthor{Ninghao Liu}{aff4} 
\icmlauthor{Na Zou}{aff3}
\icmlauthor{Xia Hu}{aff1}
\end{icmlauthorlist}

\icmlaffiliation{aff1}{Department of Computer Science, Rice University}
\icmlaffiliation{aff2}{Department of Computer Science and Engineering, Texas A\&M University}
\icmlaffiliation{aff3}{Department of Engineering Technology, Texas A\&M University}
\icmlaffiliation{aff4}{Department of Computer Science, University of Georgia}

\icmlcorrespondingauthor{Xia Hu}{xia.hu@rice.edu}


\icmlsetsymbol{equal}{*}

\icmlkeywords{Machine Learning, ICML}

\vskip 0.3in
]



\printAffiliationsAndNotice{}  


\begin{abstract}
Activation compressed training provides a solution towards reducing the memory cost of training deep neural networks~(DNNs).
However, state-of-the-art work combines a search of quantization bit-width with the training, which makes the procedure complicated and less transparent. To this end, we propose a simple and effective method to compress DNN training. Our method is motivated by an instructive observation: \emph{DNN backward propagation mainly utilizes the low-frequency component~(LFC) of the activation maps, while the majority of memory is for caching the high-frequency component~(HFC) during the training}. This indicates the HFC of activation maps is highly redundant and compressible, which inspires our proposed \Algnameunderline{}~(\Algnameabbr{}). During the training, \Algnameabbr{} preserves a high-precision copy of LFC and compresses the HFC into a light-weight copy with low numerical precision. This can significantly reduce the memory cost while maintaining a competitive model accuracy. 
Experiment results show \Algnameabbr{} has better comprehensive performance than state-of-the-art methods, including over $10\times$ compression of activation maps and competitive training throughput, without loss of model accuracy.
The source code is available at \url{https://github.com/guanchuwang/division}.


\end{abstract}



\vspace{-2mm}
\section{Introduction}
\label{sec:introduction}

Deep neural networks (DNNs) have been widely applied to real-world tasks such as language understanding~\cite{devlin2018bert}, machine translation~\cite{vaswani2017attention}, visual detection and tracking~\cite{redmon2016you}. 
With increasingly larger and deeper architectures, DNNs achieve remarkable improvement in representation learning and generalization capacity~\cite{krizhevsky2012imagenet}. %
Nevertheless, training a larger model requires more memory resources to cache the activation values of all intermediate layers during the backward propagation\footnote{The activation map of each layer is required for estimating the gradient during backward propagation.}.
For example, training a DenseNet-121~\cite{huang2017densely} on the ImageNet dataset~\cite{deng2009imagenet} requires to cache over 1.3 billion float activation values~(4.8GB) during backward propagation; 
and training a ResNet-50~\cite{he2016deep} requires to cache over 4.6 billion float activation values~(17GB). 
Some techniques have been developed to reduce the training cache of DNNs, such as checkpointing~\cite{chen2016training, gruslys2016memory}, mix precision training~\cite{vanholder2016efficient}, low bit-width training~\cite{lin2017towards, chen2020statistical} and activation compressed training~\cite{georgiadis2019accelerating, evans2021ac}. 
Among these, the activation compressed training (ACT) has emerged as a promising method due to its significant reduction of training memory and the competitive learning performance~\cite{liu2021exact}.

\begin{figure}
\setlength{\abovecaptionskip}{1mm}
\setlength{\belowcaptionskip}{-6mm}
\centering
\includegraphics[width=0.8\linewidth]{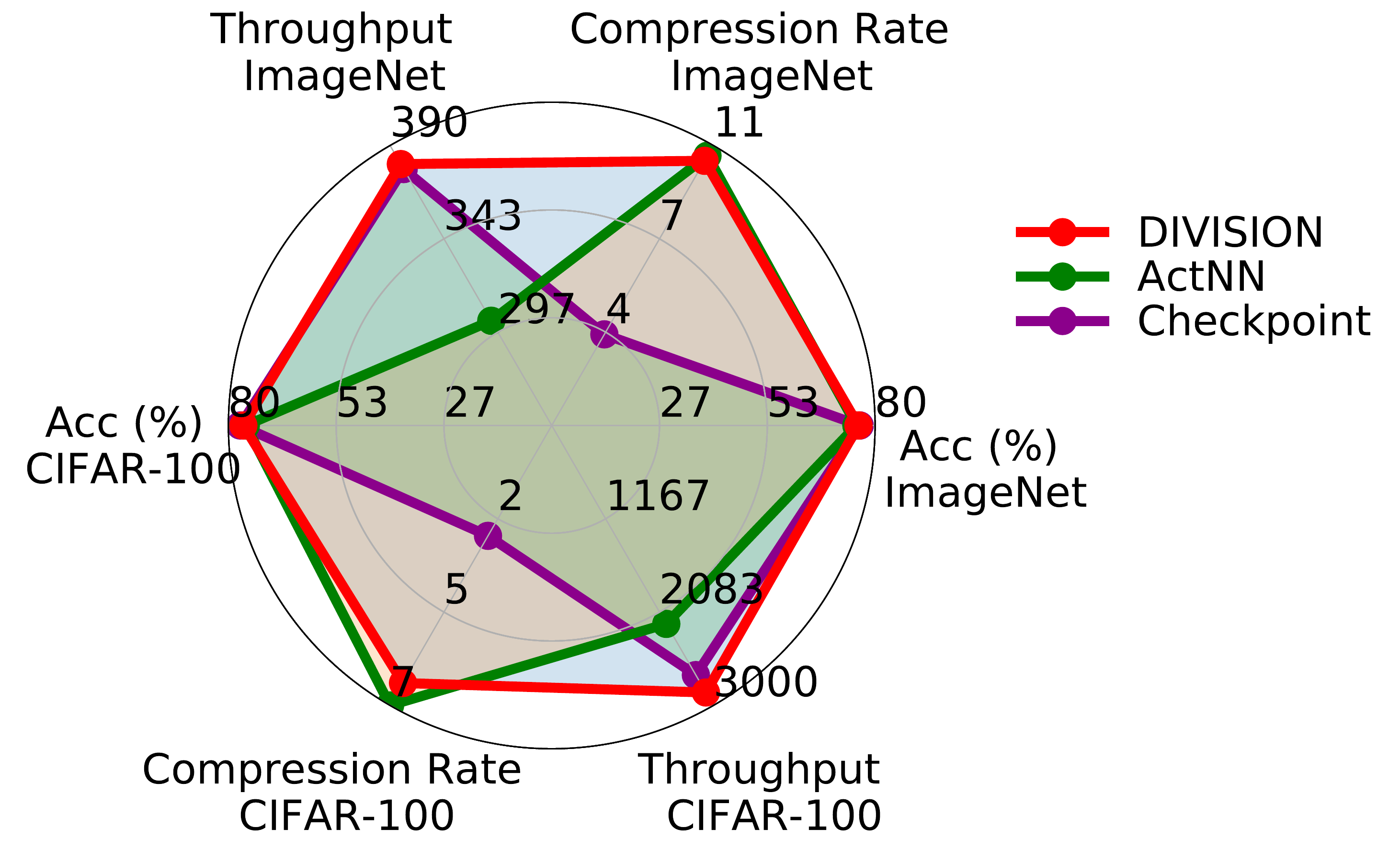}
\caption{\label{fig:radar} Performance of \Algnameabbr{} vs baseline methods.}
\end{figure}

Existing work of ACT relies on quantizing the activation maps to reduce the memory consumption of DNN training, such as BLPA~\cite{chakrabarti2019backprop}, TinyScript~\cite{fu2020don} and ActNN~\cite{chen2021actnn}.
Although ACT could significantly reduce the training memory cost, the quantization process introduces noises in backward propagation, which makes the training suffer from undesirable degradation of accuracy~\cite{fu2020don}.
Due to this reason, BLPA requires 4-bit ACT to ensure the convergence to optimal solution on the ImageNet dataset, which has only a $6\times\!$ compression rate\footnote{\footnotesize A $6 \times$ compression rate indicates the memory of cached activation maps is $1/6$ of that of normal training.} of activation maps~\cite{chakrabarti2019backprop}.
Other works propose to search for optimal bit-width to match different samples during training, such as ActNN~\cite{chen2021actnn} and AC-GC~\cite{evans2021ac}.
Although they can moderately reduce the quantization noise and achieves optimal solution under 2-bit ACT~(nearly 10$\times\!$ compression rate), the following issues cannot be ignored.
First, it is time-consuming to search for the optimal bit-width during training.
Second, the framework of bit-width searching is complicated and non-transparent, which brings new challenges to follow-up studies on the ACT and its applicability.


In this work, we propose a simple and transparent method to reduce the memory cost of DNN training.
Our method is motivated by an instructive observation:
\emph{DNN backward propagation mainly utilizes the low-frequency component~(LFC) of the activation maps, while the majority of memory is for the storage of high-frequency component~(HFC).}
This indicates the HFC of activation map is highly redundant and compressible during the training.
Following this direction, we propose \Algname{}~(\Algnameabbr{}), which preserves the high-precision copy of LFC and compresses the HFC into a light-weight copy with low numerical precision during the training.
In this way, \Algnameabbr{} can significantly reduce the memory cost.
Meanwhile, it will not negatively affect the quality of backward propagation and could maintain competitive model accuracy.

Compared with existing work that integrates searching into learning~\cite{chen2021actnn}, \Algnameabbr{} has a more simplified compressor and decompressor, speeding up the procedure of ACT.
More importantly, it reveals the compressible (HFC) and non-compressible factors (LFC) during DNN training, improving the transparency of ACT.
Figure~\ref{fig:radar} gives the comprehensive performance of \Algnameabbr{} compared with state-of-the-art methods, which demonstrates the competitiveness of \Algnameabbr{} in terms of the model accuracy, compression rate, and training throughput.
The contributions of this work are summarized as follows:
\begin{itemize}[leftmargin=9pt, topsep=-1mm]
\setlength{\parskip}{1mm}
\setlength{\parsep}{0mm}
\setlength{\itemsep}{0mm}

    \item We experimentally and theoretically prove DNN backward propagation mainly utilizes the LFC of the activation maps.
    The HFC is highly redundant and compressible.
    
    
    

    \item We propose a simple and effective framework called \Algnameabbr{} to reduce the memory cost of DNN training via removing the redundancy in the HFC of activation maps.

 
    \item Experiments on three benchmark datasets demonstrate the effectiveness of \Algnameabbr{} in terms of memory cost, model accuracy, and training throughput.
    
    
\end{itemize}

\section{Preliminary}
\subsection{Notations}

Without loss of generality, we consider an $L$-layer deep neural network in this work.
During the forward pass, for each layer $l$~($1 \leq l \leq L$), the activation map is given by
\begin{equation}
\setlength{\abovedisplayskip}{1mm}
\setlength{\belowdisplayskip}{1mm}
\mathbf{H}_{l} = \mathrm{forward}(\mathbf{H}_{l-1}; \mathbf{W}_l),
\end{equation}
where $\mathbf{H}_{l}$ denotes the activation map of layer $l$; $\mathbf{H}_0$ takes a mini-batch of input images; $\mathbf{W}_{l}$ denotes the weight of layer $l$; and $\mathrm{forward}(\cdot)$ denotes a feed-forward operation.
During the backward pass, the gradients of the loss value towards the activation maps and weights are be estimated by
\begin{equation}
\setlength{\abovedisplayskip}{1mm}
\setlength{\belowdisplayskip}{1mm}
\label{eq:grad_estimate}
\big[ \hat{\nabla}_{\mathbf{H}_{l-1}}, \hat{\nabla}_{\mathbf{W}_{l}} \big]  = \mathrm{backward}( \hat{\nabla}_{\mathbf{H}_{l}}, \mathbf{H}_{l-1}, \mathbf{W}_{l} ),
\end{equation}
where $\hat{\nabla}_{\mathbf{H}_{l-1}}$ and $\hat{\nabla}_{\mathbf{H}_{l}}$ denote the gradient towards the activation map of layer $l\!-\!1$ and $l$, respectively; $\hat{\nabla}_{\mathbf{W}_{l}}$ denotes the gradient towards the weight of layer $l$;  and $\mathrm{backward}(\cdot)$\footnote{We do not focus on the closed from the backward function, which is implemented by \texttt{torch.autograd}.} denotes the backward function which takes $\hat{\nabla}_{\mathbf{H}_{l}}$, ${\mathbf{H}}_{l-1}$ and $\mathbf{W}_{l}$, and outputs the gradients $\hat{\nabla}_{\mathbf{H}_{l-1}}$ and $\hat{\nabla}_{\mathbf{W}_{l}}$.
Equation~(\ref{eq:grad_estimate}) indicates it is required to cache the activation maps $\mathbf{H}_{0}, \cdots, \mathbf{H}_{L\!-\!1}$ after the feed-forward operations for gradient estimation during backward propagation.


\vspace{-1mm}
\subsection{Activation Compressed Training}

It has been proved in existing work~\cite{chen2020statistical} that majority of memory~(nearly 90\%) is for caching activation maps during the training of DNNs.
Following this direction, the activation compressed training~(ACT) reduces the memory cost via real-time compressing the activation maps during the training.
A typical ACT framework in existing work~\cite{chakrabarti2019backprop} is shown in Figure~\ref{fig:act_framework}.
Specifically, after the feed-forward operation of each layer $l$, activation map $\mathbf{H}_{l-1}$ is compressed into a representation for caching.
The compression enables a significant reduction of memory compared with caching the original~(exact) activation maps.
During the backward pass of layer $l$, ACT decompresses the cached representation into $\hat{\mathbf{H}}_{l-1}$, and estimates the gradient by taking the reconstructed $\hat{\mathbf{H}}_{l-1}$ into Equation~(\ref{eq:grad_estimate}): $[ \hat{\nabla}_{\mathbf{H}_{l-1}}, \hat{\nabla}_{\mathbf{W}_{l}} ]  = \mathrm{backward}( \hat{\nabla}_{\mathbf{H}_{l}}, \hat{\mathbf{H}}_{l-1}, \mathbf{W}_{l} )$.


\vspace{-1mm}
Even though the pipeline of compression and decompression is lossy, i.e. $\hat{\mathbf{H}}_l \neq \mathbf{H}_l$ for $1 \leq l \leq L$.
It has been proved ACT can limit the reconstruction error flowing back to early layers and enables the training to approach an approximately optimal solution~\cite{chen2021actnn}.

\begin{figure}
\setlength{\abovecaptionskip}{2mm}
\setlength{\belowcaptionskip}{-5mm}
    \centering
    \includegraphics[width=0.99\linewidth]{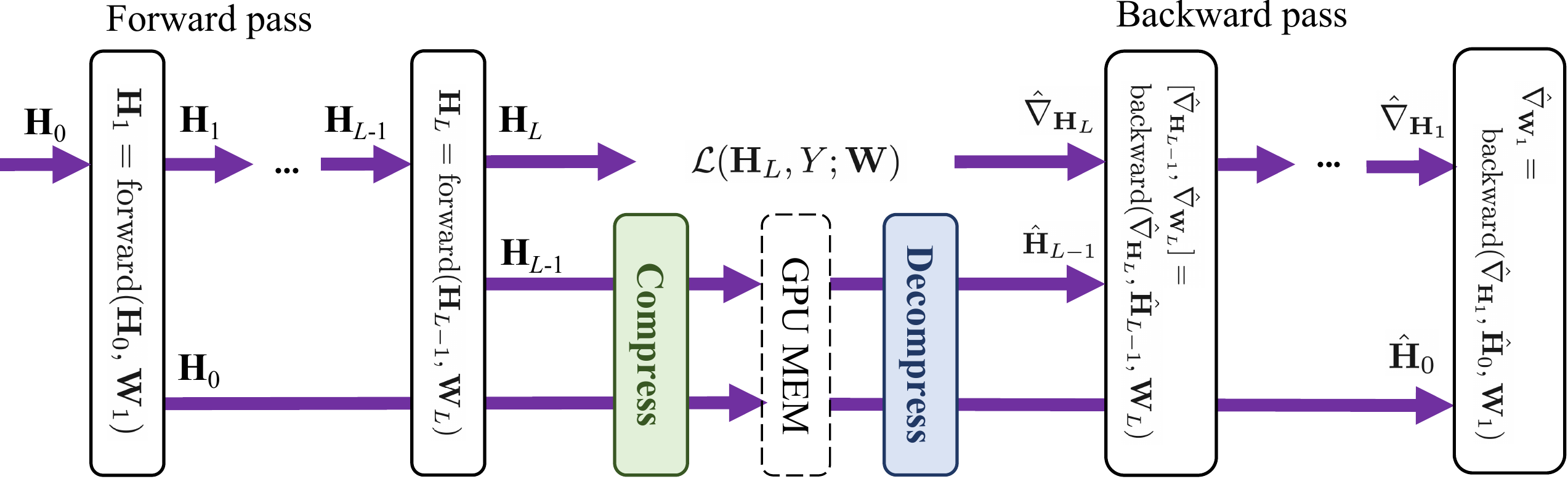}
\caption{\label{fig:act_framework} 
\small Activation compressed training.}
\end{figure}

\begin{figure*}
\vspace{-1mm}
\setlength{\abovecaptionskip}{-1mm}
\setlength{\belowcaptionskip}{-6mm}
    \centering
    \subfigure[]{
    \centering
    \begin{minipage}[t]{0.5\linewidth}
    \includegraphics[width=1.0\linewidth]{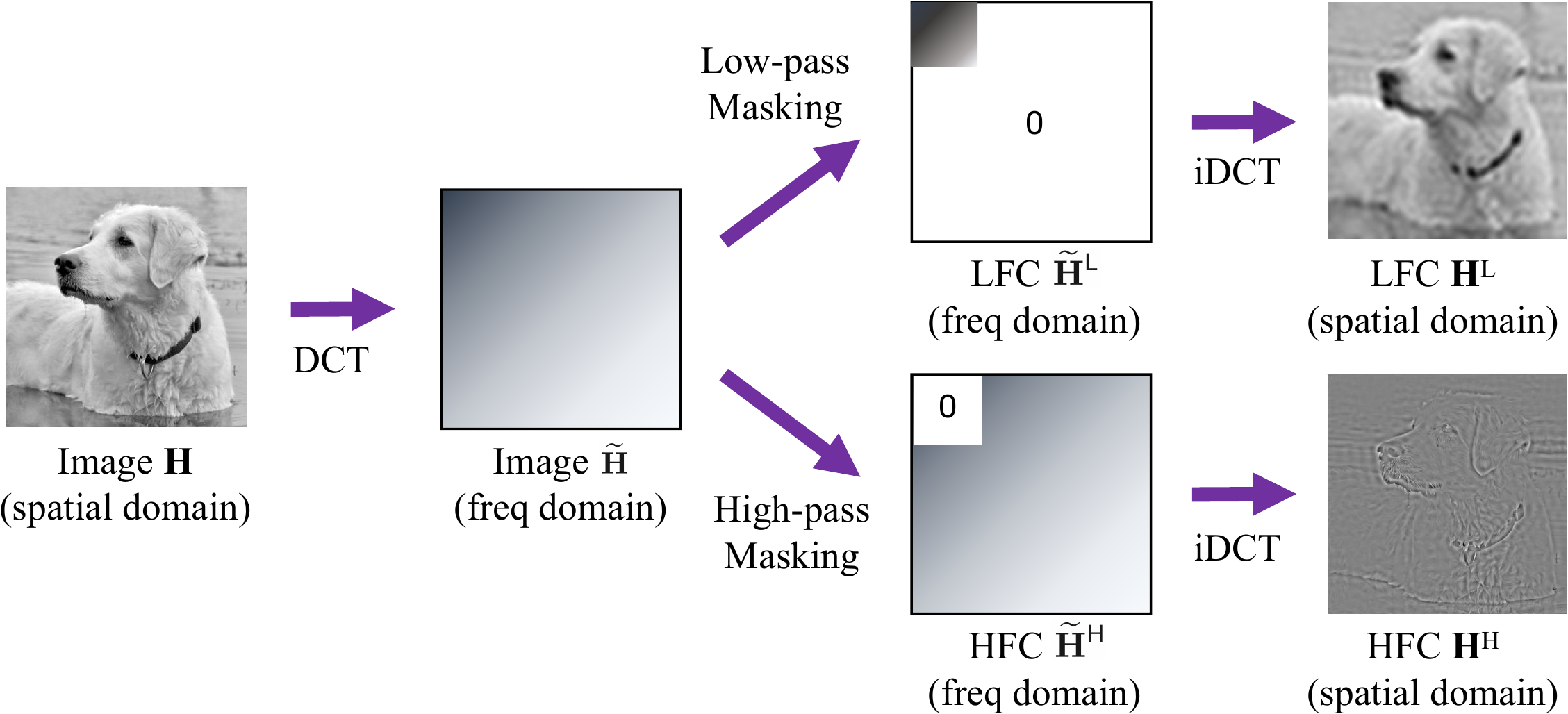}
    \end{minipage}%
    }
    \quad
    \subfigure[]{
    \centering
    \begin{minipage}[t]{0.15\linewidth} 
    	\includegraphics[width=0.99\linewidth]{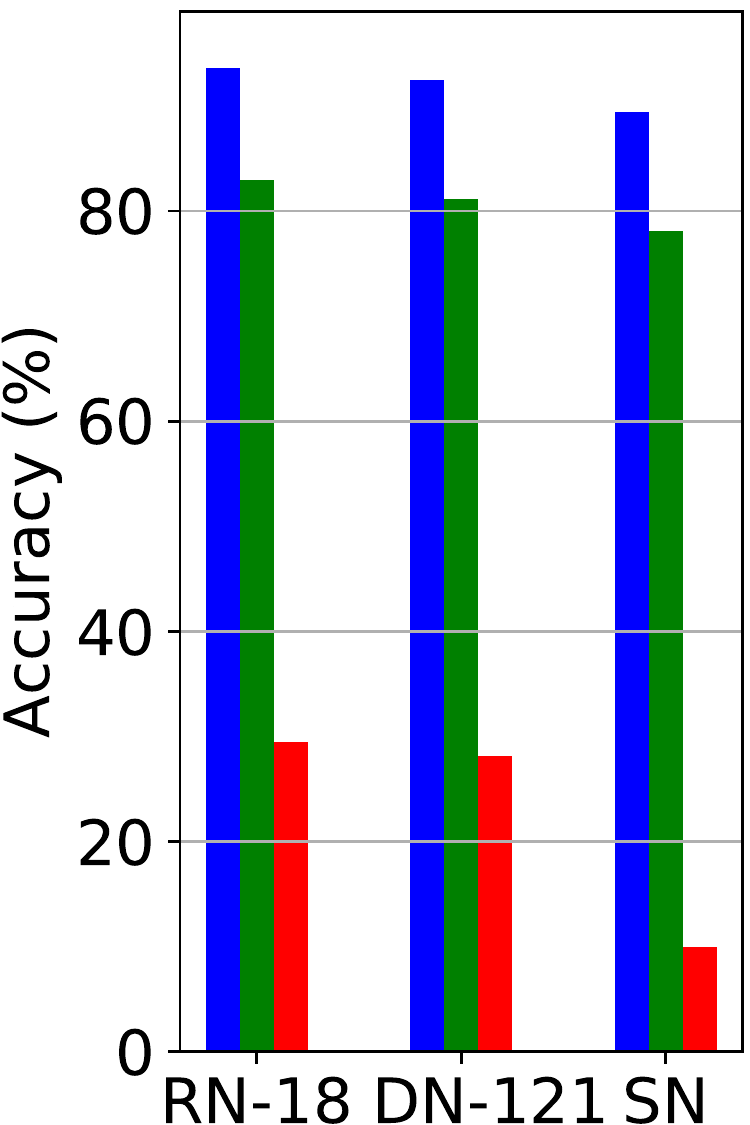}
    \end{minipage}%
    }
    \subfigure[]{
    \centering
    \begin{minipage}[t]{0.15\linewidth}
    	\includegraphics[width=0.99\linewidth]{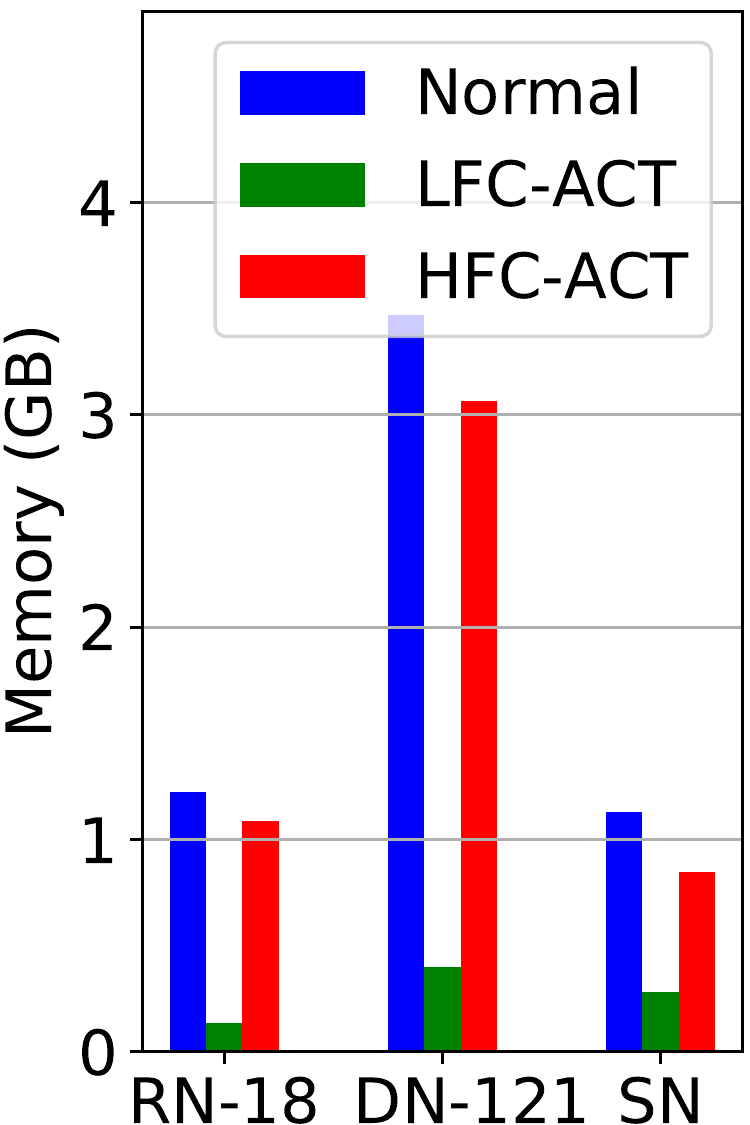}
    \end{minipage}%
    }
\caption{\label{fig:dct_framework} \small (a) Adopting DCT to estimate the low frequency component~(LFC) and high frequency component~(HFC) of an image.
(b) Top-1 accuracy and (c) Memory cost of normal training, LFC-ACT and HFC-ACT, where RN-18, DN-121, and SN refer to the ResNet-18, DenseNet-121, and ShuffleNet-V2, respectively.
}
\end{figure*}

\vspace{-1mm}
\subsection{Discrete Cosine Transformation}

\label{sec:preli_dct}

Discrete Cosine Transformation~(DCT) projects the target data from the spatial domain to the frequency domain via the inner-production of the data and a collection of cosine functions with different frequency~\cite{rao2014discrete}.
We focus on the 2D-DCT in this section, where the target data is the input image and activation maps of DNNs.
The cases of 1D/3D-DCT for 1D/3D activation maps are considered in Appendix~\ref{appendix:1D/3D-DCT}.
Specifically, for 2D-matrix data $\mathbf{H}$, the frequency-domain feature $\widetilde{\mathbf{H}}$ is estimated by $\widetilde{\mathbf{H}} \!=\! \mathrm{DCT}(\mathbf{H})$, where $\mathbf{H}$ and $\widetilde{\mathbf{H}}$ have the same shape of $N \!\times\! N$;
and each of the element $\tilde{h}_{i,j}$ is given by
\begin{equation}
\setlength{\abovedisplayskip}{0mm}
\setlength{\belowdisplayskip}{0mm}
\small
\!\!\!\!\tilde{h}_{i,j} \!=\! \sum^{N-1}_{m=0} \sum^{N-1}_{n=0} h_{m,n} \cos \! \bigg[ \frac{\pi}{N} \bigg( m \!+\! \frac{1}{2}  \bigg) i \bigg] \! \cos \! \bigg[ \frac{\pi}{N} \bigg( n \!+\! \frac{1}{2}  \bigg) j \bigg] \!,
\end{equation}
where $h_{m,n}$, $0\!\leq\! m,n \!\leq\! N\!-\!1$, are elements in the original matrix $\mathbf{H}$.
During the training of DNNs, an image or activation map has the shape of $\mathrm{Minibatch} \!\times\! \mathrm{Channel} \!\times\! N \!\times\! N$.
In this case, the frequency-domain feature is estimated via operating 2D-DCT for each $N \!\!\times\!\! N$ matrix in each channel.


With DCT, we could extract the low-frequency/high-frequency component~(LFC/HFC) of an image or activation map, using a pipeline of low-pass/high-pass masking and inverse DCT, as shown in Figure~\ref{fig:dct_framework}~(a).
To be concrete, the estimation of LFC and HFC is given by
{\setlength\abovedisplayskip{1mm}
\setlength\belowdisplayskip{1mm}
\begin{align}
    \label{eq:lfc_estimate}
    \mathbf{H}^{\mathsf{L}} &= \mathrm{iDCT}(\widetilde{\mathbf{H}} \odot \mathbf{M})
    \\[-4bp]
    \label{eq:hfc_estimate}
    \mathbf{H}^{\mathsf{H}} &= \mathrm{iDCT}(\widetilde{\mathbf{H}} \odot (\mathbf{1}_{N \times N} - \mathbf{M})),
\end{align}}
\!\!\! where $\mathrm{iDCT}(\cdot)$ denotes the inverse DCT~\cite{rao2014discrete}; $\mathbf{M} = [m_{i,j} | 1 \leq i,j \leq N]$ denotes an $N \!\times\! N$ low-pass mask satisfying $m_{i,j} \!=\! 1$ for $1 \leq i,j \leq W$ and $m_{i,j} \!=\! 0$ for other elements; and $\mathbf{1}_{N \times N} \!-\! \mathbf{M}$ indicates the high-pass mask.
Intuitively, $\mathbf{H}^{\mathsf{L}}$ has $W^2$ non-zero float numbers in each channel, in contrast with $N^2\!-\!W^2$ non-zero float numbers in each channel of $\mathbf{H}^{\mathsf{H}}$.
Generally, we have $W \!\ll\! N$ in practical scenarios, e.g. $W/N=0.1$ in Figure~\ref{fig:dct_framework}~(a).
This indicates the HFC takes the majority of the memory cost in the caching of activation maps.





\section{Contribution of LFC and HFC to Backward Propagation}
\label{sec:pre_exp}




In this section, we experimentally prove the LFC of activation maps has significantly more contribution to DNN backward propagation than the HFC.
Meanwhile, we theoretically prove that LFC makes the estimated gradient to be bounded into a tighter range around the optimal value, leading to a more accurate learned model, which is consistent with the experimental results.


\subsection{Experimental Analysis}
\label{sec:prelim_exp}


To study the contribution of LFC and HFC to DNN backward propagation, we design three training methods with different backward propagations:
\textbf{LFC-ACT} takes LFC into the backward function as shown in Equation~(\ref{eq:lfc_grad_estimate}), where $\mathbf{H}_{l}^{\mathsf{L}}$ is estimated by Equations~(\ref{eq:lfc_estimate});
\textbf{HFC-ACT} takes HFC into the backward function as given in Equation~(\ref{eq:hfc_grad_estimate}), where $\mathbf{H}_{l}^{\mathsf{H}}$ is according to Equation~(\ref{eq:hfc_estimate});
\textbf{Normal training}~(for comparison) estimates the gradients by Equation~(\ref{eq:grad_estimate}).\!\!\!\!
{\setlength\abovedisplayskip{0mm}
\setlength\belowdisplayskip{-1mm}
\begin{align}
    \label{eq:lfc_grad_estimate}
    [ \hat{\nabla}_{\mathbf{H}_{l-1}}, \hat{\nabla}_{\mathbf{W}_{l}} ]  &= \mathrm{backward}( \hat{\nabla}_{\mathbf{H}_{l}}, \mathbf{H}_{l}^{\mathsf{L}}, \mathbf{W}_{l} ),  
    \\[-2bp]
    \label{eq:hfc_grad_estimate}
    [ \hat{\nabla}_{\mathbf{H}_{l-1}}, \hat{\nabla}_{\mathbf{W}_{l}} ]  &= \mathrm{backward}( \hat{\nabla}_{\mathbf{H}_{l}}, \mathbf{H}_{l}^{\mathsf{H}}, \mathbf{W}_{l} ).  
\vspace{-0mm}
\end{align}}

We conduct the experiments on the CIFAR-10 dataset.
The implementation details are given in Appendix~\ref{appendix:hyper_pre_exp}.
The top-1 accuracy and memory cost of $\mathrm{LFC}\text{-}\mathrm{ACT}$, $\mathrm{HFC}\text{-}\mathrm{ACT}$, and normal training are shown in Figure~\ref{fig:dct_framework}~(b) and~(c), respectively.
Overall, we have the following observations:
\begin{itemize}[leftmargin=8pt, topsep=0pt]
\setlength{\parskip}{1mm}
\setlength{\parsep}{0pt}
\setlength{\itemsep}{0pt}
\vspace{-1mm}

    \item \textbf{Accuracy:} According to Figure~\ref{fig:dct_framework}~(b), $\mathrm{HFC}\text{-}\mathrm{ACT}$ suffers from significantly more degradation of accuracy than $\mathrm{LFC}\text{-}\mathrm{ACT}$.
    This indicates \emph{DNN backward propagation mainly utilizes the LFC of activation maps}.
    
    \item \textbf{Memory:} According to Figure~\ref{fig:dct_framework}~(c), the storage of HFC requires significantly more memory than that of the LFC,
    i.e., \emph{the storage of HFC consumes the majority of memory.}

\vspace{-1mm}
\end{itemize}

To better understand the results of model accuracy, we theoretically prove the gradient for backward propagation is bounded into a tighter range around the optimal value in $\mathrm{LFC}\text{-}\mathrm{ACT}$. This enables $\mathrm{LFC}\text{-}\mathrm{ACT}$ to learn a more accurate model than $\mathrm{HFC}\text{-}\mathrm{ACT}$.


\subsection{Theoretical Analysis}

We theoretically analyze the gradient estimation error of $\mathrm{LFC}\text{-}\mathrm{ACT}$ and $\mathrm{HFC}\text{-}\mathrm{ACT}$ which adopt
Equations~(\ref{eq:lfc_grad_estimate}) and~(\ref{eq:hfc_grad_estimate}) for backward propagation, respectively.
Formally, for $\mathrm{LFC}\text{-}\mathrm{ACT}$ and $\mathrm{HFC}\text{-}\mathrm{ACT}$, let $\hat{\nabla}_{\mathbf{W}_{l}}^{\mathsf{L}}$ and $\hat{\nabla}_{\mathbf{W}_{l}}^{\mathsf{H}}$ denote the estimated gradient of layer $l$, respectively.
In this way, $||\hat{\nabla}_{\mathbf{W}_{l}}^{\mathsf{L}} \!-\! \nabla_{\mathbf{W}_{l}}||_F$\footnote{\tiny The Frobenius norm of $n\!\times\!n$ matrix $\mathbf{A}$ is given by $||\mathbf{A}||_F = \sqrt{\sum_{i=1}^n\sum_{j=1}^n a_{ij}^2}$.} and $||\hat{\nabla}_{\mathbf{W}_{l}}^{\mathsf{H}} \!-\! \nabla_{\mathbf{W}_{l}}||_F$ indicates the gradient estimation errors, taking the complete gradient $\nabla_{\mathbf{W}_{l}}$ as a reference. 
To compare the distortion of backward propagation in $\mathrm{LFC}\text{-}\mathrm{ACT}$ and $\mathrm{HFC}\text{-}\mathrm{ACT}$, let $\mathrm{GEB}^{\mathsf{L}}_l$ and $\mathrm{GEB}^{\mathsf{H}}_l$ denote the gradient error upper bound~($\mathrm{GEB}$), respectively,  i.e. $||\hat{\nabla}_{\mathbf{W}_{l}}^{\mathsf{L}} \!-\! \nabla_{\mathbf{W}_{l}}||_F \!\leq\! \mathrm{GEB}^{\mathsf{L}}_l$ and $||\hat{\nabla}_{\mathbf{W}_{l}}^{\mathsf{H}} \!-\! \nabla_{\mathbf{W}_{l}}||_F \!\leq\! \mathrm{GEB}^{\mathsf{H}}_l$.
Intuitively, higher $\mathrm{GEB}$ indicates less accurate backward propagation, leading to a less accurate model after the training.
To this end, we give Theorem~\ref{theorem:grad_error} to compare $\mathrm{GEB}^{\mathsf{L}}_l$ and $\mathrm{GEB}^{\mathsf{H}}_l$, where a convolutional layer is considered. The proof is given in Appendix~\ref{appendix:proof_grad_error}.
A similar analysis of $\mathrm{GEB}$ for a linear layer applied to MLPs and Transformers is provided in Appendix~\ref{appendix:linear_grad_error}.
\begin{theorem}
\label{theorem:grad_error}
\vspace{-1mm}
\setlength{\abovedisplayskip}{0mm}
\setlength{\belowdisplayskip}{0mm}
During the backward pass of a convolutional layer~$l$, $\mathrm{GEB}^{\mathsf{L}}_l$ and $\mathrm{GEB}^{\mathsf{H}}_l$ satisfy
\begin{equation}
\begin{aligned}
\setlength{\abovedisplayskip}{0mm}
\setlength{\belowdisplayskip}{0mm}
\label{eq:grad_error}
\mathrm{GEB}^{\mathsf{L}}_l \!-\! \mathrm{GEB}^{\mathsf{H}}_l 
&\!=\! \Big( \alpha_{l,l} || \mathbf{H}_{l-1}^{\mathsf{T}} ||_F \!+\! \beta_l \Big) ( \lambda^{\mathsf{H}}_l \!-\! \lambda^{\mathsf{L}}_l ) 
\\
&\!+ || \mathbf{H}_{l-1}^{\mathsf{T}} ||_F \!\!\!\! \sum_{i=l+1}^L \!\! \alpha_{l,i} ( \lambda^{\mathsf{H}}_i \!-\! \lambda^{\mathsf{L}}_i ) \prod_{j=l}^{i-1} \gamma_j,
\end{aligned}
\end{equation}
where $\alpha_{l,i}, \beta_l, \gamma_l \!>\! 0$ for $1 \!\leq\! l,i \!\leq\! L$ depend on the model weights before backward propagation (given by Equations~(\ref{eq:beta}) in Appendix~\ref{appendix:proof_grad_error}); $\lambda^{\mathsf{L}}_l \!=\! || \widetilde{\mathbf{H}}_l \!\odot\! \mathbf{M} ||_F$; $\lambda^{\mathsf{H}}_l \!=\! || \widetilde{\mathbf{H}}_l \!\odot\! ( \mathbf{1} \!-\! \mathbf{M}) ||_F$; $\widetilde{\mathbf{H}}_l \!=\! \mathrm{DCT}({\mathbf{H}}_l)$; and $\mathbf{M}$ denotes the loss-pass mask given by Equation~(\ref{eq:lfc_estimate}).


\vspace{-1mm}
\end{theorem}


Theorem~\ref{theorem:grad_error} indicates the $\mathrm{GEB}$ difference depends on $\lambda^{\mathsf{H}}_l \!-\! \lambda^{\mathsf{L}}_l$ for $1 \!\leq\! l \!\leq\! L$ during the training.
Following this direction, we estimate $\lambda^{\mathsf{L}}_l$ and $\lambda^{\mathsf{H}}_l$ via $\lambda^{\mathsf{L}}_l \!=\! || \widetilde{\mathbf{H}}_l \odot \mathbf{M} ||_F$ and $\lambda^{\mathsf{H}}_l = || \widetilde{\mathbf{H}}_l \!\odot\! ( \mathbf{1} \!-\! \mathbf{M}) ||_F$ during the training of ResNet-18 and DenseNet-121 on the CIFAR-10 dataset.
Specifically, $\mathbf{H}_l$ takes the activation maps of the BasicBlocks in ResNet-18, and Denseblocks in DenseNet-121.
The estimation of $\lambda^{\mathsf{L}}_l$ and $\lambda^{\mathsf{H}}_l$ is based on the checkpoint of ResNet-18 in epoches 20, 40, and 60, and visualized in Figures~\ref{fig:lfc_vs_hfc}~(a)-(c), respectively.
The implementation details and results of DenseNet-121 are given in Appendix~\ref{appendix:densenet121_lambda}.
It is consistently observed that $\lambda^{\mathsf{L}}_l \!>\! \lambda^{\mathsf{H}}_l$ for different instances and layers.
This leads to $\mathrm{GEB}^{\mathsf{L}}_l \!<\! \mathrm{GEB}^{\mathsf{H}}_l$ according to Theorem~\ref{theorem:grad_error}. 
Therefore, $\mathrm{HFC}\text{-}\mathrm{ACT}$ suffers from a worse distortion of backward propagation during the training, eventually leading to less accurate learned model than $\mathrm{LFC}\text{-}\mathrm{ACT}$.

With both experiments and theoretical analysis, we have proved the HFC of activation maps has less contribution to backward propagation than LFC.
However, according to Figure~\ref{fig:dct_framework}~(c), the HFC takes the majority of memory cost during the training, so it is highly redundant and compressible during the training.
Motivated by this, we propose \Algnameabbr{} to compress the activation maps into a dual precision representation: \emph{high-precision LFC} and \emph{low-precision HFC}.
On the one hand, both LFC and low-precision HFC requires much less memory to cache.
On the other hand, removing redundancy from HFC will not cause much distortion in backward propagation.
In this way, \Algnameabbr{} significantly reduces training memory without affecting model accuracy.

\begin{figure} 
\setlength{\abovecaptionskip}{-0mm}
\setlength{\belowcaptionskip}{-5mm}
    \centering
    \!\!\!\!\!\!\!\!
    \subfigure[Epoch 20.]{
    \centering
    \begin{minipage}[t]{0.33\linewidth}
    \vspace{0pt}
    	\includegraphics[width=1.0\linewidth]{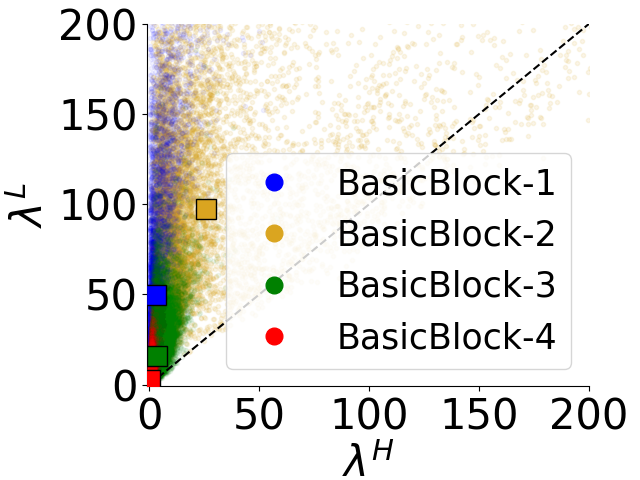}
    \end{minipage}%
    }
    \!\!\!\!\!
    \subfigure[\vspace{-10mm}Epoch 40.]{
    \centering
    \begin{minipage}[t]{0.33\linewidth}
    \vspace{0pt}
    	\includegraphics[width=1.0\linewidth]{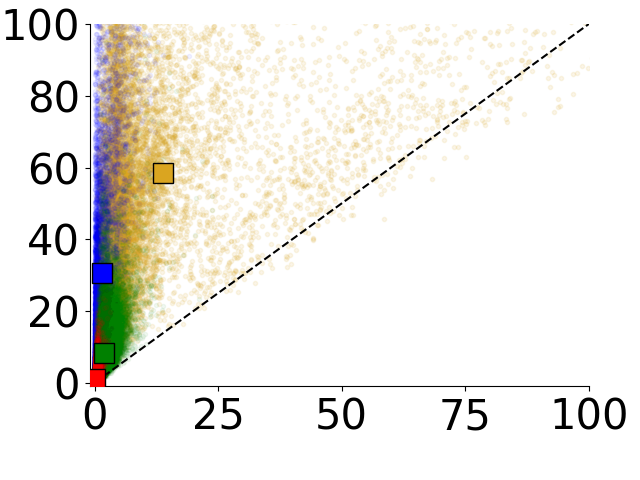}
    \end{minipage}%
    }
    \!\!\!\!\!
    \subfigure[\vspace{-10mm}Epoch 60.]{
    \centering
    \begin{minipage}[t]{0.33\linewidth}
    \vspace{0pt}
    	\includegraphics[width=1.0\linewidth]{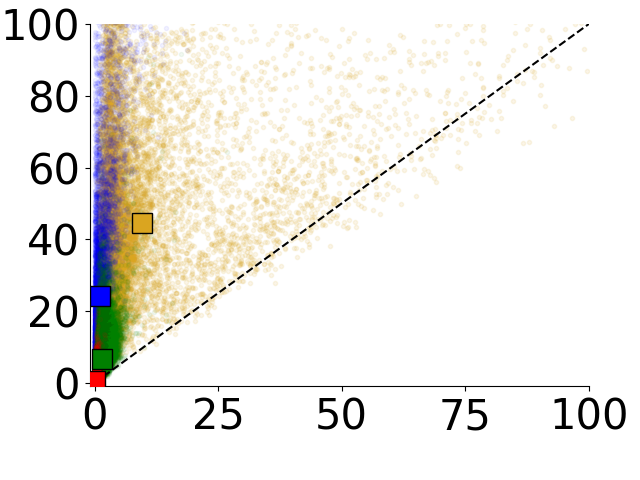}
    \end{minipage}%
    }
\caption{\small \label{fig:lfc_vs_hfc}$\lambda^{\mathsf{L}}_l \!=\! || \widetilde{\mathbf{H}}_l \!\odot\! \mathbf{M} ||_F$ versus $\lambda^{\mathsf{H}}_l \!=\! || \widetilde{\mathbf{H}}_l \!\odot\! ( \mathbf{1} \!-\! \mathbf{M}) ||_F$ in training epoches 20, 40, and 60 of ResNet-18. 
$\mathbf{H}_l$ takes the activation maps of four BasicBlocks in ResNet-18; the y- and x-axis of $\square$ indicates the expectation of $\lambda^{\mathsf{L}}_l$ and $\lambda^{\mathsf{H}}_l$, respectively.
}
\end{figure}




\section{Dual Activation Precision Training}

\begin{figure*} 
\setlength{\abovecaptionskip}{2mm}
\setlength{\belowcaptionskip}{-3mm}
    \centering
    \includegraphics[width=0.95\linewidth]{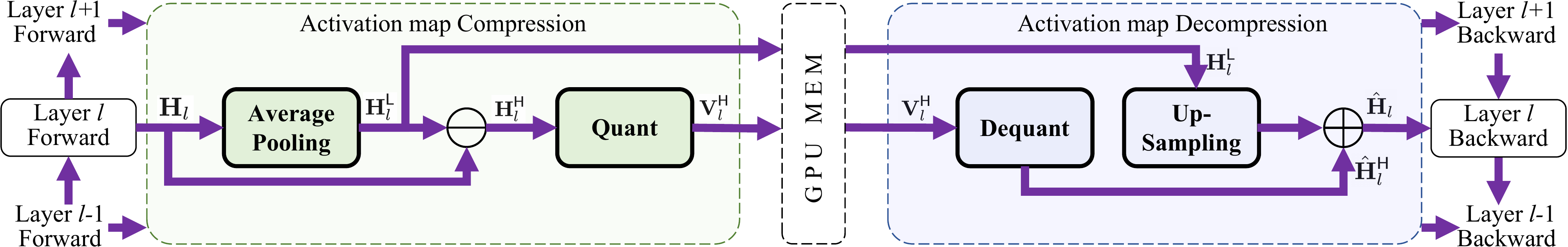}
    \caption{\label{fig:algorithm_framework} \small The proposed framework of Dual Activation Precision Training. }
\end{figure*}

We introduce the proposed \Algnameunderline{}~(\Algnameabbr{}) in this section.
The framework of \Algnameabbr{} is shown in Figure~\ref{fig:algorithm_framework}.
Specifically, after the feed-forward operation of each layer, \Algnameabbr{} estimates the LFC and compresses the HFC into a low-precision copy such that the total memory cost is significantly reduced.
Before the backward propagation of each layer, the low-precision HFC is decompressed and combined with LFC to reconstruct the activation map.
To facilitate illustration, the compression and decompression are formalized based on 2D activation maps in this section.
A similar processing of 1D/3D activation maps is discussed in Appendix~\ref{appendix:1D2D3D}.




\subsection{Activation Map Compression}

To compress the activation map $\mathbf{H}_{l}$ of layer $l$, \Algnameabbr{} estimates the LFC $\mathbf{H}_{l}^{\mathsf{L}}$ and HFC $\mathbf{H}_{l}^{\mathsf{H}}$ using DCT after the feed-forward operation.
However, the high computational complexity of DCT prevents us from directly applying it to real-time algorithms.
We thus give Theorem~\ref{theorem:low_pass} to introduce a moving average operation that can approximate the loss-pass filter. 
The proof is given in Appendix~\ref{appendix:proof_low_pass}.

\begin{theorem}
\vspace{-1mm}
\label{theorem:low_pass}
For any real-valued function $f(x)$ and its moving average $\bar{f}(x) = \frac{1}{2B} \int_{x}^{x+2B} f(t) \mathrm{d}t$, let $F(\omega)$ and $\overline{F}(\omega)$ denote the Fourier transformation of $f(x)$ and $\bar{f}(x)$, respectively.
Generally, we have $\overline{F}(\omega) = H(\omega) F(\omega)$, where $|H(\omega)| = \big| \frac{\sin \omega B}{\omega B} \big|$.
\vspace{-1mm}
\end{theorem}

\begin{remark}
\label{remark:low_pass}
The frequency response of $H(\omega)$ depends on its envelope function $\frac{1}{|\omega B|}$.
Note that $\frac{1}{|\omega B|}$ decreases with $|\omega|$ such that $\frac{1}{|\omega B|} \to 0$ as $\omega \to \infty$. 
Hence, $H(\omega)$ is an approximate loss-pass filter.
\vspace{-2mm}
\end{remark}

According to Remark~\ref{remark:low_pass}, we approximate the LFC $\mathbf{H}_{l}^{\mathsf{L}}$ with the moving average of $\mathbf{H}_{l}$.
Notably, the average pooling operator provides efficient moving average, so
\Algnameabbr{} adopts average pooling to estimate the LFC as $\mathbf{H}_{l}^{\mathsf{L}} \!=\! \mathrm{Average Pooling}(\mathbf{H}_{l})$.
The value of block-size and moving stride is a unified hyper-parameter $B$, which controls the memory of $\mathbf{H}_{l}^{\mathsf{L}}$\footnote{\footnotesize For the case $N \!<\! B$, the pooling block-size and stride will be $N$ such that the shape of $\mathbf{H}_{l}^{\mathsf{L}}$ is $\mathrm{Minibatch} \!\times\! \mathrm{Channel} \!\times\!\! 1 \!\!\times\!\! 1$.}.
Moreover, $\mathbf{H}_{l}^{\mathsf{L}}$ is cached in the format of $\texttt{bfloat16}$ for saving the memory.
In our experiments, we found $B \!=\! 8$ can provide representative LFC for backward propagation, where the memory cost of $\mathbf{H}_{l}^{\mathsf{L}}$ is only $0.8\%$ of $\mathbf{H}_{l}$.


To estimate the HFC, \Algnameabbr{} calculates the residual given by $\mathbf{H}_{l}^{\mathsf{H}} \!=\! \mathbf{H}_{l} \!-\! \mathrm{UpSampling}(\mathbf{H}_{l}^{\mathsf{L}})$, where the up-sampling operation enlarges $\mathbf{H}_{l}^{\mathsf{L}}$ to shape $\mathrm{Minibatch} \!\times\! \mathrm{Channel} \!\times\! N \!\times\! N$ via nearest interpolation. 
Then, \Algnameabbr{} compress $\mathbf{H}_{l}^{\mathsf{H}}$ into low-precision because it plays a less important role during the backward propagation but consumes most of the memory.
Specifically, \Algnameabbr{} adopts $Q$-bit per-channel quantization\footnote{\footnotesize A fixed bit-width is adopted for the quantization of all layers to maximize the efficiency of data processing.}\footnote{\footnotesize Per-channel quantization is more efficient and light than per-group quantization in state-of-the-art work.} for the compression, where the bit-width $Q$ controls the precision and memory cost of HFC after the compression.
Let $\mathbf{V}_{l}^{\mathsf{H}}$ denote a $Q$-bit integer matrix, as the low-precision representation of $\mathbf{H}_{l}^{\mathsf{H}}$.
The procedure of compressing $\mathbf{H}_{l}^{\mathsf{H}}$ into $\mathbf{V}_{l}^{\mathsf{H}}$ is given by
\begin{equation}
\setlength{\abovedisplayskip}{1mm}
\setlength{\belowdisplayskip}{1mm}
\label{eq:quantization}
\mathbf{V}_{l}^{\mathsf{H}} = \mathrm{Quant} (\mathbf{H}_{l}^{\mathsf{H}}) = \big\lfloor \Delta_{l}^{-1} (\mathbf{H}_{l}^{\mathsf{H}} - \delta_{l}) \big\rceil,
\end{equation}
where $\delta_{l}$ denotes the minimum element in $\mathbf{H}^{\mathsf{H}}_{l}$; $\Delta_{l} = (h_{\max}-\delta_{l})/(2^Q-1)$ denotes the quantization step; $h_{\max}$ denotes the maximum element in $\mathbf{H}^{\mathsf{H}}_{l}$; $\lfloor \bullet \rceil$ denotes the \emph{stochastic rounding}\footnote{\footnotesize $\lfloor x \rceil$ takes the value of $\lfloor x \rfloor$ with a probability of $x\!-\!\lfloor x \rfloor$ and takes $\lceil x \rceil$ with a probability of $\lceil x \rceil \!-\! x$.}\footnote{\footnotesize The stochastic rounding enables the pipeline of quantization and dequantization to be unbiased, i.e. $\mathbb{E}[\mathbf{V}_{l}^{\mathsf{H}}] \!=\! \mathbf{H}_{l}^{\mathsf{H}}$.}~\cite{gupta2015deep};
and $\delta_{l}$ and $\Delta_{l}$ are cached in the formate of $\texttt{bfloat16}$ for saving memory.
In this way, the memory cost of $(\mathbf{V}_{l}^{\mathsf{H}}, \delta_{l}, h_{\max})$ is $(N^2 Q / 8 + 4)$ bytes per channel, in contrast with that of $\mathbf{H}_{l}$ being $4 N^2$ bytes per channel.
In our experiments, we found $Q \!=\! 2$ can provide enough representation for backward propagation, where the memory cost of $\mathbf{V}_{l}^{\mathsf{H}}$ is only $8.3\%$ of $\mathbf{H}_{l}$.

After the compression, as the representation of $\mathbf{H}_l$, the tuple of $( \mathbf{H}_{l}^{\mathsf{L}}, \mathbf{V}_{l}^{\mathsf{H}}, \Delta_{l}, \delta_{l} )$ is cached to the memory for reconstructing the activation maps during the backward pass.




\subsection{Activation Map Decompression}  

During the backward pass, \Algnameabbr{} adopts the cached tuples of $\{( \mathbf{H}_{l}^{\mathsf{L}}, \mathbf{V}_{l}^{\mathsf{H}}, \Delta_{l}, \delta_{l} ) \!\mid\! 0 \leq l \leq L\!-\!1 \}$ to reconstruct the activation map layer-by-layer.
Specifically, for each layer $l$, \Algnameabbr{} dequantizes the HFC via $\hat{\mathbf{H}}_{l}^{\mathsf{H}} \!=\! \Delta_{l} \mathbf{V}_{l}^{\mathsf{H}} \!+\! \delta_{l}$, which is the inverse process of Equation~(\ref{eq:quantization}). 
Then, the activation map is reconstructed via
\begin{equation}
\setlength{\abovedisplayskip}{1mm}
\setlength{\belowdisplayskip}{1mm}
\hat{\mathbf{H}}_{l} = \mathrm{UpSampling}(\mathbf{H}_{l}^{\mathsf{L}}) + \hat{\mathbf{H}}_{l}^{\mathsf{H}},
\end{equation}
where $\mathrm{UpSampling}(\cdot)$ enlarges $\mathbf{H}_{l}^{\mathsf{L}}$ to the shape of $\mathrm{Minibatch} \!\!\times\!\! \mathrm{Channel} \!\!\times\!\! N \!\!\times\!\! N$ via nearest interpolation. 
After the decompression, \Algnameabbr{} frees the caching of $( \mathbf{H}_{l}^{\mathsf{L}}, \mathbf{V}_{l}^{\mathsf{H}}, \Delta_{l}, \delta_{l} )$, and takes $\hat{\mathbf{H}}_{l}$ into $[\hat{\nabla}_{\mathbf{H}_{l-1}}, \hat{\nabla}_{\mathbf{W}_{l}}] = \mathrm{backward}( \hat{\nabla}_{\mathbf{H}_{l}}, \hat{\mathbf{H}}_{l-1}, \mathbf{W}_{l} )$ to estimate the gradient for backward propagation.


\begin{algorithm} 
\caption{\small Mini-batch updating of \Algnameabbr{}}
\label{alg:division}
\small
\textbf{Input:} Mini-batch samples $\mathbf{x}$ and labels $\mathrm{y}$. \\
\textbf{Output:} Weight and bias $\{ \mathbf{W}_{l}, \mathbf{B}_{l} | 1 \!\leq\! l \!\leq\! L \}$.\\
\vspace{-4mm}
\begin{algorithmic}[1]


\FOR{\emph{layer} $l := 1 \text{ to } L$}  

\STATE $\mathbf{H}_{l} = f(\mathbf{W}_l \mathbf{H}_{l-1} + \mathbf{B}_l)$ \textcolor{gray}{// $\mathbf{H}_0 = \mathbf{x}$} 

\STATE $\mathbf{H}_{l-1}^{\mathsf{L}} = \mathrm{Average Pooling}(\mathbf{H}_{l-1})$ 

\STATE $\mathbf{H}_{l-1}^{\mathsf{H}} = \mathbf{H}_{l-1} - \mathrm{UpSampling}(\mathbf{H}_{l-1}^{\mathsf{L}}$) 

\STATE $\mathbf{V}_{l-1}^{\mathsf{H}}, \Delta_{l-1}, \delta_{l-1} = \mathrm{Quant}( \mathbf{H}_{l-1}^{\mathsf{H}} )$ 

\STATE Cache $(\mathbf{H}_{l-1}^{\mathsf{L}}, \mathbf{V}_{l-1}^{\mathsf{H}}, \Delta_{l-1}, \delta_{l-1})$ 

\ENDFOR


\STATE Estimate the loss value and gradient $\hat{\nabla}_{\mathbf{H}_{L}}$.

\FOR{\emph{layer} $l := L \text{ to } 1$}

\STATE $\hat{\mathbf{H}}_{l-1}^{\mathsf{H}} = \mathrm{Dequant}(\mathbf{V}_{l-1}^{\mathsf{H}}, \Delta_{l-1}, \delta_{l-1})$ 

\STATE $\hat{\mathbf{H}}_{l-1} = \mathrm{UpSampling}(\mathbf{H}_{l-1}^{\mathsf{L}}) + \hat{\mathbf{H}}_{l-1}^{\mathsf{H}}$ 

\STATE Estimate $[\hat{\nabla}_{\mathbf{H}_{l-1}}, \hat{\nabla}_{\mathbf{W}_{l}}]$ and update $\mathbf{W}_{l}$.

\STATE Free $(\mathbf{H}_{l-1}^{\mathsf{L}}, \mathbf{V}_{l-1}^{\mathsf{H}}, \Delta_{l-1}, \delta_{l-1})$.




\ENDFOR

\end{algorithmic}
\end{algorithm}

\begin{figure*} 
\setlength{\abovecaptionskip}{0mm}
\setlength{\belowcaptionskip}{-5mm}
    \centering
    \subfigure[\Algnameabbr{} vs BLPA]{
    \centering
    \begin{minipage}[t]{0.19\linewidth}
    	\includegraphics[width=1.0\linewidth]{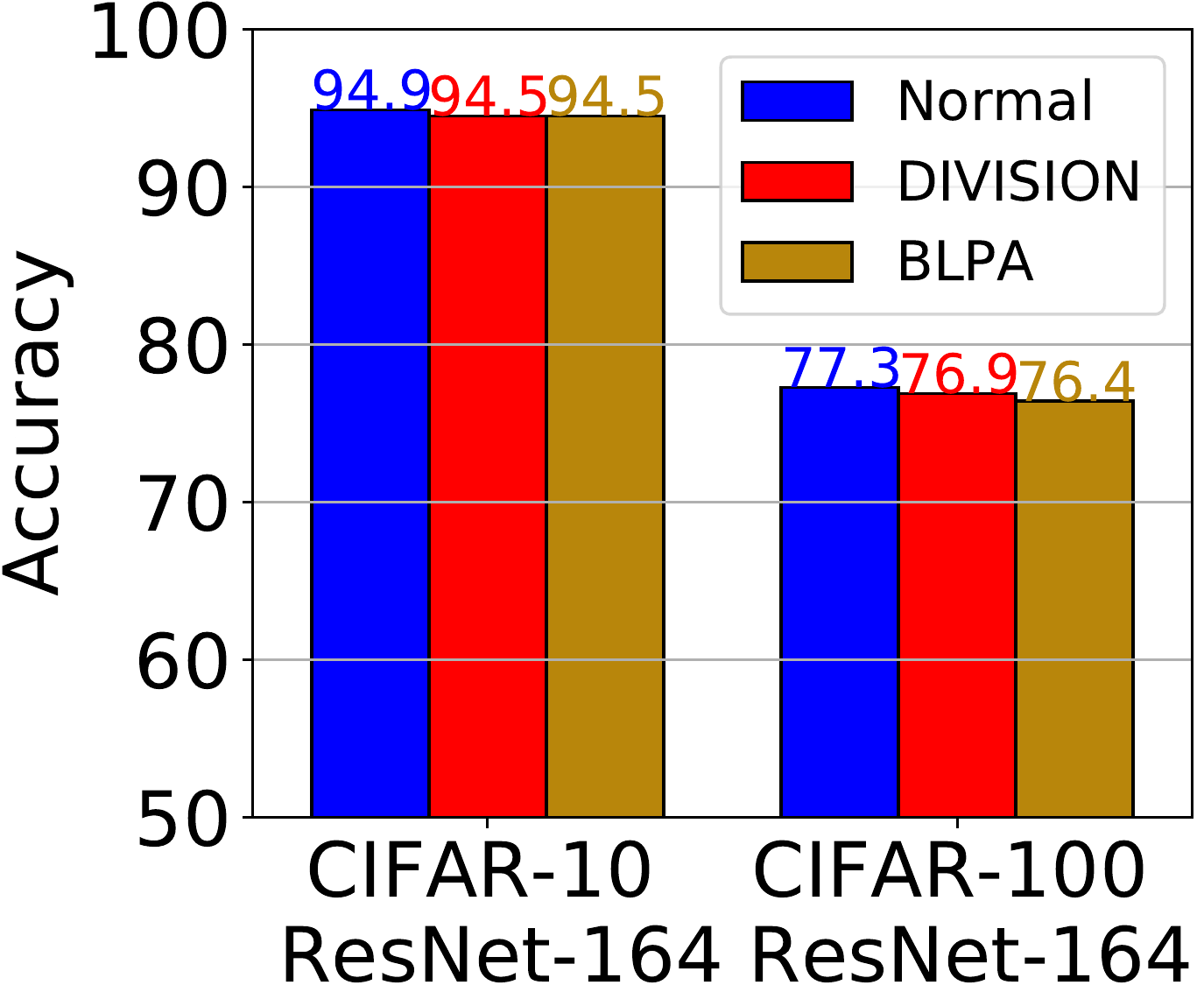}
    \end{minipage}%
    \quad
    }
    \subfigure[\Algnameabbr{} vs AC-GC]{
    \centering
    \begin{minipage}[t]{0.26\linewidth}
    	\includegraphics[width=1.0\linewidth]{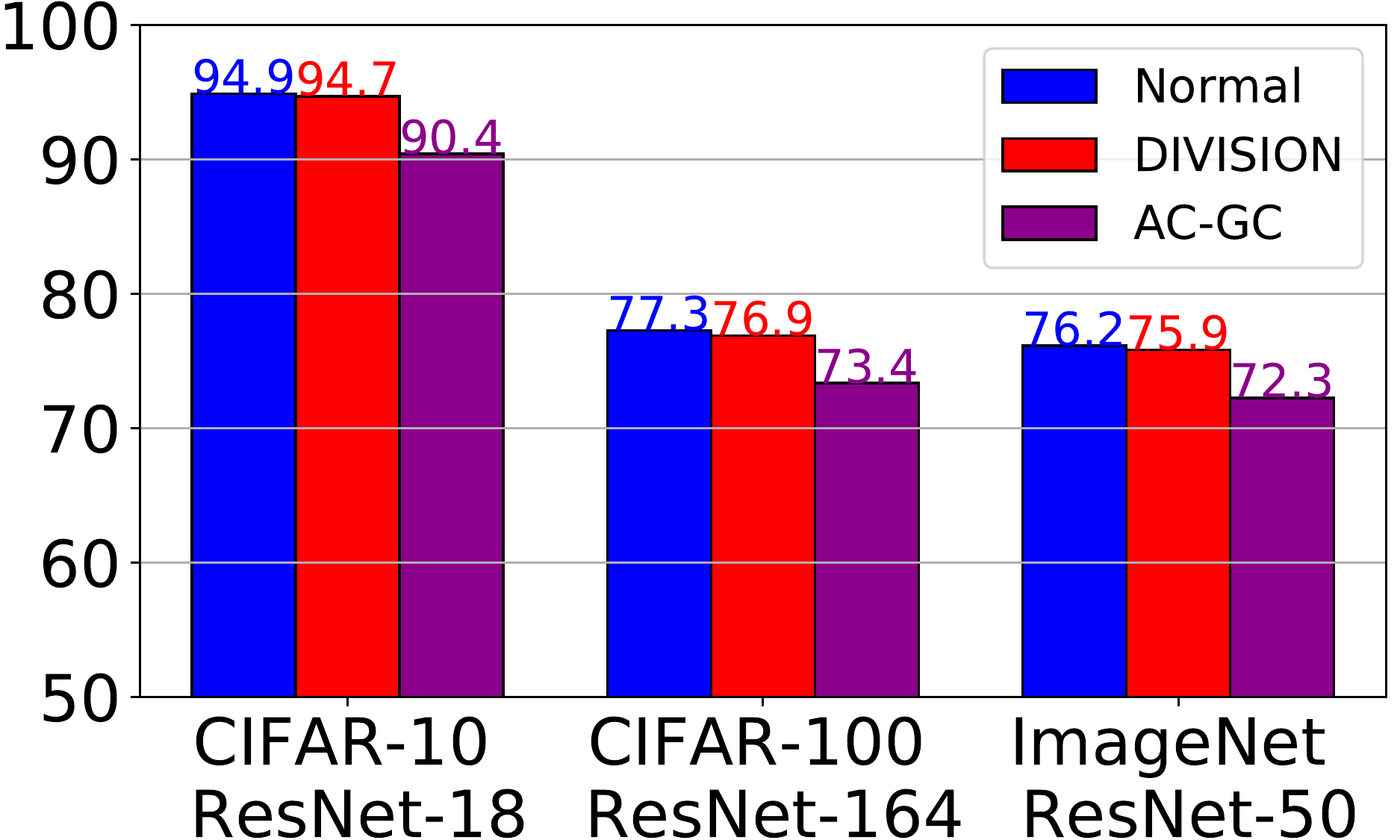}
    \end{minipage}%
    \quad
    }
    \subfigure[\Algnameabbr{} vs ActNN]{
    \centering
    \begin{minipage}[t]{0.35\linewidth}
    	\includegraphics[width=1.0\linewidth]{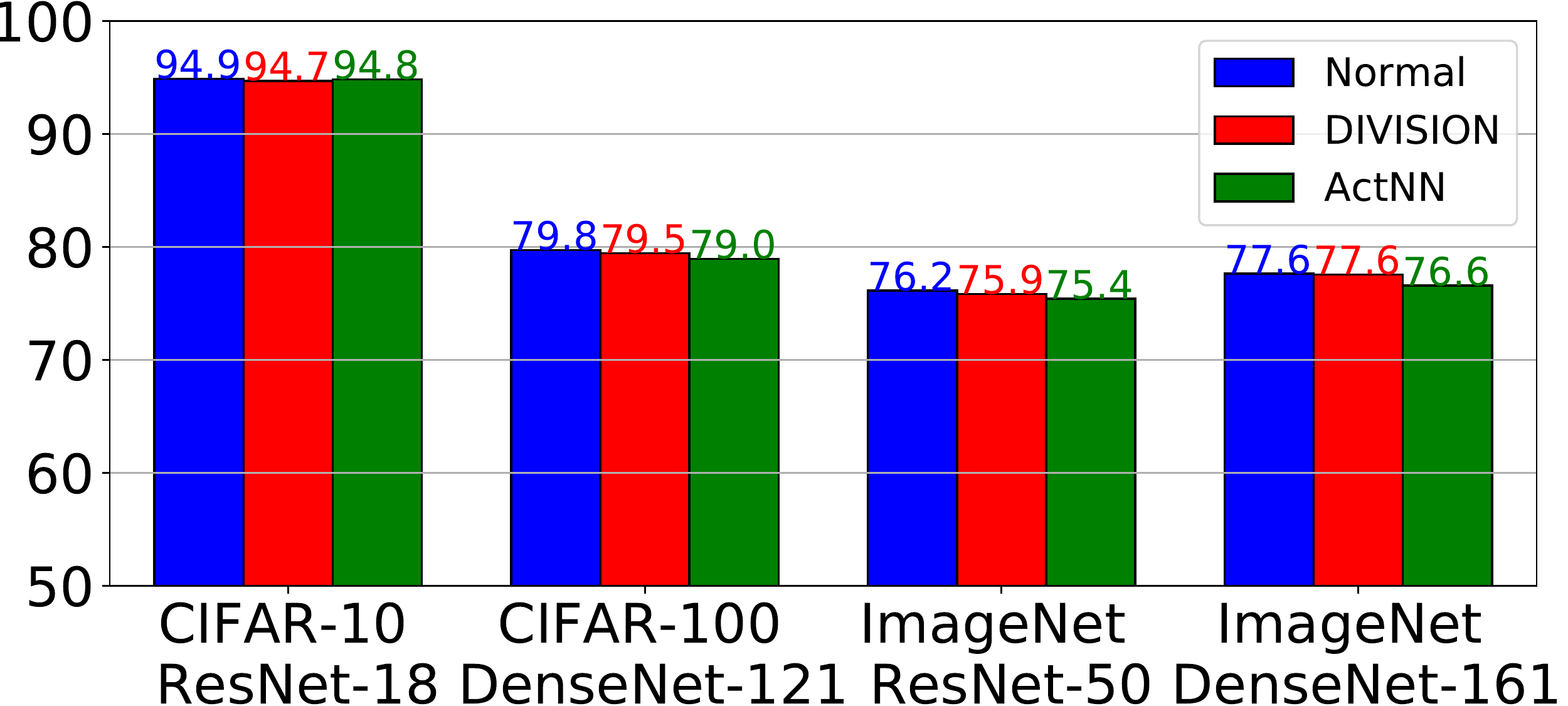}
    \end{minipage}%
    }
\caption{\label{fig:acc_div} \small Top-1 accuracy (\%) $\uparrow$ of Normal training, \Algnameabbr{}, BLPA~(a), AC-GC~(b), and ActNN~(c).
}
\end{figure*}

\vspace{-2mm}
\subsection{The Algorithm of \Algnameabbr{}}


Algorithm~\ref{alg:division} presents a mini-batch updating of \Algnameabbr{}, which includes a forward pass and backward pass.
During the forward pass of each layer, \Algnameabbr{} first forwards the exact activation map to the next layer~(line 2); then, estimates the LFC and HFC~(line 3-4); after this, achieves the low precision copy of HFC~(lines 5); finally caches the representation to the memory~(line 6).
During the backward pass of each layer, 
\Algnameabbr{} first decompresses the HFC~(line 10); reconstructs the activation map~(line 11); estimates the gradients and updates the weights of layer $l$~(line 12); finally frees the caching of $( \mathbf{H}_{l\!-\!1}^{\mathsf{L}}, \! \mathbf{V}_{l\!-\!1}^{\mathsf{H}}, \! \Delta_{l\!-\!1}, \! \delta_{l\!-\!1} )$~(line 13).
For each mini-batch updating, the memory usage reaches the maximum value after the forward pass~(caching the representation of activation maps layer-by-layer), and reduces to the minimum value after the backward pass~(freeing the cache layer-by-layer).
Existing work~\cite{chen2021actnn} estimates the memory cost of activation maps by
\begin{equation}
\setlength{\abovedisplayskip}{1mm}
\setlength{\belowdisplayskip}{1mm}
\text{MEM Cost} = \text{MEM Util}_{\text{after forward}} - \text{MEM Util}_{\text{after backward}},
\nonumber
\end{equation}
where existing deep learning tools provide APIs\footnote{\tiny {\ttfamily torch.cuda.memory\_allocated} returns the memory occupied by tensors in bytes.} to estimate the memory utilization.

The theoretical compression rate $R$ of \Algnameabbr{} is given in Appendix~\ref{appendix:compression_rate}, where general cases of convolutional neural networks and multi-layer perception are considered for the estimation.
For the model architectures in our experiments, we have $R_{\text{ResNet-50}}, R_{\text{WRN-50-2}} \geq 10.35$.

\section{Evaluation of \Algnameabbr{}}

We conduct experiments to evaluate \Algnameabbr{} by answering the following research questions.
\textbf{RQ1:} How does \Algnameabbr{} perform compared with state-of-the-art baseline methods in terms of the model accuracy, memory cost, and training throughput?
\textbf{RQ2:} Does the strategy of dual-precision compression contribute to \Algnameabbr{}?
\textbf{RQ3:} What is the effect of hyper-parameters on \Algnameabbr{}?

\subsection{Experiment Setup}

The experiment settings including the datasets, baseline methods and DNN architectures are given in Appendix~\ref{appendix:exp_set}.
The implementation details and configuration of computational infrastructure are given in Appendix~\ref{appendix:config_division} and \ref{appendix:infrastructure}, respectively.
Experiments on MLPs are given in Appendix~\ref{appendix:exp_mlp}.









\begin{table*} 
\setlength{\abovecaptionskip}{0mm}
\setlength{\belowcaptionskip}{0mm}
\caption{\label{tab:exp_mem} (a) \small Memory cost~$\!\downarrow$ and compression rate~$\!\uparrow$. 
\emph{Total Mem} refers to total memory cost of weights, optimizer, data and activation maps. 
\emph{Act Mem} refers to memory cost of activation maps. 
(b) Performance of \Algnameabbr{}, Checkpoint, and Mesa on the Swin-Transformer.
(c) Model accuracy with fixed bit-width quantization of \Algnameabbr{} w/o LFC, \Algnameabbr{} w/o HFC, and \Algnameabbr{}.} 
\scriptsize
\begin{minipage}[h]{0.68\linewidth}
\subtable[]{
\begin{tabular}{c|c|cccc|cccc|}
    \toprule[1pt]
         \multicolumn{2}{c|}{Architecture} &  \multicolumn{4}{c|}{ResNet-50} & \multicolumn{4}{c|}{WRN-50-2} \\
    \hline
        \multicolumn{2}{c|}{Batch-size} & 64 & 128 & 256 & 512 & 64 & 128 & 256 & 512 \\
    \hline
        \multirow{6}{*}{\makecell[c]{\!\!\!\!Total \!\!\!\! \\ \!\!\!\!Mem\!\!\!\! \\ (GB)}}
        & Normal & 5.46 & 10.62 & 20.92 & \!\!\!\!\!\!\!\!\!\!\emph{OOM}\!\!\!\!\!\!\!\!\!\! & 7.52 & 14.23 & 27.68 & \!\!\!\!\!\!\!\!\!\!\emph{OOM}\!\!\!\!\!\!\!\!\!\! \\
        \cline{2-10}

        & \!\!\!\!Checkpoint\!\!\!\! & \!\!\!\!2.57~(2.1\!$\times$\!)\!\!\!\! & \!\!\!\!4.84~(2.2\!$\times$\!)\!\!\!\! & \!\!\!\!9.39~(2.2\!$\times$\!)\!\!\!\! & \!\!\!\!\!\!\!\!18.49\!\!\!\!\!\!\!\! & \!\!\!\!3.05~(2.5\!$\times$\!)\!\!\!\! & \!\!\!\!5.33~(2.7\!$\times$\!)\!\!\!\! & \!\!\!\!9.88~(2.8\!$\times$\!)\!\!\!\! & \!\!\!\!\!\!\!\!\!\!\emph{OOM}\!\!\!\!\!\!\!\!\!\! \\
        \cline{2-10}

        & BLPA & \!\!\!\!1.15~(4.7\!$\times$\!)\!\!\!\! & \!\!\!\!2.01~(5.3\!$\times$\!)\!\!\!\! & \!\!\!\!3.72~(5.6\!$\times$\!)\!\!\!\! & \!\!\!\!\!\!\!\!7.14\!\!\!\!\!\!\!\! & \!\!\!\!1.87~(4.0\!$\times$\!)\!\!\!\! & \!\!\!\!2.96~(4.8\!$\times$\!)\!\!\!\! & \!\!\!\!5.15~(5.4\!$\times$\!)\!\!\!\! & \!\!\!\!\!\!\!\!9.51\!\!\!\!\!\!\!\! \\
        \cline{2-10}
        
        & AC-GC & \!\!\!\!1.80~(3.0\!$\times$\!)\!\!\!\! & \!\!\!\!3.31~(3.2\!$\times$\!)\!\!\!\! & \!\!\!\!6.31~(3.3\!$\times$\!)\!\!\!\! & \!\!\!\!\!\!\!\!12.33\!\!\!\!\!\!\!\! & \!\!\!\!2.72~(2.8\!$\times$\!)\!\!\!\! & \!\!\!\!4.66~(3.1\!$\times$\!)\!\!\!\! & \!\!\!\!8.53~(3.2\!$\times$\!)\!\!\!\! & \!\!\!\!\!\!\!\!16.27\!\!\!\!\!\!\!\! \\
        \cline{2-10}

        & ActNN & \textbf{\!\!\!\!0.81~(6.7\!$\times$\!)\!\!\!\!} & \textbf{\!\!\!\!1.34~(7.9\!$\times$\!)\!\!\!\!} & \textbf{\!\!\!\!2.39~(8.8\!$\times$\!)\!\!\!\!} & \textbf{\!\!\!\!\!\!\!\!4.47\!\!\!\!\!\!\!\!} & \textbf{\!\!\!\!1.44~(5.2\!$\times$\!)\!\!\!\!} & \textbf{\!\!\!\!2.09~(6.8\!$\times$\!)\!\!\!\!} & \textbf{\!\!\!\!3.41~(8.1\!$\times$\!)\!\!\!\!} & \!\!\!\!\!\!\!\!\textbf{6.03}\!\!\!\!\!\!\!\! \\
        \cline{2-10}

        & \!\!\!\!\Algnameabbr{}\!\!\!\! & \textbf{\!\!\!\!0.82~(6.7\!$\times$\!)\!\!\!\!} & \textbf{\!\!\!\!1.35~(7.9\!$\times$\!)\!\!\!\!} & \textbf{\!\!\!\!2.41~(8.7\!$\times$\!)\!\!\!\!} & \textbf{\!\!\!\!\!\!\!\!4.52\!\!\!\!\!\!\!\!} & \textbf{\!\!\!\!1.45~(5.2\!$\times$\!)\!\!\!\!} & \textbf{\!\!\!\!2.12~(6.7\!$\times$\!)\!\!\!\!} & \textbf{\!\!\!\!3.44~(8.0\!$\times$\!)\!\!\!\!} & \!\!\!\!\!\!\!\!\textbf{6.08}\!\!\!\!\!\!\!\! \\

        \hline

        \multirow{6}{*}{\makecell[c]{\!\!\!\!Act. \!\!\!\! \\ \!\!\!\! Mem \!\!\!\! \\ (GB)}}
        & Normal & 5.14 & 10.25 & 20.48 & \!\!\!\!\!\!\!\!\!\!\emph{OOM}\!\!\!\!\!\!\!\!\!\! & 6.70 & 13.38 & 26.75 & \!\!\!\!\!\!\!\!\!\!\emph{OOM}\!\!\!\!\!\!\!\!\!\! \\
        \cline{2-10}

        & \!\!\!\!Checkpoint\!\!\!\! & \!\!\!\!2.24~(2.3\!$\times$\!)\!\!\!\! & \!\!\!\!4.48~(2.3\!$\times$\!)\!\!\!\! & \!\!\!\!8.95~(2.3\!$\times$\!)\!\!\!\! & \!\!\!\!\!\!\!\!17.90\!\!\!\!\!\!\!\! & \!\!\!\!2.24~(3.0\!$\times$\!)\!\!\!\! & \!\!\!\!4.48~(3.0\!$\times$\!)\!\!\!\! & \!\!\!\!8.95~(3.0\!$\times$\!)\!\!\!\! & \!\!\!\!\!\!\!\!\!\!\emph{OOM}\!\!\!\!\!\!\!\!\!\! \\
        \cline{2-10}

        & BLPA & \!\!\!\!0.82~(6.3\!$\times$\!)\!\!\!\! & \!\!\!\!1.64~(6.2\!$\times$\!)\!\!\!\! & \!\!\!\!3.28~(6.2\!$\times$\!)\!\!\!\! & \!\!\!\!\!\!\!\!6.56\!\!\!\!\!\!\!\! & \!\!\!\!1.06~(6.3\!$\times$\!)\!\!\!\! & \!\!\!\!2.11~(6.3\!$\times$\!)\!\!\!\! & \!\!\!\!4.22~(6.3\!$\times$\!)\!\!\!\! & \!\!\!\!\!\!\!\!8.44\!\!\!\!\!\!\!\! \\
        \cline{2-10}
        
        & AC-GC & \!\!\!\!1.47~(3.5\!$\times$\!)\!\!\!\! & \!\!\!\!2.94~(3.5\!$\times$\!)\!\!\!\! & \!\!\!\!5.88~(3.5\!$\times$\!)\!\!\!\! & \!\!\!\!\!\!\!\!11.75\!\!\!\!\!\!\!\! & \!\!\!\!1.91~(3.5\!$\times$\!)\!\!\!\! & \!\!\!\!3.81~(3.5\!$\times$\!)\!\!\!\! & \!\!\!\!7.61~(3.5\!$\times$\!)\!\!\!\! & \!\!\!\!\!\!\!\!15.20\!\!\!\!\!\!\!\! \\
        \cline{2-10}

        & ActNN & \textbf{\!\!\!\!0.49~(10.5\!$\times$\!)\!\!\!\!} & \textbf{\!\!\!\!0.97~(10.6\!$\times$\!)\!\!\!\!} & \textbf{\!\!\!\!1.94~(10.6\!$\times$\!)\!\!\!\!} & \textbf{\!\!\!\!\!\!\!\!3.89\!\!\!\!\!\!\!\!} & \textbf{\!\!\!\!0.62~(10.8\!$\times$\!)\!\!\!\!} & \textbf{\!\!\!\!1.25~(10.7\!$\times$\!)\!\!\!\!} & \textbf{\!\!\!\!2.49~(10.7\!$\times$\!)\!\!\!\!} & \!\!\!\!\!\!\!\!\textbf{4.97}\!\!\!\!\!\!\!\! \\
        \cline{2-10}

        & \!\!\!\!\Algnameabbr{}\!\!\!\! & \textbf{\!\!\!\!0.49~(10.5\!$\times$\!)\!\!\!\!} & \textbf{\!\!\!\!0.99~(10.4\!$\times$\!)\!\!\!\!} & \textbf{\!\!\!\!1.97~(10.4\!$\times$\!)\!\!\!\!} & \textbf{\!\!\!\!\!\!\!\!3.94\!\!\!\!\!\!\!\!} & \textbf{\!\!\!\!0.64~(10.5\!$\times$\!)\!\!\!\!} & \textbf{\!\!\!\!1.27~(10.5\!$\times$\!)\!\!\!\!} & \textbf{\!\!\!\!2.52~(10.6\!$\times$\!)\!\!\!\!} & \!\!\!\!\!\!\!\!\textbf{5.02}\!\!\!\!\!\!\!\! \\
    \bottomrule[1pt]
    \end{tabular}}
    \end{minipage}
    \begin{minipage}[h]{0.25\linewidth}
    \begin{minipage}[h]{0.25\linewidth}
    \centering
    \subtable[]{
    \begin{tabular}{|l|c|c|c}
    \toprule[1pt]
         & \!\!\!\! Acc (\%) \!\!\!\! & \!\!\!\! Mem. (GB)\!\!\!\! & \!\!\! Throughput \!\!\! \\
    \hline
       \!\!\!\! Normal \!\!\!\! & 81.2 & \!\!\!\!14.43~(1$\times$)\!\!\!\! & 233.53 ips \\ 
       \!\!\!\! Mesa \!\!\!\! & \textbf{81.3} & \!\!\!\!6.56~(2.20$\times$)\!\!\!\! & 136.70 ips \\
       \!\!\!\! Checkpoint \!\!\!\! & 81.2 & \!\!\!\!6.72~(2.15$\times$)\!\!\!\! & \textbf{201.73} ips \\
       \!\!\!\! \Algnameabbr{} \!\!\!\! & 81.0 & \!\!\!\!\textbf{5.07~(2.85$\times$)}\!\!\!\! & 175.48 ips \\
    \bottomrule[1pt]
    \end{tabular}}
    \end{minipage}
    $\\$
    \begin{minipage}[h]{0.25\linewidth}
    \centering
    \subtable[]{
    \begin{tabular}{|l|c|c}
    \toprule[1pt]
         & \quad CIFAR-100 \quad & \quad ImageNet \quad \\
    \hline
       Fixed-4bit & 75.07 & 76.05 \\
       Fixed-2bit & 1 & 0.1 \\
       w/o LFC & 60.54 & 7.79 \\
       w/o HFC & 69.56 & 28.51 \\
       \Algnameabbr{} & 76.3 & 75.9 \\
    \bottomrule[1pt]
    \end{tabular}}
    \end{minipage}
    \end{minipage}
    
\vspace{-6mm}
\end{table*}


\vspace{-2mm}
\subsection{Evaluation by Model Accuracy~(RQ1)}
\vspace{-2mm}
\label{sec:exp_acc}

In this section, we evaluate the training methods in terms of model accuracy on the CIFAR-10, CIFAR-100 and ImageNet datasets. 
Specifically, \Algnameabbr{} is compared with BLPA~\cite{chakrabarti2019backprop}, AC-GC~\cite{evans2021ac} and ActNN~\cite{chen2021actnn} in Figure~\ref{fig:acc_div}~(a)-(c), respectively, where different model architectures are considered.
We do not consider Checkpoint and SWAP in this section because they 
can reduce the training memory without degradation of model accuracy.
Overall, we have the following observations:

\vspace{-2mm}
\begin{itemize}[leftmargin=10pt, topsep=0pt]
\setlength{\parskip}{1mm}
\setlength{\parsep}{0mm}
\setlength{\itemsep}{0mm}

\item \textbf{\Algnameabbr{} vs Baseline Methods:} Compared with normal training, \Algnameabbr{} achieves almost the same accuracy, which indicates a loss-less compression of training.
In contrast, BLPA and ActNN show a slight degradation; and AC-GC has lower accuracy than other methods.

\item \textbf{Flexibility of \Algnameabbr:} \Algnameabbr{} consistently achieves competitive model accuracy in the training of different architectures on different datasets. This indicates \Algnameabbr{} is a flexible framework that can be applied to different architectures and datasets.


\item \textbf{Compressibility of HFC:} Note that \Algnameabbr{} adopts a significantly high compression rate 12$\times$ for the HFC during the training, and achieves nearly loss-less accuracy.
This result indicates the HFC of activation map is highly redundant and compressible during the training.


\end{itemize}

\subsection{Evaluation by Memory Cost~(RQ1)}
\label{sec:exp_mem}

We evaluate different training methods in terms of the training memory cost on the ImageNet datasest (the configuration of our computational infrastructure is given in Appendix~\ref{appendix:infrastructure}).
Table~\ref{tab:exp_mem}~(a) indicates the training memory cost and practical compression rate of \Algnameabbr{} and baseline methods.
Overall, we have the following observations:
\begin{itemize}[leftmargin=10pt, topsep=0pt]
\setlength{\parskip}{1mm}
\setlength{\parsep}{0mm}
\setlength{\itemsep}{0mm} 

\item \textbf{\Algnameabbr{} vs Checkpoint \& BLPA:}
Checkpoint shows less effective compression because it caches some key activation maps to reconstruct other activation maps during the training. 
BLPA has lower compression rate than \Algnameabbr{} because it relies on at least 4-bit compression.






\item \textbf{\Algnameabbr{} vs AC-GC:} 
AC-GC searches the bit-width from an initial maximum value, and finalizes the bit-width as 7.01.
The compression rate of activation maps in the last epoch is $3.5\times$, which is lower than \Algnameabbr{}.

\item \textbf{\Algnameabbr{} vs ActNN:} \Algnameabbr{} has approximately the same memory cost as ActNN.
Beyond the storage of 2-bit activation maps, \Algnameabbr{} has overhead for caching the LFC; and ActNN spends almost equal overhead for storing the parameters of per-group quantization.

\item \textbf{Compression Rate:} The activation map compression rate of \Algnameabbr{} is consistent with the theoretical results ($R_{\text{ResNet-50}}, R_{\text{WRN-50-2}} \geq 10.35$, see Appendix~\ref{appendix:compression_rate}), which is not influenced by the mini-batch size.
Moreover, the overall compression rate grows with the mini-batch size.

\item \textbf{Activation Maps:} For the normal training, the storage of activation maps takes the majority of memory cost ($>\!$ 90\%, growing with the mini-batch size), which is consistent with our discussion in Section~\ref{sec:introduction}.

    
    
\end{itemize}

\begin{figure} 
\setlength{\abovecaptionskip}{-0mm}
\setlength{\belowcaptionskip}{-5mm}
    \centering
    \subfigure[ImageNet.]{
    \centering
    \begin{minipage}[t]{0.48\linewidth}
    	\includegraphics[width=1.0\linewidth]{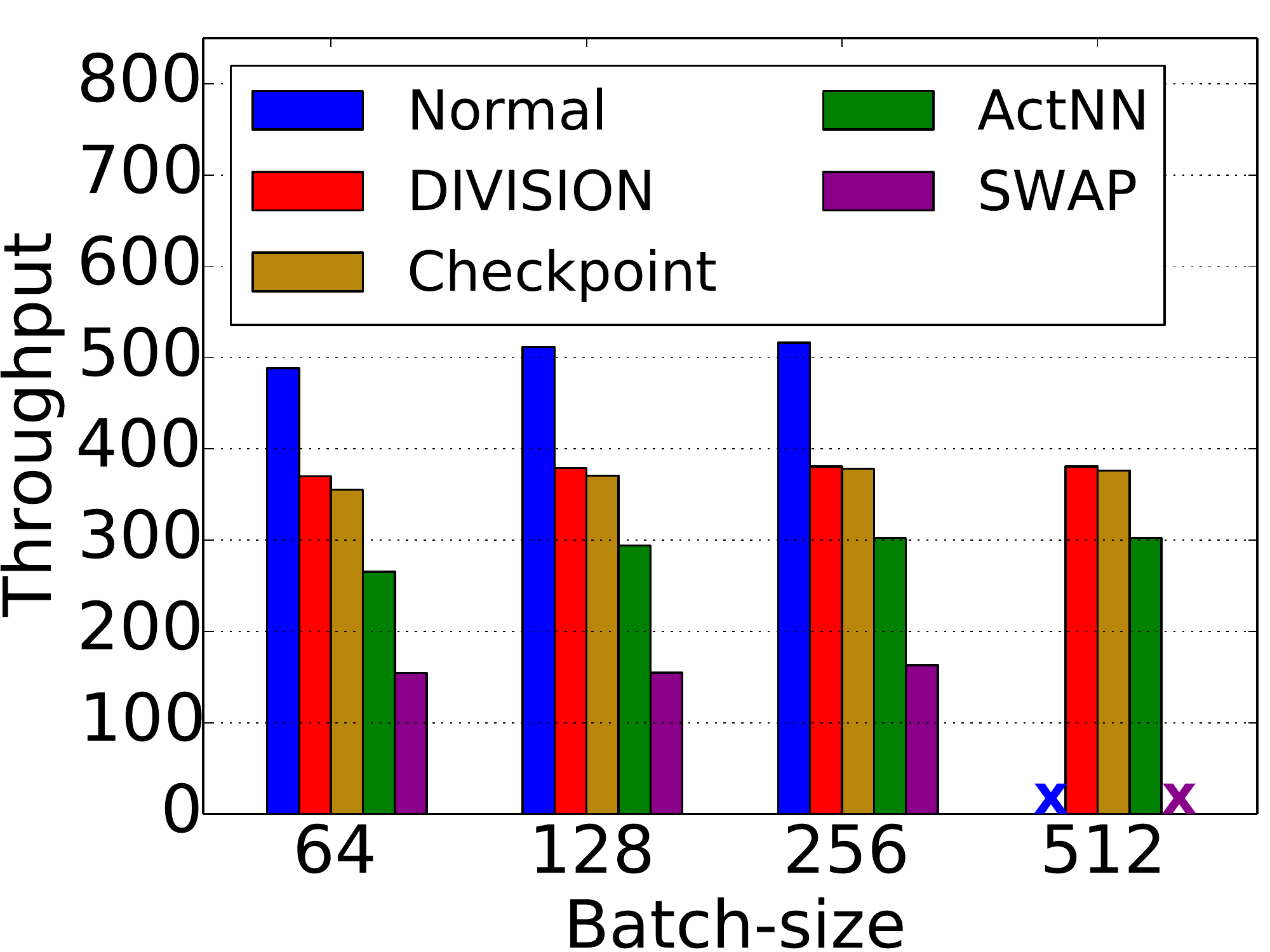}
    \end{minipage}%
    }
    \subfigure[WRN-50-2.]{
    \centering
    \begin{minipage}[t]{0.48\linewidth}
    	\includegraphics[width=1.0\linewidth]{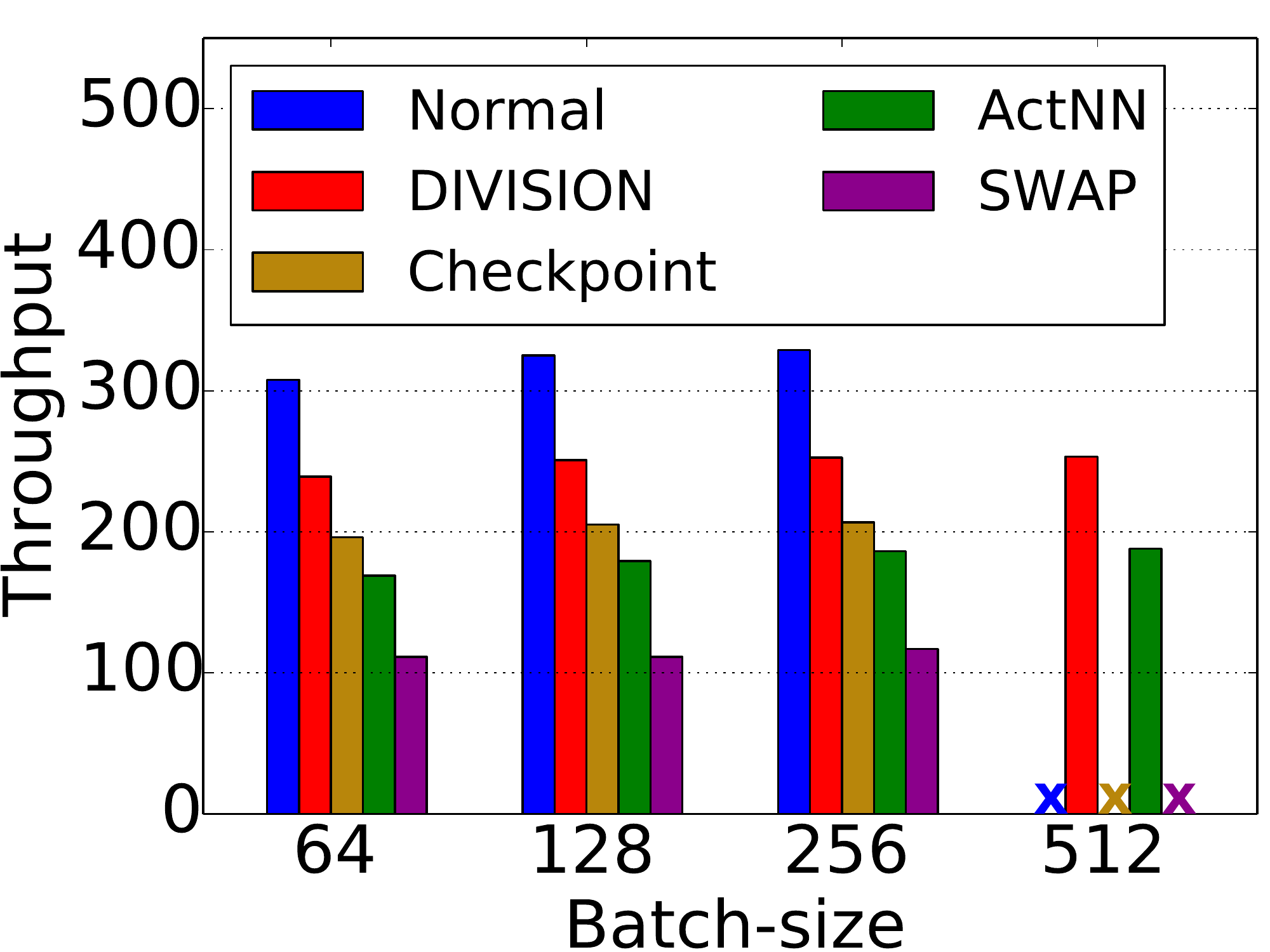}
    \end{minipage}%
    }
\caption{\label{fig:throughput} \small Training throughput $\uparrow$ of (a) Resnet-50 and (b) WRN-50-2 on the ImageNet dataset, where {\tiny \XSolidBold} \ {\small indicates out of memory. }
}
\end{figure}
\subsection{Evaluation by Training Throughput~(RQ1)}
\label{sec:exp_throughput}

We now evaluate the training methods in terms of the training throughput on the ImageNet dataset.
Generally, the throughput indicates the speed of a method via counting the number of data samples processed per second.
Formally, it is given by $\frac{\text{Mini-batch Size}}{T_\text{batch}}$, where $T_\text{batch}$ denotes the time of a single mini-batch updating.
According to the training throughput of \Algnameabbr{} and baseline methods in Figures~\ref{fig:throughput}~(a) and~(b), we have the following observations:


\begin{itemize}[leftmargin=10pt, topsep=-2pt]
\setlength{\parskip}{1mm}
\setlength{\parsep}{0mm}
\setlength{\itemsep}{0mm}


\item \textbf{Reasons for Time overhead:} Compared with normal training, the time overhead of \Algnameabbr{} comes from the estimation of LFC and quantization of HFC.
In ActNN, it mainly comes from the the dynamic bit-width allocation and activation map quantization. In Checkpoint, it comes from replaying the forward process of inter-media layers. In SWAP, the overhead mainly derives from the communication cost between the CPUs and GPUs.


\item \textbf{\Algnameabbr{} vs ActNN:} \Algnameabbr{} shows 1.3$\times$ acceleration compared to ActNN, as a result of simplified data compression.
To be concrete, \Algnameabbr{} adopts a simple average-pooling to extract the LFC, and a fixed bit-width per-channel quantization to compress the HFC.
In contrast, ActNN relies on a searching of optimal bit-width to match different samples, and adopts the searched bit-width for per-group quantization.
ActNN has a more complex processing during the training, which leads to its lower throughput than \Algnameabbr{}.

\end{itemize}

To summarize, according to Figures~\ref{fig:acc_div},~\ref{fig:throughput} and Table~\ref{tab:exp_mem}, state-of-the-art methods AC-GC, Checkpoint and ActNN shows a performance degradation in the aspect of model accuracy, compression rate, and training throughput, respectively.
According to a comprehensive comparison in terms of the three evaluation metrics in Figure~\ref{fig:radar}, \textbf{\Algnameabbr{} shows better comprehensive performance than these methods}.

\subsection{Performance of \Algnameabbr{} on Vision Transformer}

To further evaluate \Algnameabbr{} on vision transformers, we conduct experiments of Swin Transformer~\cite{liu2021swin} on the ImageNet dataset in comparison with Mesa~\cite{pan2021mesa} and Checkpoint~\cite{shoeybi2019megatron}. 
The model accuracy, memory cost~(with batch-size 128) and training throughput are given in Table~\ref{tab:exp_mem}~(b). 
It is observed that \Algnameabbr{} can effectively compress the training of the transformer, with almost the same model accuracy with normal training. 
Although \Algnameabbr{} shows slightly lower accuracy and throughput than Mesa and Checkpoint, respectively, it can significantly save more memory~(nearly 1.5GB than Mesa, and 1.7GB than Checkpoint). 
Moreover, Mesa is explicitly designed for vision transformers, and Checkpoint relies on manually selection of checkpointed layers.
\Algnameabbr{} is flexible for general vision models, including MLPs, CNNs, and vision transformers.

\subsection{Performance of \Algnameabbr{} on Depthwise and Pointwise Convolutional Layers}
\label{sec:division_depthwise}

To evaluate \Algnameabbr{} on the depthwise convlution and pointwise convulution layers, we conduct experiments of MobileNet-V2 on the CIFAR-10 and CIFAR-100 datasets.
The model accuracy of normal training and \Algnameabbr{} are given Table~\ref{tab:mobilenet_exp}~(a). 
It is observed that \Algnameabbr{} achieve nearly the same accuracy compared with normal training. 
This indicates the effectiveness of \Algnameabbr{} deployed to the depthwise convlution and pointwise convulution layers. 

\subsection{Effect of Dual Precision Strategy~(RQ2)}
\label{sec:ablation_exp}

To study the effect of our proposed dual precision strategy, \Algnameabbr{} is compared with three training methods:
\textbf{\Algnameabbr{} w/o HFC}: Merely providing the LFC for backward propagation.
\textbf{\Algnameabbr{} w/o LFC}: Merely providing the low-precision HFC for backward propagation.
\textbf{Fixed Quant}: Compressing the activation maps using a fixed bit-width quantization. 
The experiments are conducted on the CIFAR-100 and ImageNet dataset using the hyper-parameters given in Appendix~\ref{appendix:detail_ablation}.
The model accuracy are given in Table~\ref{tab:exp_mem}~(c).
Overall, we have the following insights:
\begin{itemize}[leftmargin=10pt, topsep=0pt]
\setlength{\parskip}{1mm}
\setlength{\parsep}{0mm}
\setlength{\itemsep}{0mm}

    \item \textbf{LFC \& Low Precision HFC:} Removing either HFC or LFC from \Algnameabbr{}, the training converges to far lower levels of accuracy. 
    This indicates both the LFC and low precision HFC of activation maps are necessary for leading the training to converge to an optimal solution. 

    \item \textbf{Benifits of Dual Precision:} 
    For the training with a fixed bit-width quantization, the bit-width should be at least 4.
    The noise caused by a fixed 2-bit quantization can terribly disturb the backward propagation, leading to a failure of convergence, as shown in Table~\ref{tab:exp_mem}~(c).
    \Algnameabbr{} solves this problem by combining a high-precision LFC and a fixed 2-bit quantization for the compression, and achieves nearly loss-less model accuracy.
    
    %


\end{itemize}





\begin{table}
\setlength{\abovecaptionskip}{-0mm}
\setlength{\belowcaptionskip}{-0mm}
\scriptsize
\caption{\label{tab:mobilenet_exp} \small (a) Accuracy of MoblieNet-V2 on the CIFAR-10 and CIFAR-100 datasets.
(b) Performance of \Algnameabbr{} under different hyperparameter settings. $n \times$ refers to the compression rate.
}
\begin{minipage}[h]{0.42\linewidth}
\flushleft
\subtable[]{
\begin{tabular}{l|c|c|}
\toprule[1pt]
\!\!\!\!\!\! Method \!\!\!\!\!\! & \makecell[c]{\!\!\!\!\!\! MN-V2 \!\!\!\!\!\! \\ \!\!\!\!\!\! CIFAR-10 \!\!\!\!\!\!} & \makecell[c]{\!\!\!\!\!\! MN-V2 \!\!\!\!\!\! \\ \!\!\!\!\!\! CIFAR-100 \!\!\!\!\!\!} \\ 
\hline
\!\!\!\!\!\! Normal \!\!\!\!\!\! & 91.9 & 71.0 \\
\!\!\!\!\!\! \Algnameabbr{} \!\!\!\!\!\! & 91.8 & 70.6 \\
\bottomrule[1pt]
\end{tabular}}
\end{minipage}
\begin{minipage}[h]{0.2\linewidth}
\flushright
\subtable[]{
\begin{tabular}{|l|c|c|c|c|c}
\toprule[1pt]
 & \makecell[c]{\!\!\!\!\!\! $B\!=\!18$ \!\!\!\!\!\! \\ \!\!\!\!\!\! $Q\!=\!2$ \!\!\!\!\!\!} & \makecell[c]{\!\!\!\!\!\! $B\!=\!12$ \!\!\!\!\!\! \\ \!\!\!\!\!\! $Q\!=\!2$ \!\!\!\!\!\!} & \makecell[c]{\!\!\!\!\!\! $B\!=\!8$ \!\!\!\!\!\! \\ \!\!\!\!\!\! $Q\!=\!2$ \!\!\!\!\!\!} & \makecell[c]{\!\!\!\!\!\! $B\!=\!8$ \!\!\!\!\!\! \\ \!\!\!\!\!\! $Q\!=\!4$ \!\!\!\!\!\!} & \makecell[c]{\!\!\!\!\!\! $B\!=\!8$ \!\!\!\!\!\! \\ \!\!\!\!\!\! $Q\!=\!8$ \!\!\!\!\!\!} \\
\hline
\!\!\!\!\!\! Acc(\%) \!\!\!\!\!\! & \!\!\!\!\!\!78.70\!\!\!\!\!\! & \!\!\!\!\!\!92.78\!\!\!\!\!\! & \!\!\!\!\!\!94.59\!\!\!\!\!\! & \!\!\!\!\!\!94.84\!\!\!\!\!\! & \!\!\!\!\!\!94.91\!\!\!\!\!\! \\
\!\!\!\!\!\! $n \times$ \!\!\!\!\!\! & \!\!\!\!\!\!10.18\!\!\!\!\!\! & \!\!\!\!\!\!10.18\!\!\!\!\!\! & \!\!\!\!\!\!9.74\!\!\!\!\!\! & \!\!\!\!\!\!5.74\!\!\!\!\!\! & \!\!\!\!\!\!3.25\!\!\!\!\!\! \\
\bottomrule[1pt]
\end{tabular}}
\end{minipage}
\vspace{-8mm}
\end{table}

\subsection{Hyper-parameter Tuning for \Algnameabbr{}~(RQ3)}
\label{sec:hyper_tuning_exp}

We study the effect of hyper-parameters $B$~(block-size) and $Q$~(bit-width) on the accuracy and compression rate.
Specifically, we adopt \Algnameabbr{} to train ResNet-18 on the CIFAR-10 dataset with $B \!\in\! \{ 8, 12, 18 \}$ and $Q \!\in\! \{ 2, 4, 8 \}$.
The model accuracy and compression rate are given in Table~\ref{tab:mobilenet_exp}~(b).
We have the following insights:
\begin{itemize}[leftmargin=10pt, topsep=-1mm]
\setlength{\parskip}{1mm}
\setlength{\parsep}{0mm}
\setlength{\itemsep}{0mm}

    \item \textbf{Effect of $Q$:} \Algnameabbr{} shows a stable accuracy~(nearly $94.8\%$) as the precision of HFC reduces~($Q$ reduces from $8$ to $2$).
    This indicates it only requires approximate values of HFC during backward propagation.
    
    \item \textbf{Effect of $B$:} As $B$ grows from $8$ to $18$, caching lower precision LFC during the forward pass leads to significant degradation of accuracy.
    This is because \Algnameabbr{} relies on a high-precision LFC to reconstruct the activation maps during the backward propagation. 

    \item \textbf{Optimal Setting:} \Algnameabbr{} shows optimal accuracy-compression trade-off taking $B \!=\! 8$ and $Q \!=\! 2$, where the degradation of accuracy is less than $0.35\%$. 
    



\end{itemize}

\subsection{Re-utilization of Hyper-parameter Settings Across Different Model Architectures and Datasests~(RQ3)}
\label{appendix:hyper_reuse}

We conduct follow-up experiments to study whether the hyper-parameter setting of \Algnameabbr{} has a consistent effect on different model architectures and datasets.
Specifically, the performance of \Algnameabbr{} deployed to ResNet-18 and MobileNet-V2 on the CIFAR-10 and CIFAR-100 datasets is shown in Table~\ref{tab:hp_exp}, where the hyper-parameters are selected from $B \!\in\! \{8, 18\}$ and $Q \!\in\! \{2, 8\}$. 
It is observed $B$ and $Q$ have a consistent impact on different model architectures and datasets: the accuracy slightly grows with $Q$ and considerably reduces with $B$.
This indicates we can reuse the hyper-parameters of \Algnameabbr{} on CIFAR-10 to CIFAR-100 with the same model architecture, or reuse the setting across ResNet-18 and MobileNet-V2 on the same dataset.
Moreover, $B\!=\!8$ and $Q\!=\!2$ can be a general and effective setting for most model architectures and datasets.

\begin{table}
\vspace{0mm}
\setlength{\abovecaptionskip}{-0mm}
\setlength{\belowcaptionskip}{0mm}
\caption{\label{tab:hp_exp} \small Performance of \Algnameabbr{} deployed to ResNet-18 and MoblieNet-V2 with different hyper-parameter settings.}
\vspace{2mm}
\scriptsize
\centering
\begin{tabular}{l|c|p{0.5cm}|c|p{0.5cm}|c|p{0.5cm}}
\toprule[1pt]
\makecell[c]{Hyperparameter setting \\ of \Algnameabbr{}} & \multicolumn{2}{c|}{\makecell[c]{$B=18$ \\ $Q=2$}} & \multicolumn{2}{c|}{\makecell[c]{$B=8$ \\ $Q=2$}} & \multicolumn{2}{c}{\makecell[c]{$B=8$ \\ $Q=8$}} \\
\hline
Evaluation metric & \!\!\!\!Acc(\%)\!\!\!\! & $n\times$ & \!\!\!\!Acc(\%)\!\!\!\! & $n\times$ & \!\!\!\!Acc(\%)\!\!\!\! & $n\times$ \\
\hline
RN-18 CIFAR-10 & $78.7$ & \!\!10.18$\times$ & $94.6$ & 9.74$\times$ & $94.9$ & 3.25$\times$ \\
RN-18 CIFAR-100 & $73.2$ & \!\!10.18$\times$ & $76.9$ & 9.33$\times$ & $77.0$ & 3.25$\times$ \\
MN-V2 CIFAR-10 & $10.0$ & 3.26$\times$ & $91.8$ & 3.20$\times$ & $91.0$ & 2.00$\times$ \\
MN-V2 CIFAR-100 & $62.4$ & 3.26$\times$ & $70.6$ & 3.20$\times$ & $71.6$ & 2.00$\times$ \\
\bottomrule[1pt]
\end{tabular}
\vspace{-8mm}
\end{table}


\section{Conclusion}

In this work, we propose a simple framework of activation compressed training.
Our framework is motivated by an instructive observation: \emph{DNN backward propagation mainly utilizes the LFC of the activation maps, while the majority of memory is for the storage of HFC during the training.}
This indicates the HFC of the activation maps is highly redundant and compressible during the training.
Following this direction, our proposed \Algnameabbr{} compresses the activation maps into dual precision representations: high-precision LFC and low-precision HFC, corresponding to their contributions to the backward propagation.
This dual precision compression can significantly reduce the memory cost of DNN training without loss of model accuracy.

Different from the existing work of ACT, \Algnameabbr{} is a simple and transparent framework, where the simplicity enables efficient compression and decompression; and transparency allows us to understand the compressible~(HFC) and non-compressible factors~(LFC) during DNN training. 
To this end, we hope our work could provide some inspiration for future work of DNN training.



\bibliography{citation}
\bibliographystyle{icml2023} 


\clearpage
\appendix
\setcounter{theorem}{0}

\section*{Appendix}

\section{1D/3D-DCT for 1D/3D Activation maps}
\label{appendix:1D/3D-DCT}

For 1D activation maps $\mathbf{H} \!\in\! \mathbb{R}^N$, the frequency-domain feature $\widetilde{\mathbf{H}} \!=\! \mathrm{DCT}(\mathbf{H})$ has the same shape of $N \!\times\! 1$; and each of the element $\tilde{h}_{i}$ is given by
\begin{align}
\setlength{\abovedisplayskip}{0mm}
\setlength{\belowdisplayskip}{0mm}
\small
\!\!\!\!\tilde{h}_{i} \!=\! \sum^{N-1}_{m=0} \cos \! \bigg[ \frac{\pi}{N} \bigg( m \!+\! \frac{1}{2}  \bigg) i \bigg] \!,
\nonumber
\end{align}
where $h_{i}$, $0\!\leq\! i \!\leq\! N\!-\!1$, are elements in the original matrix $\mathbf{H}$.
For 3D activation maps $\mathbf{H} \!\in\! \mathbb{R}^{N \!\times\! N \!\times\! N}$, the frequency-domain feature $\widetilde{\mathbf{H}} \!=\! \mathrm{DCT}(\mathbf{H})$ has a shape of $N \!\times\! N \!\times\! N$; and each of the element $\tilde{h}_{i,j,k}$ is given by
\begin{align}
\setlength{\abovedisplayskip}{0mm}
\setlength{\belowdisplayskip}{0mm}
\small
&\!\!\!\!\tilde{h}_{i,j,k} \!=\! \sum^{N-1}_{m=0} \sum^{N-1}_{n=0} \sum^{N-1}_{t=0} h_{m,n,t} \cos \! \bigg[ \frac{\pi}{N} \bigg( m \!+\! \frac{1}{2}  \bigg) i \bigg] 
\nonumber
\\
&\quad\quad\quad\quad \cos \! \bigg[ \frac{\pi}{N} \bigg( n \!+\! \frac{1}{2}  \bigg) j \bigg] \! \cos \! \bigg[ \frac{\pi}{N} \bigg( t \!+\! \frac{1}{2}  \bigg) k \bigg] \!,
\nonumber
\end{align}
where $h_{m,n,t}$, $0\!\leq\! m,n,t \!\leq\! N\!-\!1$, are elements in the original matrix $\mathbf{H}$.
During the training of DNNs, the frequency-domain feature is estimated via operating 1D/3D-DCT for the vector/tensor in each channel according to the shape of the activation map.
For 1D/3D activation maps, the LFC and HFC can be extracted given by  
{\setlength\abovedisplayskip{1mm}
\setlength\belowdisplayskip{1mm}
\begin{align}
    \mathbf{H}^{\mathsf{L}} &= \mathrm{iDCT}(\widetilde{\mathbf{H}} \odot \mathbf{M})
    \\[-4bp]
    \mathbf{H}^{\mathsf{H}} &= \mathrm{iDCT}(\widetilde{\mathbf{H}} \odot (\mathbf{1} - \mathbf{M})),
\end{align}}
\!\!\!\! where $\mathrm{iDCT}(\cdot)$ denotes the inverse DCT.
For 1D activation maps, $\mathbf{1}$ is $N$-dimensional vector; $\mathbf{M} = [m_{i} | 1 \leq i \leq N]$ denotes an $N$-dimensional low-pass mask satisfying $m_{i} \!=\! 1$ for $1 \leq i \leq W$ and $m_{i} \!=\! 0$ for other elements.
For 3D activation maps, $\mathbf{1}$ is $N \!\times\! N \!\times\! N$ tensor; $\mathbf{M} = [m_{i,j,k} | 1 \leq i,j,k \leq N]$ denotes an $N \!\times\! N \!\times\! N$ low-pass mask satisfying $m_{i,j,k} \!=\! 1$ for $1 \leq i,j,k \leq W$ and $m_{i,j,k} \!=\! 0$ for other elements.
$\mathbf{1} \!-\! \mathbf{M}$ indicates the high-pass mask.

\section{Implementation Details of Section~\ref{sec:pre_exp}}
\label{appendix:hyper_pre_exp}

We give the details of the experiment in Section~\ref{sec:pre_exp}.
Without loss of generality, the experiment is conducted on the CIFAR-10 dataset using ResNet-18, DenseNet-121 and ShuffleNet-V2.
During the backward propagation of normal training, the gradient of each layer $l$ is estimated by \begin{equation}
\setlength{\abovedisplayskip}{1mm}
\setlength{\belowdisplayskip}{1mm}
    [ \hat{\nabla}_{\mathbf{H}_{l-1}}, \hat{\nabla}_{\mathbf{W}_{l}} ]  = \mathrm{backward}( \hat{\nabla}_{\mathbf{H}_{l}}, \mathbf{H}_{l}, \mathbf{W}_{l} )
\end{equation}
For $\mathrm{LFC}$-$\mathrm{ACT}$, the gradient is estimated by 
\begin{equation}
\label{eq:grad_est_lfc}
\setlength{\abovedisplayskip}{1mm}
\setlength{\belowdisplayskip}{1mm}
    [ \hat{\nabla}_{\mathbf{H}_{l-1}}, \hat{\nabla}_{\mathbf{W}_{l}} ]  = \mathrm{backward}( \hat{\nabla}_{\mathbf{H}_{l}}, \mathbf{H}_{l}^{\mathsf{L}}, \mathbf{W}_{l} ),
\end{equation}
where $\mathrm{HFC}$-$\mathrm{ACT}$ denotes the HFC of $\mathbf{H}_{l}$; for $\mathrm{HFC}$-$\mathrm{ACT}$, the gradient is estimated by 
\begin{equation}
\label{eq:grad_est_hfc}
\setlength{\abovedisplayskip}{1mm}
\setlength{\belowdisplayskip}{1mm}
    [ \hat{\nabla}_{\mathbf{H}_{l-1}}, \hat{\nabla}_{\mathbf{W}_{l}} ]  = \mathrm{backward}( \hat{\nabla}_{\mathbf{H}_{l}}, \mathbf{H}_{l}^{\mathsf{H}}, \mathbf{W}_{l} ),
\end{equation}
where $\mathbf{H}_{l}^{\mathsf{H}}$ denotes the HFC of $\mathbf{H}_{l}$.
Note that Equations~(\ref{eq:grad_est_lfc}) and~(\ref{eq:grad_est_hfc}) causes the distortion of backward propagation in $\mathrm{LFC}$-$\mathrm{ACT}$ and $\mathrm{HFC}$-$\mathrm{ACT}$, respectively.
The objective of this experiment is to investigate whether this distortion of backward propagation may be powerful enough to lead training to a non-optimal solution. The hyper-parameter setting of the training is given in Table~\ref{tab:hyper_pre_exp}.


\begin{table}
\scriptsize
    \centering
    \caption{Hyper-parameters of the experiments in Section~\ref{sec:pre_exp}.}
    \label{tab:hyper_pre_exp}
    \begin{tabular}{l|c|c|c}
\toprule
         Architecture & ResNet-18 & DenseNet-121 & ShuffleNet-V2 \\
    \hline
        Epoch & 100 & 100 & 100 \\ 
        Batch-size & 256 & 256 & 256 \\
        Initial LR & 0.1 & 0.1 & 0.1 \\
        LR scheduler & Step LR & Step LR & Step LR \\
        Weight-decay & 0.0005 & 0.0005 & 0.0005 \\
        Optimizer & SGD & SGD & SGD \\
        SGD Momentum & 0.9 & 0.9 & 0.9 \\ 
        Ratio of LFC~($W/N$)& 0.3 & 0.3 & 0.5 \\
\bottomrule
    \end{tabular}
\end{table}

\section{$\lambda^L_l$ versus $\lambda^H_l$ in DenseNet-121}
\label{appendix:densenet121_lambda}

The training of DNNs follows Table~\ref{tab:hyper_pre_exp}.
For the estimation of LFC and HFC of activation maps, we take $W/N=0.5$ for the low-pass $\mathbf{M}$ mask and achieves the results in Figure~\ref{fig:lfc_vs_hfc}.
To further study whether Theorem~\ref{theorem:grad_error} holds for less $W$, we take $W/N=0.1$ and $W/N=0.2$, and achieves the values of $\lambda_l^L$ and $\lambda_l^H$ in the training (epoches 20, 40, and 60) of DenseNet-121 in Table~\ref{tab:lambda_exp_wN}, where $\lambda_l^L$ and $\lambda_l^H$ are estimated based on the input activation maps of the four DenseBlocks. It is consistently observed that $\lambda_l^L > \lambda_l^H$ for $W/N=0.1$ and $W/N=0.2$ in different training epochs. This indicates our proposed Theorem~\ref{theorem:grad_error} holds without loss of generality.

\begin{table*}
\scriptsize
\caption{\label{tab:lambda_exp_wN} Ratio of $\lambda_l^L$ to $\lambda_l^H$ on DensNet-121. (a) $W/N=0.1$ and (b) $W/N=0.2$.}
\begin{minipage}[h]{0.5\linewidth}
\centering
\subtable[$W/N=0.1$]{
\begin{tabular}{l|c|c|c|c|}
\toprule
    Epoch & 20 & 40 & 60 & \makecell[c]{Average \\ $\lambda_l^L/\lambda_l^H$} \\
\hline
    DenseBlock-1 & \makecell[c]{$\!\!\!\!\lambda_l^L\!=\!298.281$ \\ $\!\!\!\!\lambda_l^H\!=\!218.605$} & \makecell[c]{$\!\!\!\!\lambda_l^L\!=\!184.913$ \\ $\!\!\!\!\lambda_l^H\!=\!138.069$} & \makecell[c]{$\!\!\!\!\lambda_l^L\!=\!142.668$ \\ $\!\!\!\!\lambda_l^H\!=\!104.755$} & $1.36$ \\
    \hline
    DenseBlock-2 & \makecell[c]{$\!\!\!\!\lambda_l^L\!=\!3.245$ \\ $ \!\!\!\!\lambda_l^H\!=\!1.713$}       & \makecell[c]{$\!\!\!\!\lambda_l^L\!=\!1.284$ \\ $ \!\!\!\!\lambda_l^H\!=\!0.689$}     &  \makecell[c]{$\!\!\!\!\lambda_l^L\!=\!0.687$ \\ $ \!\!\!\!\lambda_l^H\!=\!0.372$} & $1.87$ \\
    \hline
    DenseBlock-3 & \makecell[c]{$\!\!\!\!\lambda_l^L\!=\!0.387$ \\ $ \!\!\!\!\lambda_l^H\!=\!0.260$}       & \makecell[c]{$\!\!\!\!\lambda_l^L\!=\!0.160$ \\ $ \!\!\!\!\lambda_l^H\!=\!0.086$}     &  \makecell[c]{$\!\!\!\!\lambda_l^L\!=\!0.084$ \\ $ \!\!\!\!\lambda_l^H\!=\!0.048$} & $1.70$ \\
    \hline
    DenseBlock-4 & \makecell[c]{$\!\!\!\!\lambda_l^L\!=\!0.062$ \\ $ \!\!\!\!\lambda_l^H\!=\!0.009$}       & \makecell[c]{$\!\!\!\!\lambda_l^L\!=\!0.011$ \\ $ \!\!\!\!\lambda_l^H\!=\!0.001$}     &  \makecell[c]{$\!\!\!\!\lambda_l^L\!=\!0.006$ \\ $ \!\!\!\!\lambda_l^H\!=\!0.001$} & $7.56$ \\
    \hline
    \makecell[c]{Average \\ $\lambda_l^L/\lambda_l^H$} & $2.95$ & $3.12$ & $3.30$ & $3.12$ \\
\bottomrule
\end{tabular}}
\end{minipage}
\begin{minipage}[h]{0.5\linewidth}
\centering
\subtable[$W/N=0.2$]{
\begin{tabular}{|l|c|c|c|c}
\toprule
Epoch & $20$ & $40$ & $60$ & \makecell[c]{Average \\ $\lambda_l^L/\lambda_l^H$} \\
\hline
DenseBlock-1 & \makecell[c]{$\!\!\!\!\lambda_l^L\!=\!362.672$ \\ $\!\!\!\!\lambda_l^H\!=\!154.214$}   & \makecell[c]{$\!\!\!\!\lambda_l^L\!=\!225.543$ \\ $\!\!\!\!\lambda_l^H\!=\!97.439$}  & \makecell[c]{$\!\!\!\!\lambda_l^L\!=\!173.595$ \\ $\!\!\!\!\lambda_l^H\!=\!73.828$}  & $2.34$ \\
\hline
DenseBlock-2 & \makecell[c]{$\!\!\!\!\lambda_l^L\!=\!3.632$ \\ $\!\!\!\!\lambda_l^H\!=\!1.326$}       & \makecell[c]{$\!\!\!\!\lambda_l^L\!=\!1.440$ \\ $\!\!\!\!\lambda_l^H\!=\!0.533$}     & \makecell[c]{$\!\!\!\!\lambda_l^L\!=\!0.774$ \\ $\!\!\!\!\lambda_l^H\!=\!0.285$}  & $2.72$ \\
\hline
DenseBlock-3 & \makecell[c]{$\!\!\!\!\lambda_l^L\!=\!0.445$ \\ $\!\!\!\!\lambda_l^H\!=\!0.202$}       & \makecell[c]{$\!\!\!\!\lambda_l^L\!=\!0.179$ \\ $\!\!\!\!\lambda_l^H\!=\!0.067$}     & \makecell[c]{$\!\!\!\!\lambda_l^L\!=\!0.095$ \\ $\!\!\!\!\lambda_l^H\!=\!0.037$}  & $2.49$ \\
\hline
DenseBlock-4 & \makecell[c]{$\!\!\!\!\lambda_l^L\!=\!0.062$ \\ $\!\!\!\!\lambda_l^H\!=\!0.009$}       & \makecell[c]{$\!\!\!\!\lambda_l^L\!=\!0.011$ \\ $\!\!\!\!\lambda_l^H\!=\!0.001$}     & \makecell[c]{$\!\!\!\!\lambda_l^L\!=\!0.006$ \\ $\!\!\!\!\lambda_l^H\!=\!0.001$}  & $7.56$ \\ 
\hline
\makecell[c]{Average \\ $\lambda_l^L/\lambda_l^H$} & $3.58$ & $3.78$ & $3.97$ & $3.78$ \\
\bottomrule
\end{tabular}}
\end{minipage}
\end{table*}

\section{Compression of 1D and 3D Activation Maps by \Algnameabbr{}}
\label{appendix:1D2D3D}

We give more details about \Algnameabbr{} considering 1D, 2D and 3D activation maps in this section.

\subsection{Activation Map Compression}  
\label{appendix:1D2D3D_compression}

\Algnameabbr{} adopts average pooling to estimate the LFC by $\mathbf{H}_{l}^{\mathsf{L}} \!=\! \mathrm{Average Pooling}(\mathbf{H}_{l})$. The value of block-size and moving stride is a unified hyper-parameter $B$, which controls the memory of $\mathbf{H}_{l}^{\mathsf{L}}$.
The average pooling of 1D, 2D and 3D activation maps are considered as follows,
\begin{align}
\small
\label{eq:average_poolnd}
\mathrm{Minibatch} \!\times\! \mathrm{Channel} \!\times\! N  &\stackrel{\mathrm{Average Pooling}1D}{\longrightarrow} 
\nonumber
\\
&\!\!\!\!\!\!\!\!\!\!\!\!\!\!\!\!\!\!\!\!\!\!\!\!\!\!\!\!\!\!\!\! \mathrm{Minibatch} \!\times\! \mathrm{Channel} \!\times\! \big\lfloor N/B \big\rfloor
\nonumber
\\
\mathrm{Minibatch} \!\times\! \mathrm{Channel} \!\times\! N \!\times\! N &\stackrel{\mathrm{Average Pooling}2D}{\longrightarrow} 
\\
&\!\!\!\!\!\!\!\!\!\!\!\!\!\!\!\!\!\!\!\!\!\!\!\!\!\!\!\!\!\!\!\!\!\!\!\!\!\!\!\!\!\!\!\!\!\!\!\!\!\!\!\! \mathrm{Minibatch} \!\times\! \mathrm{Channel} \!\times\! \big\lfloor N/B \big\rfloor \!\times\! \big\lfloor N/B \big\rfloor  
\nonumber 
\\
\mathrm{Minibatch} \!\!\times\!\! \mathrm{Channel} \!\!\times\!\! N \!\!\times\!\! N \!\!\times\!\! N  &\stackrel{\mathrm{Average Pooling}3D}{\longrightarrow}
\nonumber
\\
&\!\!\!\!\!\!\!\!\!\!\!\!\!\!\!\!\!\!\!\!\!\!\!\!\!\!\!\!\!\!\!\!\!\!\!\!\!\!\!\!\!\!\!\!\!\!\!\!\!\!\!\!\!\!\!\!\!\!\!\!\!\!\!\! \mathrm{Minibatch} \!\!\times\!\! \mathrm{Channel} \!\!\times\!\! \big\lfloor N/B \big\rfloor \!\!\times\!\! \big\lfloor N/B \big\rfloor \!\!\times\!\! \big\lfloor N/B \big\rfloor 
\nonumber
\end{align}

To estimate the HFC, \Algnameabbr{} calculates the residual value $\mathbf{H}_{l}^{\mathsf{H}} = \mathbf{H}_{l} - \mathrm{UpSampling}(\mathbf{H}_{l}^{\mathsf{L}})$, where the $\mathrm{UpSampling}(\cdot)$ enlarges $\mathbf{H}_{l}^{\mathsf{L}}$ to the shape of $\mathbf{H}_{l}$ via nearest interpolation. 
The up sampling of 1D, 2D and 3D activation maps are considered as follows,
\begin{align}
\small
\label{eq:up_samplend}
\mathrm{Minibatch} \!\times\! &\mathrm{Channel} \!\times\! \big\lfloor N/B \big\rfloor
\nonumber
\\
&\stackrel{\mathrm{UpSampling}1D}{\longrightarrow} \mathrm{Minibatch} \!\times\! \mathrm{Channel} \!\times\! N
\nonumber
\\
\mathrm{Minibatch} \!\times\! &\mathrm{Channel} \!\times\! \big\lfloor N/B \big\rfloor \!\times\! \big\lfloor N/B \big\rfloor  
\\
&\stackrel{\mathrm{UpSampling}2D}{\longrightarrow} \mathrm{Minibatch} \!\times\! \mathrm{Channel} \!\times\! N \!\times\! N
\nonumber 
\\
\mathrm{Minibatch} \!\times\! &\mathrm{Channel} \!\!\times\!\! \big\lfloor N/B \big\rfloor \!\!\times\!\! \big\lfloor N/B \big\rfloor \!\!\times\!\! \big\lfloor N/B \big\rfloor  
\nonumber
\\
&\!\!\!\! \stackrel{\mathrm{UpSampling}3D}{\longrightarrow} \mathrm{Minibatch} \!\times\! \mathrm{Channel} \!\!\times\!\! N \!\!\times\!\! N \!\!\times\!\! N 
\nonumber
\end{align}

Then, \Algnameabbr{} adopts $Q$-bit per-channel quantization for the compression, where the bit-width $Q$ controls the precision and memory cost of HFC after the compression.
Let $\mathbf{V}_{l}^{\mathsf{H}}$ denote a $Q$-bit integer matrix, as the low-precision representation of $\mathbf{H}_{l}^{\mathsf{H}}$.
The detailed procedure of compressing $\mathbf{H}_{l}^{\mathsf{H}}$ into $\mathbf{V}_{l}^{\mathsf{H}}$ is given by
\begin{equation}
\label{eq:appendix_quantization}
\mathbf{V}_{l}^{\mathsf{H}} = \mathrm{Quant} (\mathbf{H}_{l}^{\mathsf{H}}) = \big\lfloor \Delta_{l}^{-1} (\mathbf{H}_{l}^{\mathsf{H}} - \delta_{l}) \big\rceil,
\end{equation}
where $\delta_{l}$ denotes the minimum element in $\mathbf{H}^{\mathsf{H}}_{l}$; $\Delta_{l} = (h_{\max}-\delta_{l})/(2^Q-1)$ denotes the quantization step; $h_{\max}$ denotes the maximum element in $\mathbf{H}^{\mathsf{H}}_{l}$; $\lfloor \bullet \rceil$ denotes the stochastic rounding.

After the compression, as the representation of $\mathbf{H}_l$, the tuple of $( \mathbf{H}_{l}^{\mathsf{L}}, \mathbf{V}_{l}^{\mathsf{H}}, \Delta_{l}, \delta_{l} )$ is cached to the memory for reconstructing the activation maps during the backward pass.

\subsection{Activation Map Decompression}  
\label{appendix:1D2D3D_decompression}

During the backward pass, \Algnameabbr{} adopts the cached tuples of $\{( \mathbf{H}_{l}^{\mathsf{L}}, \mathbf{V}_{l}^{\mathsf{H}}, \Delta_{l}, \delta_{l} ) \!\mid\! 0 \leq l \leq L\!-\!1 \}$ to reconstruct the activation map layer-by-layer.
Specifically, for each layer $l$, \Algnameabbr{} dequantizes the HFC via $\hat{\mathbf{H}}_{l}^{\mathsf{H}} \!=\! \Delta_{l} \mathbf{V}_{l}^{\mathsf{H}} \!+\! \delta_{l}$, which is the inverse process of Equation~(\ref{eq:appendix_quantization}). 
Then, the activation map is reconstructed via $\hat{\mathbf{H}}_{l} = \mathrm{UpSampling}(\mathbf{H}_{l}^{\mathsf{L}}) + \hat{\mathbf{H}}_{l}^{\mathsf{H}}$, where $\mathrm{UpSampling}(\cdot)$ enlarges $\mathbf{H}_{l}^{\mathsf{L}}$ to the shape of $\mathbf{H}_{l}$ via nearest interpolation. 
The cases of 1D, 2D and 3D activation maps are considered in Equation~(\ref{eq:up_samplend}).

After the decompression, \Algnameabbr{} frees the caching of $( \mathbf{H}_{l}^{\mathsf{L}}, \mathbf{V}_{l}^{\mathsf{H}}, \Delta_{l}, \delta_{l} )$, and takes $\hat{\mathbf{H}}_{l}$ into $[\hat{\nabla}_{\mathbf{H}_{l-1}}, \hat{\nabla}_{\mathbf{W}_{l}}] = \mathrm{backward}( \hat{\nabla}_{\mathbf{H}_{l}}, \hat{\mathbf{H}}_{l-1}, \mathbf{W}_{l} )$ to estimate the gradient for backward propagation.

\section{Experiment Setting}
\label{appendix:exp_set}

We give the experiment setting including the datasets, baseline methods and model architectures in this section.

\noindent
\textbf{Datasets.} We consider CIFAR-10, CIFAR-100~\cite{krizhevsky2009learning} and ImageNet~\cite{deng2009imagenet} datasets in our experiments.
\textbf{CIFAR-10:} An image dataset with 60,000 color images in 10 different classes, where each image has $32 \!\times\! 32$ pixels.
\textbf{CIFAR-100:} An image dataset with 60,000 color images in 100 different classes, where each image has $32 \!\times\! 32$ pixels.
\textbf{ImageNet:} A large scale image dataset which has over one million color images covering 1000 categories, where each image has $224 \!\times\! 224$ pixels.

\noindent
\textbf{Baseline Methods.}
\textbf{Normal:} Caching the exact activation map for backward propagation.
\textbf{BLPA:} A systemic implementation of ACT by~\cite{chakrabarti2019backprop}, which only supports ResNet-related architectures. 
\textbf{AC-GC:} A framework of ACT with automatic searched bit-width for the quantization of activation maps~\cite{evans2021ac}.
\textbf{ActNN:} Activation compression training with dynamic bit-width quantization, where the bit-allocation minimizes the variance of activation maps via dynamic processing~\cite{chen2021actnn}.
\textbf{Mesa:} Mesa is an ACT-based memory-efficient training method explicitly designed for vision transformers~\cite{pan2021mesa}.
\textbf{Checkpoint:} Caching some key activation maps to reconstruct other activation maps via replaying parts of the forward pass during the backward pass~\cite{chen2016training}.
\textbf{SWAP:} Swapping the activation maps to the CPU during the forward pass the memory consumption of GPU, and reload the activation maps to GPU during the backward pass~\cite{huang2020swapadvisor}.

\noindent
\textbf{DNN Architectures.}
For benchmarking the model accuracy, we consider ResNet-18~(top-1 accuracy 94.89\%), ResNet-164~(top-1 accuracy 94.9\%), and MobileNet-V2~(top-1 accuracy 91.9\%) on the CIFAR-10 dataset; ResNet-34~(top-1 accuracy 77.1\%), DenseNet-121~(top-1 accuracy 79.75\%), ResNet-164~(top-1 accuracy 77.3\%), and MobileNet-V2~(top-1 accuracy 71.0\%) on the CIFAR-100 dataset; and ResNet-50~(top-1 accuracy 76.15\%), DenseNet-161~(top-1 accuracy 77.65\%) and Swin Transformer-T~(top-1 accuracy 81.2\%) on the ImageNet dataset.
Our reproduced validating accuracy on the ImageNet dataset is consistent with the official results of torchvision\footnote{\scriptsize \url{https://paperswithcode.com/lib/torchvision}}.
Moreover, for benchmarking the memory cost and training throughput, we consider the large models ResNet-50 and WRN-50-2 on the ImageNet dataset.
To study the performance of DIVISION on depthwise and pointwise convolutional Layers, we consider MobileNet-V2.
The comprehensive comparison in Figure~\ref{fig:radar} considers ResNet-50 and ResNet-34 on the ImageNet and CIFAR-100 datasets, respectively.

\section{Implementation details about \Algnameabbr{} and Baseline Methods}
\label{appendix:config_division}

\textbf{\Algnameabbr{}:}
\Algnameabbr{} adopts block-size 8 ($B=8$) and 2-bit quantization~($Q=2$) to compress the activation maps of linear, convolutional and BatchNorm layers, where the theoretical compression rate is not less than $10.35\times$.
For the operators without quantization error during backward propagation such as pooling layers, ReLu activation, and Dropout, \Algnameabbr{} follows the algorithms in Appendix~\ref{appendix:relu_compression} to compress the activation maps. 
Other hyper-parameter settings are given in Table~\ref{tab:Hyper-parameter setting}.
\textbf{BLPA:}
Existing work~\cite{chakrabarti2019backprop} has shown that BLPA requires at least 4-bit ACT for loss-less DNN training.
We follow this setting for BLPA, where the compression rate of activation maps is not more than $8\times$.
\textbf{AC-GC:}
AC-GC follows existing work~\cite{evans2021ac} to take the multiplicative error $(1+e^2_{\text{AC-GC}})=1.5$, where the searched bit-width enables AC-GC to satisfy this loss bound~(training loss not more than $150\%$ of normal training).
In this setting, AC-GC finalizes the bit-with as 7.01 after the searching, which has a nearly $3.5\times$ compression rate of activation maps.
\textbf{ActNN:}
ActNN adopts 2-bit ACT and dynamic programming for searching the optimal bit-width specific for each layer, and uses per-group quantization for compressing the activation map, which has approximately $10.5\times$ compression rate of activation maps.
Such experimentally setting is denoted as L3 strategy in the original work~\cite{chen2021actnn}, and we follow this setting in this section.
\textbf{Mesa:} We follow the default setting in the original work of Mesa~\cite{pan2021mesa}.
\textbf{Checkpoint:}
Checkpoint relies on a manually design of the checkpointing layers.
We follow the checkpoint strategy of Megatron-LM in our experiment~\cite{shoeybi2019megatron}. 
Specifically, Megatron-LM checkpoints the activation map after each transformer block. 
We follow this strategy to checkpoint the activation map after each transformer block in the Swin Transformer, and after each Bottleneck block in the ResNet-50. 
For all methods, the memory cost is measured in a single mini-batch updating; and the throughput is estimated by averaging that of 20 mini-batch updating to achieve a stable result.


\section{Compression of Pooling layers, Relu activations, and Dropout}
\label{appendix:relu_compression}

\Algnameabbr{} follows Algorithms~\ref{alg:max-pool},~\ref{alg:avg-pool},~\ref{alg:relu} and~\ref{alg:dropout} to compresse the activation map of a Max-Pooling layer, Average-Pooling layer, Relu activation and Dropout operator, respectively. 
For the pooling layers, we consider a simple case $\mathrm{kernel size} \!=\! \mathrm{moving stride} \!=\! k$.
General cases with different $\mathrm{kernel size}$ and $\mathrm{moving stride}$ can be designed in analogous ways.

\begin{algorithm}
\small
\caption{\small Max-Pooling layer.}
\label{alg:max-pool}
\small
\begin{algorithmic}[1]

\STATE $\textbf{Function}$ Forward ($\mathbf{H}_{l-1}$, $k$, **kwargs)
\STATE \quad $\mathbf{H}_{l}, \mathbf{V}_{l-1}\footnotemark[11] = $Max-Pooling($\mathbf{H}_{l-1}$, $k$, kwargs)
\STATE \quad Pack \& Cache $\mathbf{V}_{l-1}$ using Int8.
\STATE \quad \textbf{return} $\mathbf{H}_{l}$
\STATE
\STATE $\textbf{Function}$ Backward($\nabla_{\mathbf{H}_l}$)
\STATE \quad Load $\mathbf{V}_{l-1}$ and $k$.
\STATE \quad $\nabla_{\mathbf{H}_l}' = \mathbf{1}_{k \times k} \otimes \nabla_{\mathbf{H}_l}$
\STATE \quad $\nabla_{\mathbf{H}_{l-1}} = \mathbf{V}_{l-1} \odot \nabla_{\mathbf{H}_l}'$
\STATE \quad \textbf{return} $\nabla_{\mathbf{H}_{l-1}}$

\end{algorithmic}
\end{algorithm}

\begin{algorithm}
\small
\caption{\small Average-Pooling layer.}
\label{alg:avg-pool}
\small
\begin{algorithmic}[1]

\STATE $\textbf{Function}$ Forward ($\mathbf{H}_{l-1}$, $k$, **kwargs)
\STATE \quad $\mathbf{H}_{l} = $Avg-Pooling($\mathbf{H}_{l-1}$, $k$, kwargs)
\STATE \quad \textbf{return} $\mathbf{H}_{l}$
\STATE
\STATE $\textbf{Function}$ Backward($\nabla_{\mathbf{H}_l}$)
\STATE \quad $\nabla_{\mathbf{H}_l}' = \mathbf{1}_{k \times k} \otimes \nabla_{\mathbf{H}_l}$
\STATE \quad $\nabla_{\mathbf{H}_{l-1}} = k^{-2} \nabla_{\mathbf{H}_l}$
\STATE \quad \textbf{return} $\nabla_{\mathbf{H}_{l-1}}$

\end{algorithmic}
\end{algorithm}

\section{Evaluation of \Algnameabbr{} on Multi-layer Perceptrons (MLPs)}
\label{appendix:exp_mlp}

We conduct experiments on the GAS dataset~\cite{Dua:2019} (128-dimensional features, 13910 instances, 6 classification task).
The classification model is a 4-layer MLP (128 neuros in the input layer, 6 neuros in the output layer, and 64 neuros in the hidden layer); The setting of
\Algnameabbr{} is $B=16$ and $Q=2$. The model accuracy and memory cost of activation map are given in Table~\ref{tab:mlp_exp}. It is observed that \Algnameabbr{} has $7.3\times$ compression rate with only $0.07\%$ degradation of model accuracy. This indicates the effectiveness of \Algnameabbr{} on the MLP models.

\begin{table}
\scriptsize
\centering
\caption{\label{tab:mlp_exp} Model Accuracy on the GAS dataset.}
\begin{tabular}{lccc}
\toprule
Training & Testing Accuracy (\%) & Memory (KB) & Compression \\
\midrule
Normal Training & $98.92$  &  $250$ & N/A \\
\Algnameabbr{} & $98.85$ &  $34.2$ & $7.3\times$ \\
\bottomrule
\end{tabular}
\end{table}

\footnotetext{\scriptsize $\mathbf{V}_{l-1}$ reserves the locations of each kernel-wise max-values in $\mathbf{H}_{l-1}$.}

\section{Evaluation of \Algnameabbr{} on NLP tasks}
\label{appendix:exp_nlp}

To evaluate \Algnameabbr{} on the tasks of natural language processing, we conducted experiments of deploying DIVISION (with $B=8$ and $Q=4$) to the T5-Base~\cite{raffel2020exploring} language model on the CoLA, SST2, MRPC, and STS-B datasets. 
The input text is padded to a maximum length of 128 during the training. The results of evaluation metrics on different datasets are shown in the Table~\ref{tab:division_glue}. 
It is observed that DIVISION achieves nearly loss-less model accuracy compared with normal training. 
The memory cost considers that of model weight and activation maps, and \Algnameabbr{} achieves over 4.2$\times$ compression rate on all of the datasets, further emphasizing its effectiveness on language models. 

\begin{table}[]
    \centering
    \caption{\small Evaluation results on GLUE tasks.}
\resizebox{0.5\textwidth}{!}{
    \begin{tabular}{lccc}
    \toprule
    Dataset	& Standard Evaluation Metric &	Normal &	\Algnameabbr{}	  \\
    \midrule
    CoLA &	Matthew’s Correlation &	60.3 &	60.6	\\
    SST2 &	Accuracy &	93.9 &	94.6	\\
    MRPC &	F1 score &	91.4 &	91.9	\\
    STS-B &	Pearson-Spearman correlation &	90.7 &	90.4 \\
    \bottomrule
    \end{tabular}
}
    \label{tab:division_glue}
    \vspace{-5mm}
\end{table}

\begin{table}[h]
    \vspace{-5mm}
    \centering
    \caption{\small Comparison of \Algnameabbr{} with TinyScript.}
\resizebox{0.5\textwidth}{!}{
    \begin{tabular}{lccc}
    \toprule
    Top-1 & Error &	Top-5 Error &	Compression Rate \\
    \midrule
    TinyScript n=8	& N/A &	7.74 &	9.8$\times$ \\
    \Algnameabbr{} & 24.1 &	7.33 &	10.4$\times$ \\
    \bottomrule
    \end{tabular}
}
    \label{tab:division_tinyscript}
\end{table}

\section{Comparison with TinyScript}
\label{appendix:exp_nlp}

In this section, we conducted a comparison between \Algnameabbr{} and TinyScript~\cite{fu2020don} for training ResNet-50 on the ImageNet dataset. 
The model accuracy and compression rate are given in Table~\ref{tab:division_tinyscript}. 
It is observed that DIVISION outperforms TinyScript, achieving a higher compression rate with lower Top-5 error. 
These results demonstrate the superiority of DIVISION over TinyScript.

\begin{algorithm}
\caption{\small Relu operator.}
\label{alg:relu}
\small
\begin{algorithmic}[1]

\STATE $\textbf{Function}$ Forward ($\mathbf{H}_{l-1}$)
\STATE \quad $\mathbf{V}_{l-1} = \text{sgn}(\mathbf{H}_{l-1})$
\STATE \quad $\mathbf{H}_l = \mathbf{V}_{l-1} \odot \mathbf{H}_{l-1}$
\STATE \quad Pack \& Cache $\mathbf{V}_{l-1}$ using Int8
\STATE \quad \textbf{return} $\mathbf{H}_{l}$
\STATE
\STATE
\STATE $\textbf{Function}$ Backward($\nabla_{\mathbf{H}_l}$)
\STATE \quad Load $\mathbf{V}_{l-1}$.
\STATE \quad $\nabla_{\mathbf{H}_{l-1}} = \mathbf{V}_{l-1} \odot \nabla_{\mathbf{H}_l}$
\STATE \quad \textbf{return} $\nabla_{\mathbf{H}_{l-1}}$

\end{algorithmic}
\end{algorithm}

\begin{algorithm}[h]
\caption{Dropout operator.}
\label{alg:dropout}
\small
\begin{algorithmic}[1]

\STATE $\textbf{Function}$ Forward ($\mathbf{H}_{l-1}$)
\STATE \quad Generate a $\mathrm{Minibatch} \!\times\! \mathrm{Channel} \!\times\! N \!\times\! N$ binary matrix  
\STATE \quad $\!\mathbf{V}_{l\!-\!1}\!$ following the Bernoulli distribution with dropout 
\STATE \quad probability $p$.
\STATE \quad $\mathbf{H}_l = \mathbf{V}_{l-1} \odot \mathbf{H}_{l-1}$
\STATE \quad Pack \& Cache $\mathbf{V}_{l-1}$ using Int8.
\STATE \quad \textbf{return} $\mathbf{H}_{l}$
\STATE
\STATE $\textbf{Function}$ Backward($\nabla_{\mathbf{H}_l}$)
\STATE \quad Load $\mathbf{V}_{l-1}$.
\STATE \quad $\nabla_{\mathbf{H}_{l-1}} = \mathbf{V}_{l-1} \odot \nabla_{\mathbf{H}_l}$
\STATE \quad \textbf{return} $\nabla_{\mathbf{H}_{l-1}}$

\end{algorithmic}
\end{algorithm}

\section{Computation Infrastructure}
\label{appendix:infrastructure}

The details about our physical computing infrastructure for testing the training memory cost and throughput are given in Table~\ref{tab:computing_infrastructure}.

\begin{table}[]
\vspace{-5mm}
\centering
\caption{Computing infrastructure for the experiments.}
\begin{tabular}{l|c}
\toprule
Device Attribute & Value \\
\hline
Computing infrastructure & GPU \\
GPU model & Nvidia-RTX3090 \\
GPU number & 1 \\
CUDA Version & 12.0 \\
\bottomrule
\end{tabular}
\label{tab:computing_infrastructure}
\vspace{-5mm}
\end{table}

\section{Implementation Details of the Experiment in Section~\ref{sec:ablation_exp}}
\label{appendix:detail_ablation}

We give the implementation details of the experiment in Section~\ref{sec:ablation_exp}.
The model of image classification on the CIFAR-100 and ImageNet datasets are ResNet-34 and ResNet-50, respectively.
Regarding to the training methods, \emph{\Algnameabbr{} w/o HFC} takes block-size $B \!=\! 4$ for estimating LFC; \emph{\Algnameabbr{} w/o LFC} takes the bit-width $Q \!=\! 2$ for the quantization of HFC;
\Algnameabbr{} combines these settings for the training;
and \emph{Fixed Quant} adopts 4-bit and 2-bit per-group quantization to compress the activation maps during the training, where the group size of quantization follows existing work~\cite{chen2021actnn} to be 256.
Other training hyper-parameters are given in Table~\ref{tab:Hyper-parameter setting}.

\newpage
\onecolumn

\begin{table}[]
    \centering
    \scriptsize
    \caption{Hyper-parameter setting.}
    \label{tab:Hyper-parameter setting}
    \begin{tabular}{l|c|c|c|c|c|c|c|c|c|c}
\toprule
         Dataset & \multicolumn{2}{c|}{CIFAR-10} & \multicolumn{3}{c|}{CIFAR-100} & \multicolumn{5}{c}{ImageNet} \\
    \hline
         Architecture & \!\!ResNet-18\!\! & \!\!ResNet-164\!\! & \!\!MobileNet-V2\!\! & \!\!ResNet-34\!\!  & \!\!ResNet-164\!\! & \!\!DenseNet-121\!\! & \!\!MobileNet-V2\!\! & \!\!ResNet-50\!\! & \!\!DenseNet-161\!\! & Swin-T \\
        Epoch & 100 & 100 & 100 & 100 & 200 & 100 & 100 & 120 & 120 & 300 \\ 
        Batch-size & 256 & 256 & 256 & 256 & 256 & 256 & 256 & 256 & 256 & 128 \\
        Initial LR & 0.1 & 0.1 & 0.1 & 0.1 & 0.15 & 0.1 & 0.1 & 0.1 & 0.1 & 5e-4 \\
        LR scheduler & Cos LR & Cos LR & Cos LR & Cos LR & Cos LR & Cos LR & Cos LR & Cos LR & Cos LR & Cos LR \\
        Weight-decay & 0.0005 & 0.0005 & 0.00004 & 0.0005 & 0.0005 & 0.0005 & 0.00004 & 0.0001 & 0.0001 & 0.05 \\
        Optimizer & SGD & SGD & SGD & SGD & SGD & SGD & SGD & SGD & SGD & SGD \\
        SGD Momentum & 0.9 & 0.9 & 0.9 & 0.9 & 0.9 & 0.9 & 0.9 & 0.9 & 0.9 & 0.9 \\ 
        Block-size $B$ & 8 & 8 & 8 & 8 & 8 & 8 & 8 & 8 & 8 & 64 \\
        Bit-width $Q$ & 2 & 2 & 2 & 2 & 2 & 2 & 2 & 2 & 2 & 2\\
\bottomrule
    \end{tabular}
\end{table}

\section{Proof of Theorem~\ref{theorem:grad_error}}
\label{appendix:proof_grad_error}

We prove Theorem~\ref{theorem:grad_error} in this section.

\textbf{Theorem 1}
\textit{
During the backward pass of a convolutional layer~$l$, $\mathrm{GEB}^{\mathsf{L}}_l$ and $\mathrm{GEB}^{\mathsf{H}}_l$ satisfy
\begin{equation}
\mathrm{GEB}^{\mathsf{L}}_l \!-\! \mathrm{GEB}^{\mathsf{H}}_l 
\!=\! \Big( \alpha_{l,l} || \mathbf{H}_{l-1}^{\mathsf{T}} ||_F \!+\! \beta_l \Big) ( \lambda^{\mathsf{H}}_l \!-\! \lambda^{\mathsf{L}}_l ) \!+ || \mathbf{H}_{l-1}^{\mathsf{T}} ||_F \!\!\!\! \sum_{i=l+1}^L \!\! \alpha_{l,i} ( \lambda^{\mathsf{H}}_i \!-\! \lambda^{\mathsf{L}}_i ) \prod_{j=l}^{i-1} \gamma_j,
\end{equation}
where $\alpha_{l,i}, \beta_l, \gamma_l > 0$ for $1 \!\leq\! l,i \!\leq\! L$ are given by Equation~(\ref{eq:beta}); $\lambda^{\mathsf{L}}_l \!=\! || \widetilde{\mathbf{H}}_l \!\odot\! \mathbf{M} ||_F$; $\lambda^{\mathsf{H}}_l \!=\! || \widetilde{\mathbf{H}}_l \!\odot\! ( \mathbf{1} \!-\! \mathbf{M}) ||_F$; 
$\widetilde{\mathbf{H}}_l = \mathrm{DCT}({\mathbf{H}}_l)$; and $\mathbf{M}$ denotes the loss-pass mask given by Equation~(\ref{eq:lfc_estimate}).}


\begin{proof}

For simplicity of derivation, we study the case with a single input channel and output channel number.
In this case, $\mathbf{H}_l$ and $\mathbf{W}_l$ are 2-D matrix for each layer $l$, where $1 \leq l \leq L$.
The backward propagation of a convolutional layer is given by
\begin{equation}
\begin{aligned}
\label{eq:CNN_grad}
&\hat{\nabla}_{\mathbf{Z}_{l}} \, \ = \hat{\nabla}_{\mathbf{Z}_{l+1}} \ast  \mathbf{W}_{l+1}^{\mathsf{rot}} \odot \sigma'(\hat{\mathbf{Z}}_{l}), 
\\
&\hat{\nabla}_{\mathbf{W}_{l}} = \hat{\nabla}_{\mathbf{Z}_{l}} \ast  \hat{\mathbf{H}}_{l-1}^{\mathsf{T}},
\end{aligned}
\end{equation}
where $\ast$ denotes a convolutional operation; $\hat{\mathbf{Z}}_{l} = \mathbf{W}_{l} \ast \hat{\mathbf{H}}_{l-1} + b_{l}$; $b_{l}$ denotes the bias of layer $l$; and $\mathbf{W}_{l}^{\mathsf{rot}}$ denotes to rotate $\mathbf{W}_{l}$ by $180^\circ$.
The case of multiple input and output channels can be proved in an analogous way, which is omitted in this work.

According to Equation~(\ref{eq:CNN_grad}), we have the gradient of $\mathbf{Z}_{l}$ given by
\begin{align}
&\hat{\nabla}_{\mathbf{Z}_{l}} - \nabla_{\mathbf{Z}_{l}}
\nonumber
\\
&= \hat{\nabla}_{\mathbf{Z}_{l+1}} \ast  \mathbf{W}_{l+1}^{\mathsf{rot}} \odot \sigma'(\hat{\mathbf{Z}}_{l}) - \nabla_{\mathbf{Z}_{l+1}} \ast  \mathbf{W}_{l+1}^{\mathsf{rot}} \odot \sigma'(\mathbf{Z}_{l}),
\nonumber
\\
&= \!\hat{\nabla}_{\mathbf{Z}_{l+1}} \!\!\!\!\ast\!  \mathbf{W}_{l+1}^{\mathsf{rot}} \!\odot\! \sigma'(\hat{\mathbf{Z}}_{l}) \!-\! \hat{\nabla}_{\mathbf{Z}_{l+1}} \!\!\!\!\ast \! \mathbf{W}_{l+1}^{\mathsf{rot}} \!\odot\! \sigma'(\mathbf{Z}_{l}) \!+\! \hat{\nabla}_{\mathbf{Z}_{l+1}} \!\!\!\!\ast\!  \mathbf{W}_{l+1}^{\mathsf{rot}} \!\odot\! \sigma'(\mathbf{Z}_{l}) \!-\! \nabla_{\mathbf{Z}_{l+1}} \!\ast\!  \mathbf{W}_{l+1}^{\mathsf{rot}} \!\odot\! \sigma'(\mathbf{Z}_{l}),
\nonumber
\\
\label{eq:delta_grad_z}
&= \hat{\nabla}_{\mathbf{Z}_{l+1}} \ast \mathbf{W}_{l+1}^{\mathsf{rot}} \odot [ \sigma'(\hat{\mathbf{Z}}_{l}) - \sigma'(\mathbf{Z}_{l}) ] + (\hat{\nabla}_{\mathbf{Z}_{l+1}} - \nabla_{\mathbf{Z}_{l+1}}) \ast  \mathbf{W}_{l+1}^{\mathsf{rot}} \odot \sigma'(\mathbf{Z}_{l}). 
\end{align}
For the activation functions $\mathrm{ReLu}(\cdot)$, $\mathrm{LeakyReLu}(\cdot)$, $\mathrm{Sigmoid}(\cdot)$, $\mathrm{Tanh}(\cdot)$ and $\mathrm{SoftPlus}(\cdot)$, the gradient $\sigma'(\cdot)$ satisfies $| \sigma''(\cdot) | \leq 1$ in the differentable domains.
Note that we have $|| \mathbf{W}_l \ast \mathbf{H}_{l-1} ||_F \leq (K_l + N_l - 1) || \mathbf{W}_l ||_F || \mathbf{H}_{l-1} ||_F$ according to Corollary~\ref{corol:convolution_norm_inequation}.
$|| \sigma'(\hat{\mathbf{Z}}_{l}) - \sigma'(\mathbf{Z}_{l}) ||_F$ satisfies
\begin{align}
\label{eq:delta_grad_sigma}
|| \sigma'(\hat{\mathbf{Z}}_{l}) - \sigma'(\mathbf{Z}_{l}) ||_F \leq || \hat{\mathbf{Z}}_{l} - \mathbf{Z}_{l} ||_F \leq (K_l + N_l - 1) || \hat{\mathbf{H}}_{l-1} - \mathbf{H}_{l-1} ||_F || \mathbf{W}'_{l} ||_F,
\end{align}
where $K_l$ and $N_l$ denote the size of convolutional kernel $\mathbf{W}_l$ and activation map $\mathbf{H}_l$ in layer $l$, respectively.
After taking Equation~(\ref{eq:delta_grad_sigma}) into Equation~(\ref{eq:delta_grad_z}), we have
\begin{align}
|| \hat{\nabla}_{\mathbf{Z}_{l}} - \nabla_{\mathbf{Z}_{l}} ||_F 
&\leq (K_l + N_l - 1) || \hat{\nabla}_{\mathbf{Z}_{l+1}} ||_F || \mathbf{W}_{l+1}^{\mathsf{rot}} ||_F || \sigma'(\hat{\mathbf{Z}}_{l}) - \sigma'(\mathbf{Z}_{l}) ||_F 
\nonumber
\\
&+ (K_l + N_l - 1) || \hat{\nabla}_{\mathbf{Z}_{l+1}} - \nabla_{\mathbf{Z}_{l+1}} ||_F || \mathbf{W}_{l+1}^{\mathsf{rot}} ||_F || \sigma'(\mathbf{Z}_{l}) ||_F,
\nonumber
\\
&= (K_l + N_l - 1)^2 || \hat{\nabla}_{\mathbf{Z}_{l+1}} ||_F || \mathbf{W}_{l+1}^{\mathsf{rot}} ||_F || \hat{\mathbf{H}}_{l-1} \!-\! \mathbf{H}_{l-1} ||_F || \mathbf{W}'_{l} ||_F 
\nonumber
\\
&+ (K_l + N_l - 1) || \hat{\nabla}_{\mathbf{Z}_{l+1}} \!-\! \nabla_{\mathbf{Z}_{l+1}} ||_F || \mathbf{W}_{l+1}^{\mathsf{rot}} ||_F || \sigma'(\mathbf{Z}_{l}) ||_F,
\nonumber
\\
\label{eq:grad_Z_upperbound}
&= \eta_l || \hat{\mathbf{H}}_{l-1} - \mathbf{H}_{l-1} ||_F + \gamma_l || \hat{\nabla}_{\mathbf{Z}_{l+1}} - \nabla_{\mathbf{Z}_{l+1}} ||_F,
\end{align}
where $\eta_l$ and $\gamma_l$ are given by
\begin{equation}
\begin{aligned}
\label{eq:alpha_gamma}
\eta_l &= (K_l + N_l - 1)^2 || \hat{\nabla}_{\mathbf{Z}_{l+1}} ||_F || \mathbf{W}_{l+1} ||_F || \mathbf{W}'_{l} ||_F;
\\
\gamma_l &= (K_l + N_l - 1) || \mathbf{W}_{l+1} ||_F || \sigma'(\mathbf{Z}_{l}) ||_F;
\end{aligned}
\end{equation}
the value $\eta_l$ and $\gamma_l$ depend on the model weight before backward propagation, which is constant with respect to the gradient. 
Iterate Equation~(\ref{eq:grad_Z_upperbound}) until $l\!=\!L$ where $|| \hat{\nabla}_{\mathbf{Z}_{L}} - \nabla_{\mathbf{Z}_{L}} ||_F \!\leq\! \eta_L || \hat{\mathbf{H}}_{L-1} - \mathbf{H}_{L-1} ||_F$. 
In this way, we have
\begin{align}
\label{eq:grad_Z_upperbound2}
\!\!\!\! || \hat{\nabla}_{\mathbf{Z}_{l}} \!-\! \nabla_{\mathbf{Z}_{l}} ||_F \!\leq\! \eta_l || \hat{\mathbf{H}}_{l-1} \!-\! \mathbf{H}_{l-1} ||_F + \!\!\sum_{i=l+1}^L \!\! \eta_i || \hat{\mathbf{H}}_{i-1} \!-\! \mathbf{H}_{i-1} ||_F \prod_{j=l}^{i-1} \gamma_j.
\end{align} 

According to Equation~(\ref{eq:CNN_grad}), we have the gradient of $\mathbf{W}_{l}$ given by
\begin{align}
\hat{\nabla}_{\mathbf{W}_{l}} - \nabla_{\mathbf{W}_{l}}
&= \hat{\nabla}_{\mathbf{Z}_{l}} \ast  \hat{\mathbf{H}}_{l-1}^{\mathsf{T}} - \nabla_{\mathbf{Z}_{l}} \ast  \mathbf{H}_{l-1}^{\mathsf{T}},
\nonumber
\\
&= \hat{\nabla}_{\mathbf{Z}_{l}} \ast  \hat{\mathbf{H}}_{l-1}^{\mathsf{T}} - \hat{\nabla}_{\mathbf{Z}_{l}} \ast \mathbf{H}_{l-1}^{\mathsf{T}} + \hat{\nabla}_{\mathbf{Z}_{l}} \ast \mathbf{H}_{l-1}^{\mathsf{T}} - \nabla_{\mathbf{Z}_{l}} \ast  \mathbf{H}_{l-1}^{\mathsf{T}},
\nonumber
\\
\label{eq:grad_W_difference}
&= \hat{\nabla}_{\mathbf{Z}_{l}} \ast ( \hat{\mathbf{H}}_{l-1}^{\mathsf{T}} - \mathbf{H}_{l-1}^{\mathsf{T}} ) + ( \hat{\nabla}_{\mathbf{Z}_{l}} - \nabla_{\mathbf{Z}_{l}} ) \ast  \mathbf{H}_{l-1}^{\mathsf{T}}.
\end{align}

Taking Equation~(\ref{eq:grad_Z_upperbound2}) into Equation~(\ref{eq:grad_W_difference}), we have
\begin{align}
&||\hat{\nabla}_{\mathbf{W}_{l}} - \nabla_{\mathbf{W}_{l}}||_F 
\nonumber
\\
&\leq (K_l + N_l - 1) || \hat{\nabla}_{\mathbf{Z}_{l}} ||_F || \hat{\mathbf{H}}_{l-1}^{\mathsf{T}} - \mathbf{H}_{l-1}^{\mathsf{T}} ||_F + (K_l + N_l - 1) || \hat{\nabla}_{\mathbf{Z}_{l}} - \nabla_{\mathbf{Z}_{l}} ||_F || \mathbf{H}_{l-1}^{\mathsf{T}} ||_F,
\nonumber
\\
&\leq \!\! (K_l \!\!+\!\! N_l \!\!-\!\! 1) \! \Bigg( \!\! || \hat{\nabla}_{\mathbf{Z}_{l}} ||_F \! || \hat{\mathbf{H}}_{l-1}^{\mathsf{T}} \!\!-\!\! \mathbf{H}_{l-1}^{\mathsf{T}} ||_F \!\!+\!\! || \mathbf{H}_{l-1}^{\mathsf{T}} ||_F \!\! \bigg[ \! \eta_l || \hat{\mathbf{H}}_{l-1} \!\!-\!\! \mathbf{H}_{l-1} ||_F \!\! + \!\!\!\!\! \sum_{i=l+1}^L \!\!\!\! \eta_i || \hat{\mathbf{H}}_{i-1} \!\!-\!\! \mathbf{H}_{i-1} ||_F \!\! \prod_{j=l}^{i-1} \!\! \gamma_j \bigg] \Bigg) \!\!
\nonumber
\\
&= \!\! (K_l \!+\! N_l \!-\! 1) \bigg[ \big( || \hat{\nabla}_{\mathbf{Z}_{l}} ||_F \!+\! \eta_l || \mathbf{H}_{l-1}^{\mathsf{T}} ||_F \big) || \hat{\mathbf{H}}_{l-1}^{\mathsf{T}} \!\!-\!\! \mathbf{H}_{l-1}^{\mathsf{T}} ||_F \!+\! || \mathbf{H}_{l-1}^{\mathsf{T}} ||_F \!\!\!\! \sum_{i=l+1}^L \!\!\!\! \eta_i || \hat{\mathbf{H}}_{i-1} \!\!-\!\! \mathbf{H}_{i-1} ||_F \! \prod_{j=l}^{i-1} \! \gamma_j \bigg],
\nonumber
\\
\label{eq:W_GER}
&= \Big( \beta_l \!+\! \alpha_{l,l} || \mathbf{H}_{l-1}^{\mathsf{T}} ||_F \Big) || \hat{\mathbf{H}}_{l-1}^{\mathsf{T}} \!-\! \mathbf{H}_{l-1}^{\mathsf{T}} ||_F \!+\! || \mathbf{H}_{l-1}^{\mathsf{T}} ||_F \!\!\!\! \sum_{i=l+1}^L \alpha_{l,i} || \hat{\mathbf{H}}_{i-1}^{\mathsf{T}} \!-\! \mathbf{H}_{i-1}^{\mathsf{T}} ||_F \prod_{j=l}^{i-1} \gamma_j,
\end{align}
where 
\begin{equation}    
\begin{aligned}
\label{eq:beta}
\alpha_{l,i} &= (K_l + N_l - 1) (K_i + N_i - 1)^2 || \hat{\nabla}_{\mathbf{Z}_{i+1}} ||_F || \mathbf{W}_{i+1} ||_F || \mathbf{W}'_{i} ||_F;
\\
\beta_l &= (K_l + N_l - 1) || \hat{\nabla}_{\mathbf{Z}_{l}} ||_F;
\\
\gamma_l &= (K_l + N_l - 1) || \mathbf{W}_{l+1} ||_F || \sigma'(\mathbf{Z}_{l}) ||_F;
\end{aligned}
\end{equation}
$K_l$ and $N_l$ denote the size of convolutional kernel $\mathbf{W}_l$ and activation map $\mathbf{H}_l$ in layer $l$, respectively.

During the $\mathrm{LFC}\text{-}\mathrm{ACT}$ and $\mathrm{HFC}\text{-}\mathrm{ACT}$ trainings, the activation map of a convolutional layer satisfies
\begin{align}
\label{eq:H_difference_lfc}
|| \mathbf{H}_l - \mathbf{H}^{\mathsf{L}}_l ||_F &= || \widetilde{\mathbf{H}}_l - \widetilde{\mathbf{H}}^{\mathsf{L}}_l ||_F = || \widetilde{\mathbf{H}}_l \odot( \mathbf{1} - \mathbf{M}) ||_F \triangleq \lambda^{\mathsf{H}}_l,
\\
\label{eq:H_difference_hfc}
|| \mathbf{H}_l - \mathbf{H}^{\mathsf{H}}_l ||_F &= || \widetilde{\mathbf{H}}_l - \widetilde{\mathbf{H}}^{\mathsf{H}}_l ||_F = || \widetilde{\mathbf{H}}_l \odot \mathbf{M} ||_F \triangleq \lambda^{\mathsf{L}}_l.
\end{align}
Taking Equations~(\ref{eq:H_difference_lfc}) and~(\ref{eq:H_difference_hfc}) into~(\ref{eq:W_GER}), we have $\mathrm{GEB}^{\mathsf{L}}_l$ and $\mathrm{GEB}^{\mathsf{H}}_l$ of a convolutional layer by
\begin{align}
\label{eq:GER_L_upperbound}
||\hat{\nabla}_{\mathbf{W}_{l}} \!-\! \nabla^{\mathsf{L}}_{\mathbf{W}_{l}}||_F 
\!&\leq\! \Big( \alpha_{l,l} || \mathbf{H}_{l-1}^{\mathsf{T}} ||_F \!+\! \beta_l \Big) \lambda^{\mathsf{H}}_l + || \mathbf{H}_{l-1}^{\mathsf{T}} ||_F \!\! \sum_{i=l+1}^L \alpha_{l,i} \lambda^{\mathsf{H}}_i \prod_{j=l}^{i-1} \gamma_j \triangleq \mathrm{GEB}^{\mathsf{L}}_l,
\\
\label{eq:GER_H_upperbound}
||\hat{\nabla}_{\mathbf{W}_{l}} \!-\! \nabla^{\mathsf{H}}_{\mathbf{W}_{l}}||_F 
\!&\leq\! \Big( \alpha_{l,l} || \mathbf{H}_{l-1}^{\mathsf{T}} ||_F \!+\! \beta_l \Big) \lambda^{\mathsf{L}}_l + || \mathbf{H}_{l-1}^{\mathsf{T}} ||_F \!\!\sum_{i=l+1}^L \alpha_{l,i} \lambda^{\mathsf{L}}_i \prod_{j=l}^{i-1} \gamma_j \triangleq \mathrm{GEB}^{\mathsf{H}}_l.
\end{align}

Given the expression of $\mathrm{GEB}^{\mathsf{L}}_l$ and $\mathrm{GEB}^{\mathsf{H}}_l$ by Equations~(\ref{eq:GER_L_upperbound}) and~(\ref{eq:GER_H_upperbound}), respectively, we have the $\mathrm{GEB}$ for a convolutional layer given by
\begin{align}
\mathrm{GEB}^{\mathsf{L}}_l \!-\! \mathrm{GEB}^{\mathsf{H}}_l 
\!=\! \Big( \alpha_{l,l} || \mathbf{H}_{l-1}^{\mathsf{T}} ||_F \!+\! \beta_l \Big) ( \lambda^{\mathsf{H}}_l \!-\! \lambda^{\mathsf{L}}_l ) \!+ || \mathbf{H}_{l-1}^{\mathsf{T}} ||_F \!\!\!\! \sum_{i=l+1}^L \!\! \alpha_{l,i} ( \lambda^{\mathsf{H}}_i \!-\! \lambda^{\mathsf{L}}_i ) \prod_{j=l}^{i-1} \gamma_j. 
\nonumber
\end{align}


\end{proof}

\begin{corollary}
\label{corol:convolution_norm_inequation}
For a $K \times K$ convolutional kernel and a $N \times N$ square matrix $\mathbf{H}$, we have the 
\begin{equation}
    || \mathbf{W} \ast \mathbf{H} ||_F \leq (K+N-1) || \mathbf{W} ||_F || \mathbf{H} ||_F
\end{equation}
\end{corollary}

\begin{proof}
According to the relations between convolutional operation and Discrete Fourier Transformation~\cite{sundararajan2001discrete}, $\mathbf{W} \ast \mathbf{H}$ satisfies
\begin{equation}
\label{eq:convolution_property}
    \mathrm{FFT}(\mathbf{W} \ast \mathbf{H}) = \mathrm{FFT}(\mathrm{ZP}(\mathbf{W})) \odot \mathrm{FFT}(\mathrm{ZP}(\mathbf{H})),
\end{equation}
where $\mathrm{FFT}(\cdot)$ denotes the discrete Fourier transformation; $\mathrm{ZP}(\mathbf{W})$ denotes zero-padding $\mathbf{W}$ into a $(K+N-1) \!\times\! (K+N-1)$ matrix. 
According to the Parseval's theorem~\cite{diniz2010digital}, $\mathrm{FFT}(\mathrm{ZP}(\mathbf{W}))$ and $\mathrm{FFT}(\mathrm{ZP}(\mathbf{H}))$ and $\mathrm{FFT}(\mathbf{W} \ast \mathbf{H})$ satisfy
\begin{equation}
\begin{aligned}
\label{eq:Parseval_theory}
||\mathrm{FFT}(\mathrm{ZP}(\mathbf{W}))||_F &= (K+N-1) ||\mathbf{W}||_F,
\\
||\mathrm{FFT}(\mathrm{ZP}(\mathbf{H}))||_F &= (K+N-1) ||\mathbf{H}||_F,
\\
||\mathrm{FFT}(\mathbf{W} \ast \mathbf{H})||_F &= (K+N-1) ||\mathbf{W} \ast \mathbf{H}||_F.
\end{aligned}
\end{equation}
Taking $||\mathbf{A}_1 \odot \mathbf{A}_2||_F \leq ||\mathbf{A}_1||_F ||\mathbf{A}_2||_F$ into Equation~(\ref{eq:Parseval_theory}), we have 
\begin{equation}
    \label{eq:matrix_inequality}
    \mathrm{FFT}(\mathrm{ZP}(\mathbf{W})) \odot \mathrm{FFT}(\mathrm{ZP}(\mathbf{H})) \leq || \mathrm{FFT}(\mathbf{W}) ||_F || \mathrm{FFT}(\mathbf{H}) ||_F
\end{equation}
Taking Equation~(\ref{eq:Parseval_theory}) into Equation~(\ref{eq:matrix_inequality}), we have
\begin{align}
(K+N-1) || \mathbf{W} \ast \mathbf{H} ||_F &= || \mathrm{FFT}(\mathbf{W}) \odot \mathrm{FFT}(\mathbf{H}) ||_F
\nonumber
\\
&\leq || \mathrm{FFT}(\mathbf{W}) ||_F || \mathrm{FFT}(\mathbf{H}) ||_F
\nonumber
\\
&= (K+N-1) || \mathbf{W}||_F (K+N-1) || \mathbf{H} ||_F
\nonumber
\end{align}


\end{proof}

\section{Gradient Error Bound ($\mathrm{GEB}$) of a Linear Layer}
\label{appendix:linear_grad_error}

We give the Gradient Error upper Bound ($\mathrm{GEB}$) of a linear layer and proof in this section.

\textbf{Theorem 1B.} 
\textit{During the backward pass of a linear layer~$l$, $\mathrm{GEB}^{\mathsf{L}}_l$ and $\mathrm{GEB}^{\mathsf{H}}_l$ satisfy
\begin{equation}
\label{eq:MLP_grad_error}
\mathrm{GEB}^{\mathsf{L}}_l \!-\! \mathrm{GEB}^{\mathsf{H}}_l  \!=\! \big( \alpha_l || \mathbf{H}_{l-1}^{\mathsf{T}} ||_F \!+\! \beta_l \big) (\lambda^{\mathsf{H}}_l \!-\! \lambda^{\mathsf{L}}_l) \!+\! || \mathbf{H}_{l-1}^{\mathsf{T}} ||_F \!\!\!\!\sum_{i=l+1}^L \!\!\! \alpha_i  (\lambda^{\mathsf{H}}_i \!-\! \lambda^{\mathsf{L}}_i) \!\prod_{j=l}^{i-1}\! \gamma_j,
\end{equation}
where $\alpha_l, \beta_l, \gamma_l > 0$ for $1 \!\leq\! l \!\leq\! L$ are given by Equation~(\ref{eq:beta_MLP}); $\lambda^{\mathsf{L}}_l \!=\! || \widetilde{\mathbf{H}}_l \!\odot\! \mathbf{M} ||_F$; $\lambda^{\mathsf{H}}_l \!=\! || \widetilde{\mathbf{H}}_l \!\odot\! ( \mathbf{1} \!-\! \mathbf{M}) ||_F$; 
$\widetilde{\mathbf{H}}_l = \mathrm{DCT}({\mathbf{H}}_l)$; and $\mathbf{M}$ denotes the 1-D loss-pass mask.}


\begin{proof}

For simplicity of derivation, we consider the case $\mathrm{MiniBatch} \!\!=\!\! 1$.
In this case, $\mathbf{H}_l$ is a vector; and $\mathbf{W}_l$ is a 2-D matrix, for $1 \!\leq\! l \!\leq\! L$.
The backward propagation of a linear layer is given by
\begin{equation}
\begin{aligned}
\label{eq:MLP_grad}
&\hat{\nabla}_{\mathbf{Z}_{l}} \, \ = (\mathbf{W}_{l+1} \hat{\nabla}_{\mathbf{Z}_{l+1}}) \odot \sigma'(\hat{\mathbf{Z}}_{l}), 
\\
&\hat{\nabla}_{\mathbf{W}_{l}} = \hat{\nabla}_{\mathbf{Z}_{l}}  \hat{\mathbf{H}}_{l-1}^{\mathsf{T}},
\end{aligned}
\end{equation}
where $\hat{\mathbf{Z}}_{l} = \mathbf{W}_{l}^{\mathsf{T}} \hat{\mathbf{H}}_{l-1} + b_{l}$; and $b_{l}$ denotes the bias of layer $l$.
The case of $\mathrm{MiniBatch}\!\geq\! 2$ can be proved in an analogous way, which is omitted in this work.

According to Equation~(\ref{eq:MLP_grad}), we have the gradient of $\mathbf{Z}_{l}$ given by
\begin{align}
&\hat{\nabla}_{\mathbf{Z}_{l}} - \nabla_{\mathbf{Z}_{l}}
\nonumber
\\
&= \mathbf{W}_{l+1} \hat{\nabla}_{\mathbf{Z}_{l+1}} \odot \sigma'(\hat{\mathbf{Z}}_{l}) - \mathbf{W}_{l+1} \nabla_{\mathbf{Z}_{l+1}} \odot \sigma'(\mathbf{Z}_{l}),
\nonumber
\\
&= \! \mathbf{W}_{l+1} \hat{\nabla}_{\mathbf{Z}_{l+1}} \!\odot\! \sigma'(\hat{\mathbf{Z}}_{l}) \!-\! \mathbf{W}_{l+1} \hat{\nabla}_{\mathbf{Z}_{l+1}}  \!\odot\! \sigma'(\mathbf{Z}_{l}) \!+\! \mathbf{W}_{l+1} \hat{\nabla}_{\mathbf{Z}_{l+1}} \!\odot\! \sigma'(\mathbf{Z}_{l}) \!-\! \mathbf{W}_{l+1} \nabla_{\mathbf{Z}_{l+1}} \!\odot\! \sigma'(\mathbf{Z}_{l}),
\nonumber
\\
&= \mathbf{W}_{l+1} \hat{\nabla}_{\mathbf{Z}_{l+1}} \odot [ \sigma'(\hat{\mathbf{Z}}_{l}) - \sigma'(\mathbf{Z}_{l}) ] + (\hat{\nabla}_{\mathbf{Z}_{l+1}} - \mathbf{W}_{l+1} \nabla_{\mathbf{Z}_{l+1}}) \odot \sigma'(\mathbf{Z}_{l}). 
\end{align}
For activation functions $\mathrm{ReLu}(\cdot)$, $\mathrm{LeakyReLu}(\cdot)$, $\mathrm{Sigmoid}(\cdot)$, $\mathrm{Tanh}(\cdot)$ and $\mathrm{SoftPlus}(\cdot)$, the gradient $\sigma'(\cdot)$ satisfies $| \sigma''(\cdot) | \!\leq\! 1$ in each differentiable domain.
Combined with Cauchy–Schwarz inequality $|| \mathbf{A}_1 \mathbf{A}_2 ||_F \leq || \mathbf{A}_1 ||_F || \mathbf{A}_2 ||_F$ ~\cite{horn2012matrix}, we have 
\begin{align}
|| \sigma'(\hat{\mathbf{Z}}_{l}) - \sigma'(\mathbf{Z}_{l}) ||_F \leq || \hat{\mathbf{Z}}_{l} - \mathbf{Z}_{l} ||_F \leq || \mathbf{W}_{l} ||_F || \hat{\mathbf{H}}_{l-1} - \mathbf{H}_{l-1} ||_F.
\end{align}
According to inequality $|| \mathbf{A}_1 \odot \mathbf{A}_2 ||_F \!\leq\! || \mathbf{A}_1 ||_F || \mathbf{A}_2 ||_F$~\cite{horn2012matrix},  we have the upper bound of $|| \hat{\nabla}_{\mathbf{Z}_{l}} - \nabla_{\mathbf{Z}_{l}} ||_F$ given by
\begin{align}
&|| \hat{\nabla}_{\mathbf{Z}_{l}} - \nabla_{\mathbf{Z}_{l}} ||_F 
\nonumber
\\
&\leq || \mathbf{W}_{l+1} ||_F || \hat{\nabla}_{\mathbf{Z}_{l+1}} ||_F  || \sigma'(\hat{\mathbf{Z}}_{l}) - \sigma'(\mathbf{Z}_{l}) ||_F + || \mathbf{W}_{l+1} ||_F || \hat{\nabla}_{\mathbf{Z}_{l+1}} - \nabla_{\mathbf{Z}_{l+1}} ||_F || \sigma'(\mathbf{Z}_{l}) ||_F,
\nonumber
\\
&= || \mathbf{W}_{l+1} ||_F || \hat{\nabla}_{\mathbf{Z}_{l+1}} ||_F || \hat{\mathbf{H}}_{l-1} \!-\! \mathbf{H}_{l-1} ||_F || \mathbf{W}'_{l} ||_F + || \mathbf{W}_{l+1} ||_F || \hat{\nabla}_{\mathbf{Z}_{l+1}} \!-\! \nabla_{\mathbf{Z}_{l+1}} ||_F || \sigma'(\mathbf{Z}_{l}) ||_F,
\nonumber
\\
\label{eq:grad_Z_upperbound_MLP}
&= \alpha_l || \hat{\mathbf{H}}_{l-1} - \mathbf{H}_{l-1} ||_F + \gamma_l || \hat{\nabla}_{\mathbf{Z}_{l+1}} - \nabla_{\mathbf{Z}_{l+1}} ||_F,
\end{align}
where $\alpha_l$ and $\gamma_l$ are given by
\begin{equation}
\begin{aligned}
\label{eq:alpha_gamma_MLP}
\alpha_l &= || \mathbf{W}_{l+1} ||_F || \hat{\nabla}_{\mathbf{Z}_{l+1}} ||_F || \mathbf{W}'_{l} ||_F;
\\
\gamma_l &= || \mathbf{W}_{l+1} ||_F || \sigma'(\mathbf{Z}_{l}) ||_F;
\end{aligned}
\end{equation}
the value $\alpha_l$ and $\gamma_l$ depend on the model weight before backward propagation, which are constant with respect to the gradient. 
Iterate Equation~(\ref{eq:grad_Z_upperbound_MLP}) until $l=L$ where $|| \hat{\nabla}_{\mathbf{Z}_{L}} - \nabla_{\mathbf{Z}_{L}} ||_F 
\leq \alpha_l || \hat{\mathbf{H}}_{L-1} \!-\! \mathbf{H}_{L-1} ||_F$.
In such a manner, we have 
\begin{equation}
\label{eq:grad_Z_upperbound_MLP2}
|| \hat{\nabla}_{\mathbf{Z}_{l}} - \nabla_{\mathbf{Z}_{l}} ||_F 
\leq \alpha_l || \hat{\mathbf{H}}_{l-1} \!-\! \mathbf{H}_{l-1} ||_F + \!\!\sum_{i=l+1}^L \!\! \alpha_i || \hat{\mathbf{H}}_{i-1} \!-\! \mathbf{H}_{i-1} ||_F \prod_{j=l}^{i-1} \gamma_j.
\end{equation}

According to Equation~(\ref{eq:MLP_grad}), we have the gradient of $\mathbf{W}_{l}$ given by
\begin{align}
\hat{\nabla}_{\mathbf{W}_{l}} - \nabla_{\mathbf{W}_{l}}
&= \hat{\nabla}_{\mathbf{Z}_{l}}  \hat{\mathbf{H}}_{l-1}^{\mathsf{T}} - \nabla_{\mathbf{Z}_{l}}  \mathbf{H}_{l-1}^{\mathsf{T}},
\nonumber
\\
&= \hat{\nabla}_{\mathbf{Z}_{l}}  \hat{\mathbf{H}}_{l-1}^{\mathsf{T}} - \hat{\nabla}_{\mathbf{Z}_{l}} \mathbf{H}_{l-1}^{\mathsf{T}} + \hat{\nabla}_{\mathbf{Z}_{l}} \mathbf{H}_{l-1}^{\mathsf{T}} - \nabla_{\mathbf{Z}_{l}}  \mathbf{H}_{l-1}^{\mathsf{T}},
\nonumber
\\
\label{eq:grad_W_difference_MLP}
&= \hat{\nabla}_{\mathbf{Z}_{l}} ( \hat{\mathbf{H}}_{l-1}^{\mathsf{T}} - \mathbf{H}_{l-1}^{\mathsf{T}} ) + ( \hat{\nabla}_{\mathbf{Z}_{l}} - \nabla_{\mathbf{Z}_{l}} )  \mathbf{H}_{l-1}^{\mathsf{T}}.
\end{align}

Taking Equation~(\ref{eq:grad_Z_upperbound_MLP2}) into Equation~(\ref{eq:grad_W_difference_MLP}), we have
\begin{align}
&||\hat{\nabla}_{\mathbf{W}_{l}} - \nabla_{\mathbf{W}_{l}}||_F 
\nonumber
\\
&\leq || \hat{\nabla}_{\mathbf{Z}_{l}} ||_F || \hat{\mathbf{H}}_{l-1}^{\mathsf{T}} - \mathbf{H}_{l-1}^{\mathsf{T}} ||_F + || \hat{\nabla}_{\mathbf{Z}_{l}} - \nabla_{\mathbf{Z}_{l}} ||_F || \mathbf{H}_{l-1}^{\mathsf{T}} ||_F,
\nonumber
\\
&\leq || \hat{\nabla}_{\mathbf{Z}_{l}} ||_F || \hat{\mathbf{H}}_{l-1}^{\mathsf{T}} \!-\! \mathbf{H}_{l-1}^{\mathsf{T}} ||_F \!+\! || \mathbf{H}_{l-1}^{\mathsf{T}} ||_F \Big[ \alpha_l || \hat{\mathbf{H}}_{l-1} \!-\! \mathbf{H}_{l-1} ||_F \!+ \!\!\!\! \sum_{i=l+1}^L \!\! \alpha_i || \hat{\mathbf{H}}_{i-1} \!-\! \mathbf{H}_{i-1} ||_F \prod_{j=l}^{i-1} \gamma_j \Big],
\nonumber
\\
\label{eq:W_GER_MLP}
&= \Big( \beta_l \!+\! \alpha_l || \mathbf{H}_{l-1}^{\mathsf{T}} ||_F \Big) || \hat{\mathbf{H}}_{l-1}^{\mathsf{T}} \!-\! \mathbf{H}_{l-1}^{\mathsf{T}} ||_F \!+\! || \mathbf{H}_{l-1}^{\mathsf{T}} ||_F \!\! \sum_{i=l+1}^L \alpha_i || \hat{\mathbf{H}}_{i-1}^{\mathsf{T}} \!-\! \mathbf{H}_{i-1}^{\mathsf{T}} ||_F \prod_{j=l}^{i-1} \gamma_j,
\end{align}
where $\beta_l$ is given by
\begin{equation}
\begin{aligned}
\label{eq:beta_MLP}
\alpha_l &= || \mathbf{W}_{l+1} ||_F || \hat{\nabla}_{\mathbf{Z}_{l+1}} ||_F || \mathbf{W}'_{l} ||_F;
\\
\beta_l &= || \hat{\nabla}_{\mathbf{Z}_{l}} ||_F;
\\
\gamma_l &= || \mathbf{W}_{l+1} ||_F || \sigma'(\mathbf{Z}_{l}) ||_F.
\end{aligned}
\end{equation}

During the $\mathrm{LFC}\text{-}\mathrm{ACT}$ and $\mathrm{HFC}\text{-}\mathrm{ACT}$ trainings, the activation map of a linear layer satisfies
\begin{align}
\label{eq:H_difference_lfc_MLP}
|| \mathbf{H}_l - \mathbf{H}^{\mathsf{L}}_l ||_F &= || \widetilde{\mathbf{H}}_l - \widetilde{\mathbf{H}}^{\mathsf{L}}_l ||_F = || \widetilde{\mathbf{H}}_l \odot( \mathbf{1} - \mathbf{M}) ||_F \triangleq \lambda^{\mathsf{H}}_l,
\\
\label{eq:H_difference_hfc_MLP}
|| \mathbf{H}_l - \mathbf{H}^{\mathsf{H}}_l ||_F &= || \widetilde{\mathbf{H}}_l - \widetilde{\mathbf{H}}^{\mathsf{H}}_l ||_F = || \widetilde{\mathbf{H}}_l \odot \mathbf{M} ||_F \triangleq \lambda^{\mathsf{L}}_l.
\end{align}
Taking Equations~(\ref{eq:H_difference_lfc_MLP}) and~(\ref{eq:H_difference_hfc_MLP}) into~(\ref{eq:W_GER_MLP}), we have the $\mathrm{GEB}^{\mathsf{L}}_l$ and $\mathrm{GEB}^{\mathsf{H}}_l$ of a linear layer given by
\begin{align}
\label{eq:GER_L_upperbound_MLP}
||\hat{\nabla}_{\mathbf{W}_{l}} \!-\! \nabla^{\mathsf{L}}_{\mathbf{W}_{l}}||_F 
\!&\leq\! \Big( \alpha_l || \mathbf{H}_{l-1}^{\mathsf{T}} ||_F + \beta_l \Big) \lambda^{\mathsf{H}}_l + || \mathbf{H}_{l-1}^{\mathsf{T}} ||_F \!\! \sum_{i=l+1}^L \alpha_i \lambda^{\mathsf{H}}_i \prod_{j=l}^{i-1} \gamma_j \triangleq \mathrm{GEB}^{\mathsf{L}}_l,
\\
\label{eq:GER_H_upperbound_MLP}
||\hat{\nabla}_{\mathbf{W}_{l}} \!-\! \nabla^{\mathsf{H}}_{\mathbf{W}_{l}}||_F 
\!&\leq\! \Big( || \alpha_l || \mathbf{H}_{l-1}^{\mathsf{T}} ||_F + \beta_l \Big) \lambda^{\mathsf{L}}_l + || \mathbf{H}_{l-1}^{\mathsf{T}} ||_F  \!\!\sum_{i=l+1}^L \alpha_i \lambda^{\mathsf{L}}_i \prod_{j=l}^{i-1} \gamma_j \triangleq \mathrm{GEB}^{\mathsf{H}}_l.
\end{align}

Given the expression of $\mathrm{GEB}^{\mathsf{L}}_l$ and $\mathrm{GEB}^{\mathsf{H}}_l$ by Equations~(\ref{eq:GER_L_upperbound_MLP}) and~(\ref{eq:GER_H_upperbound_MLP}), we have the $\mathrm{GEB}$ difference for a linear layer given by
\begin{equation}
\mathrm{GEB}^{\mathsf{L}}_l \!-\! \mathrm{GEB}^{\mathsf{H}}_l
\!=\! \Big( || \alpha_l || \mathbf{H}_{l-1}^{\mathsf{T}} ||_F + \beta_l \Big) (\lambda^{\mathsf{H}}_l \!-\! \lambda^{\mathsf{L}}_l) \!+ || \mathbf{H}_{l-1}^{\mathsf{T}} ||_F  \!\!\!\! \sum_{i=l+1}^L \!\! \alpha_i (\lambda^{\mathsf{H}}_i \!-\! \lambda^{\mathsf{L}}_i) \prod_{j=l}^{i-1} \gamma_j. 
\nonumber
\end{equation}

\end{proof}

\section{Proof of Theorem~\ref{theorem:low_pass}}
\label{appendix:proof_low_pass}

We give the proof of Theorem~\ref{theorem:low_pass} in this section.

\textbf{Theorem 2.} 
\textit{
For any real-valued function $f(x)$ and its moving average $\bar{f}(x) = \frac{1}{2B} \int_{x}^{x+2B} f(t) \mathrm{d}t$, let $F(\omega)$ and $\overline{F}(\omega)$ denote the Fourier transformation~\cite{madisetti1997digital} of $f(x)$ and $\bar{f}(x)$, respectively. 
Generally, we have $\overline{F}(\omega) = H(\omega) F(\omega)$, where $|H(\omega)| = \big| \frac{\sin \omega B}{\omega B} \big|$.}

\begin{proof}

We adopt the limit operator to reformulate $\bar{f}(x)$ into
\begin{equation}
\label{eq:f_bar}
    \bar{f}(x) = \frac{1}{2B} \int_{x}^{x+2B} \!\!\!\!\!\!\!\!\!\! f(t) \mathrm{d}t = \frac{1}{2B} \lim_{N \to \infty} \sum_{n=0}^{N-1} \frac{2B}{N} f(x + \frac{2Bn}{N}) = \lim_{N \to \infty} \sum_{n=0}^{N-1} \frac{1}{N} f(x + \frac{2Bn}{N})
\end{equation}

Taking Equation~(\ref{eq:f_bar}) into the Fourier Transform of $\bar{f}(x)$, we have 
\begin{equation}
\begin{aligned}
F'(\omega) &= \int_{-\infty}^{\infty} \bar{f}(x) e^{-i \omega x} \mathrm{d}x
= \int_{-\infty}^{\infty} \frac{1}{N} \lim_{N \to \infty} \sum_{n=0}^{N-1} f(x + \frac{2Bn}{N}) e^{-i \omega x} \mathrm{d}x
\\
&= \lim_{N \to \infty} \frac{1}{N} \sum_{n=0}^{N-1} \int_{-\infty}^{\infty} f(x + \frac{2Bn}{N}) e^{-i \omega x} \mathrm{d}x
\\
&= \lim_{N \to \infty} \frac{1}{N} \sum_{n=0}^{N-1} e^{i \omega \frac{2Bn}{N}} \int_{-\infty}^{\infty} f(x) e^{-i \omega x} \mathrm{d}x
= F(\omega) \lim_{N \to \infty} \frac{1}{N} 
 \sum_{n=0}^{N-1} e^{i \omega \frac{2Bn}{N}} 
\\
&= F(\omega) (1-e^{i \omega 2B}) \lim_{N \to \infty} \frac{1}{N (1 - e^{i \omega \frac{2B}{N}})} 
= F(\omega) \frac{1 - e^{i \omega 2B}}{-i \omega 2B},
\end{aligned}
\end{equation}
where $i$ denotes the imaginary unit.

Let $H(\omega) = \frac{1 - e^{i \omega 2B}}{-i \omega 2B}$.
The magnitude of $H(\omega)$ is given by
\begin{equation}
\begin{aligned}
\Big| H(\omega) \Big| &= \frac{| 1 - \cos \omega 2B + i \sin \omega 2B |}{| \omega 2B |}
= \frac{\sqrt{ (1 - \cos \omega 2B)^2 + \sin^2 \omega 2B |}}{| \omega 2B}
\\
&= \frac{\sqrt{ 4 \sin^4 \omega B + 4 \sin^2 \omega B \cos^2 \omega B}}{| \omega 2B |}
= \frac{\sqrt{ 4 \sin^2 \omega B (\sin^2 \omega B +  \cos^2 \omega B )}}{| \omega 2B |}
\\
&= \Big| \frac{\sin \omega B}{\omega B} \Big|
\end{aligned}
\end{equation}

\end{proof}

\section{Theoretical Compression Rate of \Algnameabbr{}}
\label{appendix:compression_rate}

The compression rate of \Algnameabbr{} is estimated in this section.
A general case of convolutional neural networks~(CNN) and multi-layer perceptron~(MLP) are considered for the estimation.

\subsection{Compression Rate of CNN training}

Without loss of generality, we estimate the compression rate for a block of convolutional layer~(\texttt{conv}), batch normalization layer~(\texttt{BN}) and Relu activation. 
Most of existing backbones purely stacks \texttt{conv-BN-Relu} blocks~\cite{he2016deep, huang2017densely, szegedy2015going, tan2019efficientnet, simonyan2014very}, which makes our estimated compression rate hold in practice.
Generally, the compression rate is defined as the memory reduction ratio after the compression.
To be concrete, let $\mathrm{Minibatch} \!\times\! \mathrm{Channel} \!\times\! N \!\times\! N$ denote the shape of activaition maps for a \texttt{conv-BN-Relu} block; given the block-size $B$ and bit-width $Q$, \Algnameabbr{} has the compression rate of activation maps given by Theorem~3A.


%

\textbf{Theorem 3A.}
\textit{\Algnameabbr{} has average activation map compression rate for a $\texttt{conv}$-$\texttt{BN}$-$\texttt{Relu}$ block given by 
\begin{equation}
\label{eq:appendix_compresstion_rate_conv}
R = \frac{\text{Mem of } \mathbf{H}}{\text{Mem of } ( \mathbf{H}^{\mathsf{L}}, \mathbf{V}^{\mathsf{H}}, \Delta, \delta )} = \frac{9}{ \frac{4}{\min\{B^2,N^2\}} + \frac{Q}{4} + \frac{8}{N^2} + \frac{1}{8}},
\end{equation}
where $\mathrm{Minibatch} \!\times\! \mathrm{Channel} \!\times\! N \!\times\! N$ is the shape of activation map $\boldsymbol{H}_l$; $B$ denotes the block-size of LFC average pooling; and $Q$ denotes the bit-width of HFC quantization.}

\begin{proof}

For each mini-batch updating of normal training, a \texttt{conv}-layer or \texttt{BN}-layer caches $N^2 \texttt{float32} \times 4 \mathrm{byte}/\texttt{float32} = 4N^2 \mathrm{byte}$ activation map; a $\mathrm{Relu}$ operator caches $N^2 \texttt{int8} \times 1 \mathrm{byte}/\texttt{int8} = N^2 \mathrm{byte}$ activation map.
For each mini-batch updating of \Algnameabbr{}, a \texttt{conv}-layer or \texttt{BN}-layer caches $\frac{N^2}{\min\{B^2,N^2\}} \texttt{bfloat16} \times 2 \mathrm{byte}/\texttt{bfloat16} = \frac{2N^2}{\min\{B^2,N^2\}} \mathrm{byte}$ LFC; and $QN^2 \mathrm{bit} \times \frac{1}{8} \mathrm{bit}/\mathrm{byte}= \frac{Q}{8}N^2 \mathrm{byte}$ HFC; and spends $2\texttt{bfloat16} \times 2 \mathrm{byte}/\texttt{bfloat16} = 4 \mathrm{byte}$ for $\Delta_{l}$ and $\delta_{l}$.
Moreover, a $\mathrm{Relu}$ operator caches $N^2 \texttt{bit} \times \frac{1}{8} \mathrm{byte}/\texttt{bit} = \frac{N^2}{8} \mathrm{byte}$ activation map.
Therefore, the average activation map compression rate of a $\texttt{conv}$-$\texttt{BN}$-$\texttt{Relu}$ block is given by 
\begin{equation}
R = \frac{4N^2 \times 2 + N^2}{ \big( \frac{2N^2}{\min\{B^2,N^2\}} + \frac{Q}{8}N^2 + 4 \big) \times 2 + \frac{1}{8}N^2} = \frac{9}{ \frac{4}{\min\{B^2,N^2\}} + \frac{Q}{4} + \frac{8}{N^2} + \frac{1}{8}}.
\end{equation}

\end{proof}

A higher compression rate indicates more effective compression. 
It is observed that the compression rate grows with $B$ and $N$, and decreases with $Q$.
In our experiments, we found $B \!=\! 8$ and $Q \!=\! 2$ can provide loss-less model accuracy.
In this condition, the shape of activation maps satisfies $N \geq 7$ for ResNet-50 and WRN-50-2 on the ImageNet dataset~\cite{he2016deep}.
According to Equation~(\ref{eq:appendix_compresstion_rate_conv}), we have $R_{\text{ResNet-50}}, R_{\text{WRN-50-2}} \geq 10.35$.

\subsection{Compression Rate of MLP Training}

We estimate the compression rate for a \texttt{linear}-\texttt{Relu} block in Theorem 3B.
An MLP simply stacks multiple \texttt{linear}-\texttt{Relu} blocks, such that our estimated compression rate holds for MLP models.


\textbf{Theorem 3B.}
\textit{
\Algnameabbr{} has average activation map compression rate for a \texttt{linear}-\texttt{Relu} block given by 
\begin{equation}
\label{eq:appendix_compresstion_rate_linear}
R = \frac{\text{Mem of } \mathbf{H}}{\text{Mem of } ( \mathbf{H}^{\mathsf{L}}, \mathbf{V}^{\mathsf{H}}, \Delta, \delta )} = \frac{5}{ \frac{2}{\min\{B,N\}} + \frac{Q}{8} + \frac{4}{N} + \frac{1}{8}},
\end{equation}
where $\mathrm{Minibatch} \!\times\! N$ is the shape of activation map $\boldsymbol{H}_l$; $B$ denotes the block-size of LFC average pooling; and $Q$ denotes the bit-width of HFC quantization.
}

\begin{proof}
For each mini-batch updating of normal training, a \texttt{linear}-layer caches $N \texttt{float32} \times 4 \mathrm{byte}/\texttt{float32} = 4N \mathrm{byte}$ activation map; a $\mathrm{Relu}$ operator caches $N \texttt{int8} \times 1 \mathrm{byte}/\texttt{int8} = N \mathrm{byte}$ activation map.
For each mini-batch updating of \Algnameabbr{}, a \texttt{linear}-layer caches $\frac{N}{\min\{B,N\}} \texttt{bfloat16} \times 2 \mathrm{byte}\texttt{bfloat16} = \frac{2N}{\min\{B,N\}} \mathrm{byte}$ LFC; and $QN \mathrm{bit} \times \frac{1}{8} \mathrm{bit}/\mathrm{byte}= \frac{Q}{8}N \mathrm{byte}$ HFC; and spends $2\texttt{bfloat16} \times 2 \mathrm{byte}/\texttt{bfloat16} = 4 \mathrm{byte}$ for $\Delta_{l}$ and $\delta_{l}$.
Moreover, a $\texttt{Relu}$ operator caches $N \texttt{bit} \times \frac{1}{8} \mathrm{byte}/\texttt{bit} = \frac{N}{8} \mathrm{byte}$ activation map.
Therefore, the average activation map compression rate of a \texttt{linear}-\texttt{Relu} block given by 
\begin{equation}
R = \frac{4N + N}{ \frac{2N}{\min\{B,N\}} + \frac{Q}{8}N + 4 + \frac{1}{8}N} = \frac{5}{ \frac{2}{\min\{B,N\}} + \frac{Q}{8} + \frac{4}{N} + \frac{1}{8}}.
\end{equation}
\end{proof}

\section{Related Work}

We discuss more related work about memory efficient training in this section. 
The discussion of existing work is from the perspectives of model pruning, quantization, model distributed training, randomized approximate method,  embedding table sharding, and local sensitive harshing.

\paragraph{Model Pruning.} Model pruning aims to reduce the memory footprint of DNNs by eliminating unnecessary connections or weights. These techniques leverage the observation that many weights in a trained network are redundant or have minimal impact on the model's performance. By removing these parameters, memory usage can be significantly reduced without compromising accuracy. Notable approaches include magnitude pruning, structured pruning, and iterative pruning, which selectively prune weights based on their importance or magnitude~\cite{zhong2022revisit, duan2022comprehensive, duan2022benchmarking}.

\paragraph{Quantization.}
Quantization focuses on reducing the precision of network weights and activations. Instead of using full-precision (32-bit) floating-point numbers, quantization methods represent weights and activations with lower bit-widths, such as 8-bit integers or even binary values. By quantizing parameters, memory requirements can be significantly reduced, allowing for efficient storage and computation. Recent advancements in quantization techniques, such as learned quantization and differentiable quantization~\cite{wang2022bed}, have shown promising results in preserving model accuracy.

\paragraph{Model Distributed Training.}
Model Distributed Training divides the neural network across multiple devices or machines, allowing each device to handle a specific portion of the model~\cite{10.14778/3457390.3457399, 10.14778/3529337.3529343}. This approach is particularly useful for training extremely large models that cannot fit into the memory of a single device. By partitioning the model and carefully orchestrating data and computation flows, memory requirements can be distributed across multiple devices, enabling the training of models with higher capacity.

\paragraph{Randomized Approximate Method.}
Randomized algorithms often employ sampling and sketching techniques to reduce memory requirements while providing useful insights about the data. These techniques involve selecting a representative subset of the data or summarizing the data using compact data structures. By utilizing random sampling, the algorithm can approximate characteristics of the entire dataset with a smaller memory footprint, making it more memory-efficient~\cite{wang2022dragonn, xu2021breaking, zirui2022, liu2022rsc}.

\paragraph{Embedding Table Sharding.}
Embedding table sharding refers to the practice of partitioning or dividing the embedding table into multiple smaller tables, known as shards. Each shard contains a subset of the entire embedding table's entries. Specifically, sharding provides flexibility in allocating memory resources. Instead of allocating a single large block of memory for the entire embedding table, we can allocate memory separately for each shard. This flexibility enables more efficient memory management, as memory can be allocated dynamically based on the specific needs of each shard~\cite{zha2022autoshard, zha2022dreamshard, zha2022neuroshard, ferber2022surco}. It also allows for more effective utilization of memory hierarchies, such as caching mechanisms, by optimizing the storage and retrieval of shard-specific embeddings.

\paragraph{Local sensitive Harshing~(LSH).}
LSH is a technique used in machine learning to approximate nearest neighbor search efficiently.
It is known for its scalability in handling large-scale datasets. By employing LSH, it becomes possible to partition the data into smaller subsets and store them separately. This distributed representation of the data reduces the memory requirements for each subset, making it easier to handle and process large amounts of data within the available memory resources~\cite{xu2021locality, xu2023tale, desai2021raw}.


\end{document}